\documentclass{article} 
\usepackage{iclr2025_conference,times}

\usepackage{arxiv}

\usepackage[utf8]{inputenc} 
\usepackage[T1]{fontenc}    
\usepackage{hyperref}
\usepackage{url}            
\usepackage{booktabs}       
\usepackage{amsfonts}       
\usepackage{nicefrac}       
\usepackage{microtype}      
\usepackage{xcolor}         
\usepackage{etoc}           
\usepackage{pifont}
\usepackage{adjustbox}
\usepackage{wrapfig}
\usepackage{multirow,mathtools}
\usepackage{rotating}
\usepackage{lscape}
\usepackage{longtable}
\usepackage{subcaption}
\usepackage{amssymb}
\usepackage{fontawesome}
\usepackage[font={small}]{caption}

\usepackage{multirow}
\usepackage{graphicx}
\graphicspath{{figures/}}

\newcommand{\Def}[0]{\mathrel{\mathop:}=}

\usepackage{pifont}

\usepackage{color, colortbl}
\definecolor{Gray}{gray}{0.93}
\definecolor{Orange}{rgb}{1,0.5,0}
\definecolor{DGray}{gray}{0.83}
\definecolor{LightCyan}{rgb}{0.88,1,1}

\usepackage[T1]{fontenc}

\definecolor{darkgreen}{rgb}{0.0, 0.65, 0.0}
\definecolor{darkred}{rgb}{0.5, 0.0, 0.0}
\definecolor{darkblue}{rgb}{0.0, 0.0, 0.5}
\definecolor{darkyellow}{rgb}{0.65, 0.65, 0}

\definecolor{mr}{RGB}{251, 111, 111}

\usepackage{microtype}
\usepackage{graphicx}
\usepackage{subcaption}
\usepackage{booktabs}
\usepackage{hyperref}

\usepackage{wrapfig}
\usepackage{pifont}
\usepackage{color, colortbl}

\usepackage{blindtext}
\usepackage{lipsum}

\usepackage{multirow}
\usepackage{graphicx}
\usepackage{listings}

\usepackage{bbm}

\usepackage [english]{babel}
\usepackage [autostyle, english = american]{csquotes}

\usepackage{amsmath}
\usepackage{amssymb}
\usepackage{mathtools}
\usepackage{amsthm}

\usepackage[capitalize,noabbrev]{cleveref}

\theoremstyle{plain}
\newtheorem{theorem}{Theorem}[section]
\newtheorem{proposition}[theorem]{Proposition}

\theoremstyle{definition}

\theoremstyle{remark}
\newtheorem{remark}[theorem]{Remark}
\usepackage[textsize=tiny]{todonotes}

\usepackage{pifont}
\usepackage{url}
\usepackage[most]{tcolorbox}
\usepackage{lipsum}
\usepackage{wrapfig}
\usepackage{booktabs}
\usepackage{multirow,mathtools } 

\usepackage{adjustbox}
\MakeOuterQuote{"}
\usepackage{algorithm}
\usepackage{algorithmic}
\usepackage{enumitem}
\usepackage{dsfont}

\usepackage{amsmath,amsfonts,bm}

\def\eqref#1{(\ref{#1})}

\def\1{\bm{1}}

\DeclareMathAlphabet{\mathsfit}{\encodingdefault}{\sfdefault}{m}{sl}
\SetMathAlphabet{\mathsfit}{bold}{\encodingdefault}{\sfdefault}{bx}{n}

\newcommand{\Var}{\mathrm{Var}}

\DeclareMathOperator*{\argmin}{arg\,min}

\DeclareMathOperator{\diag}{diag}

\DeclareMathOperator{\sign}{sign}

\DeclareMathOperator{\std}{std}
\DeclareMathOperator{\Bern}{Bern}
\DeclareMathOperator{\clip}{clip}

\everydisplay{\small}

\title{
Rethinking Muon Beyond Pretraining: Spectral Failures and High-Pass Remedies for VLA and RLVR
}

\author{Chongyu Fan$^{\dag}$ ~~Gaowen Liu$^{\ddag}$ ~~Mingyi Hong$^{\P}$ ~~Ramana Rao Kompella$^{\ddag}$ ~~Sijia Liu$^{\dag,\S}$\\
  $^\dag$Michigan State University ~~$^\ddag$Cisco ~~$^\P$University of Minnesota ~~$^\S$IBM Research
}
\date{}

\begin{document}

\pagestyle{fancy}
\fancyhf{}
\cfoot{\thepage}

\maketitle

\vspace{-14mm}
\begin{center}
\href{https://github.com/OPTML-Group/Pion}{\faGithub\ \texttt{GitHub}} \quad\textbar\quad
\href{https://chongyu-fan.netlify.app/posts/pion/}{\faGlobe\ \texttt{{Project Page}}}
\end{center}
\vspace{4mm}

\etocsettocdepth.toc{none}

\begin{abstract}
Muon (\underline{M}oment\underline{U}m \underline{O}rthogonalized by \underline{N}ewton--Schulz) is a matrix-aware optimizer that leverages Newton--Schulz (NS) iterations to enforce spectral gradient orthogonalization by driving all singular values of the momentum matrix toward $1$. While this \textit{uniform spectral whitening} enhances exploration and outperforms AdamW in LLM pretraining, we show it could lead to fundamental limitations beyond pretraining in two increasingly important regimes: (i) cross-modality \textit{vision-language-action} (VLA) training, where inherently low-rank action-module gradients cause amplification of noisy tail directions, and (ii) \textit{reinforcement learning with verifiable rewards} (RLVR), where low-SNR gradients and the need to preserve per-head specialization inherited from prior training make whitening unstable.
To address these challenges, we propose \textbf{Pion} (s\underline{P}ectral h\underline{I}gh-pass \underline{O}ptimization on mome\underline{N}tum), a drop-in replacement for Muon that preserves its computational efficiency while replacing uniform spectral whitening with a two-stage \textit{Promotion+Suppression} mechanism, which we call the \textit{high-pass NS} iteration. This design induces a sharp spectral high-pass effect, anchoring dominant singular values at $1$ while suppressing noisy tail components toward $0$, with controllable filter strength. 
To preserve pretrained per-head heterogeneity, Pion also supports a \textit{per-head} mode that applies updates independently across attention heads via a simple reshape, at no extra cost. Extensive experiments demonstrate consistent gains over Muon and AdamW across both VLA and RLVR regimes. In VLA training on LIBERO and LIBERO-Plus, Pion consistently outperforms both baselines across $\ell_1$-regression (VLA-Adapter) and flow-matching (VLANeXt) architectures, \textit{e.g.}, reaching $100\%$ success rate on LIBERO Object after $1{,}500$ training steps with VLA-Adapter, vs.\ $97.0\%$ for Muon and only $32.2\%$ for AdamW. {The advantage of Pion further extends to a real Franka Research 3 robot with a $\pi_{0.5}$ backbone under the DROID setup on three grasp\mbox{-}and\mbox{-}place tasks.} In RLVR post-training on Qwen3-1.7B/4B with GRPO and GMPO, Pion also outperforms AdamW on MATH and GSM8K while Muon collapses to zero.
\end{abstract}

\section{Introduction}
\label{sec: intro}

AdamW has been the dominant optimizer for deep learning. A recent line of \textit{matrix-aware} optimizers \citep{gupta2018shampoo,vyas2024soap,jordan2024muon,liu2025muon} departs from this element-wise paradigm by exploiting the spectral geometry of weight matrices. Among them, \textbf{Muon} \citep{jordan2024muon,liu2025muon} approximates steepest descent under the spectral norm via multi-step Newton--Schulz (NS) iterations that orthogonalize the momentum matrix. This design has achieved consistent gains in large language model (LLM) pretraining and inspired a family of variants \citep{li2025normuon,si2025adamuon,he2025root,amsel2025polar,ahn2025dion,wang2026taming,he2025low,pan2025unbiased,lang2026powering}. 

Despite this progress, Muon’s effectiveness \textit{beyond} pretraining remains underexplored. In this work, we ask whether its core mechanism, the matrix sign operation (\textit{i.e.}, gradient orthogonalization that drives all singular values toward $1$), remains a desirable inductive bias in non-pretraining regimes.


Inspired by this, we study two representative paradigms beyond pretraining: (i) \textit{multimodal training}, which adapts a base model to new modalities, with our focus on vision-language-action (\textbf{VLA}) models \citep{kim2024openvla,black2024pi_0,intelligence2025pi_,wang2026vla,kim2025fine} built on vision-language models (VLMs); and (ii) \textit{reinforcement-learning-based post-training}, with our focus on RL with verifiable rewards (\textbf{RLVR}) \citep{shao2024deepseekmath,guo2025deepseek,zhang2025survey}.

Therefore, the key research question we address in this work is:
\begin{tcolorbox}[before skip=2mm, after skip=2mm, boxsep=0cm, top=0.1cm, bottom=0.1cm, boxrule=0.6pt]
\begin{center}
\textit{\textbf{(Q)} 
Does Muon exhibit promise or limitations in underexplored training paradigms such as VLA and RLVR? If limitations arise, what are the causes and remedies?
}
\end{center}
\end{tcolorbox}



To address (Q), we attribute Muon’s limitations in both VLA and RLVR to a shared \textit{spectral mismatch}. In VLA, the action gradient is highly low-rank, while in RLVR the policy gradient is low-SNR. In both cases, informative directions concentrate in a few leading singular values, with the remaining tail dominated by noise (\textit{e.g.}, spectral floor or stochastic estimation noise). Muon’s NS iteration uniformly whitens this spectrum, elevating noisy tail directions to the same magnitude as the informative head and thereby corrupting the update. In addition, Muon applies NS to each weight matrix as a \textit{single} block, ignoring the per-head specialization in attention projections inherited from pretraining. This prevents Muon from respecting the heterogeneous update scales required across heads during post-training. The closest related line of work is \textit{Low-Rank Muon} \citep{he2025low,pan2025unbiased,lang2026powering}, which projects the momentum onto a top-$k$ subspace (via SVD or random sketching) before applying NS. However, it (i) has been studied primarily in LLM pretraining rather than regimes such as VLA or RLVR; (ii) relies on a fixed rank $k$ that cannot adapt across layers or training steps; and (iii) incurs non-trivial per-step SVD or sketching overhead, resulting in significantly poorer scalability than NS iterations in standard Muon.

We exploit the structure of NS to design a direct drop-in alternative to Muon, avoiding computationally intensive spectral operations such as SVD or sketching. Since each NS step reshapes normalized singular values via a scalar polynomial, improving NS reduces to redesigning this polynomial map.
Building on this view, we propose \textbf{Pion} (s\underline{P}ectral h\underline{I}gh-pass \underline{O}ptimization on mome\underline{N}tum), which splits the NS iterations into a two-stage \textit{Promotion+Suppression} sequence. The polynomial coefficients are determined by constraints that first promote dominant singular values and then suppress the tail. This yields a \textit{soft} high-pass filter that anchors leading singular values at $1$ while driving the tail toward $0$, with per-step cost identical to Muon.
We further introduce a \textit{per-head} mode that reshapes each attention projection along its head dimension and applies the high-pass NS independently per head, thereby respecting the heterogeneous update scales required across heads beyond pretraining.

$\bullet$ 
We identify fundamental limitations of Muon in VLA and RLVR (beyond pretraining) for the first time, arising from its uniform spectral whitening, which amplifies noise in low-rank gradients (\textit{e.g.}, VLA action heads) or low-SNR gradients (\textit{e.g.}, RLVR).

$\bullet$
We propose \textit{Pion}, which redesigns NS into a two-stage \textit{Promotion+Suppression} polynomial iteration (termed \textit{high-pass NS}) that preserves leading singular directions while suppressing noise, at per-step cost identical to Muon. Pion further supports a \textit{per-head mode} that applies the iteration independently across attention heads via a simple reshape, incurring no additional cost.

$\bullet$ 
On VLA training with $\ell_1$-regression and flow-matching heads over LIBERO and LIBERO-Plus as well as on a real Franka Research 3 robot using a $\pi_{0.5}$ backbone \citep{intelligence2025pi_}, 
and on RLVR post-training with GRPO and GMPO using Qwen3-1.7B/4B on MATH and GSM8K, Pion consistently outperforms AdamW and Muon while matching Muon’s computational efficiency.

\section{Related Work}
\label{sec: related_work}

\noindent \textbf{Muon and matrix-aware optimizers.}
\textit{Matrix-aware} optimizers exploit the spectral geometry of weights: Shampoo/SOAP \citep{gupta2018shampoo,vyas2024soap} use Kronecker-factored preconditioners at high memory cost, while Muon \citep{jordan2024muon,liu2025muon} orthogonalizes momentum via NS iterations. Variants improve Muon's per-parameter LR \citep{li2025normuon,si2025adamuon}, noise robustness \citep{he2025root}, NS coefficients \citep{amsel2025polar}, distributed orthonormalization \citep{ahn2025dion}, and low-rank momentum \citep{wang2026taming,he2025low}, but all retain its \textit{uniform} whitening or rely on costly SVD/sketching. Pion replaces uniform whitening with a polynomial-iteration spectral \textit{high-pass} at no additional overhead.

\noindent \textbf{Vision-language-action models.}
VLA models turn pretrained VLMs into closed-loop robot policies \citep{kim2024openvla,black2024pi_0,intelligence2025pi_,zhong2025survey}, differing mainly in the action head -- $\ell_1$-regression \citep{wang2026vla,kim2025fine,wu2026vlanext,goyal2025vla}, flow-matching \citep{lipman2022flow,black2024pi_0}, tokenization \citep{pertsch2025fast}, and discrete/diffusion decoders \citep{liang2025discrete,wen2025llada,li2024cogact} -- with further work on compactness \citep{shukor2025smolvla,wen2025tinyvla}, prompting \citep{zheng2024tracevla,zhang2026vlm4vla}, and benchmarks \citep{liu2023libero,mees2022calvin,o2024open,li2024evaluating}. The cross-modal VLA \textit{optimizer} is overlooked; we show its action-module gradient is low-rank and calls for a rank-adaptive optimizer.

\noindent \textbf{RLVR and policy optimization for LLM reasoning.}
RLVR \citep{shao2024deepseekmath,guo2025deepseek,yang2025qwen3,zhang2025survey} turns programmatic verifiers into a post-training reward, building on classical policy gradients \citep{williams1992simple,schulman2015trust,schulman2017proximal} and RLHF \citep{ouyang2022training,bai2022constitutional,ethayarajh2024kto,li2023remax}. Subsequent work mostly refines the GRPO \citep{shao2024deepseekmath} \textit{objective} -- importance-ratio normalization \citep{zhao2025geometric,zheng2025group}, clipping/IS \citep{yu2025dapo,wang2025aspo,mao2025clip,liu2026length,su2025klear}, critic-free advantage \citep{hu2025reinforce++}, KL \citep{zhang2025design}, exploration \citep{li2026back,fan2026cyclicreflex}, off-policy stability \citep{zheng2025prosperity,roux2025tapered}, and infra/dynamics \citep{sheng2025hybridflow,kwon2023efficient,liu2025understanding,zhu2025path,yue2025does}. Orthogonal to these, we target the \textit{optimizer}: per-head Pion yields stable, AdamW-matching gains where Muon collapses on the low-SNR RLVR gradient.

\section{Muon and Two Underexplored Training Regimes: VLA and RLVR}
\label{sec: preliminary}

\noindent \textbf{Muon as spectral optimization.}
Muon~\citep{jordan2024muon} is a matrix-aware optimizer whose core principle is to update a weight matrix $\boldsymbol{\Theta} \in \mathbb{R}^{m \times n}$ along the \textit{steepest descent direction under the spectral norm}. 
Given a stochastic gradient $\mathbf G_t$ at iteration $t$ as well as a momentum buffer $\mathbf M_t = \mu \mathbf M_{t-1} + \mathbf  G_t$ (with $\mu$ denoting the momentum coefficient), Muon updates the weight as
\begin{align}
    \boldsymbol{\Theta}_t = \boldsymbol{\Theta}_{t-1} - \eta \, \mathrm{msign}(\mathbf{M}_t),
    \label{eq:Muon_basic}
\end{align}
where $\eta > 0$ is the step size, and 
$\mathrm{msign}(\cdot)$ denotes a matrix sign operator, also known as \textit{gradient orthogonalization}, which transforms the momentum $\mathbf{M}_t$ in the spectral domain by mapping its singular values to $1$ while preserving the singular vectors. This gives rise to
\begin{align}
    \mathrm{msign}(\mathbf{M}) = \mathbf U \mathrm{sign}(\boldsymbol{\Sigma}) \mathbf V^\top =  \mathbf U \mathbf V^\top
    \label{eq:msign}
\end{align}
where the iteration index $t$ is omitted for brevity. Here, $\mathbf{M} = \mathbf U \boldsymbol{\Sigma} \mathbf V^\top$ denotes the \textit{compact} singular value decomposition (SVD) of $\mathbf{M}$, where $\mathbf U$ and $\mathbf V$ are the left and right singular vector matrices, and $\boldsymbol{\Sigma}$ is the $r \times r$ diagonal matrix collecting the $r = \mathrm{rank}(\mathbf M)$ strictly positive singular values. The sign operator then yields $\mathrm{sign}(\boldsymbol{\Sigma}) = \mathbf I_r$, returning $1$ for every (strictly positive) singular value.
%


\noindent \textbf{Newton--Schulz (NS) iterations in Muon.}
 As shown in \eqref{eq:msign}, Muon induces a spectrally isotropic update by assigning equal magnitude to all singular directions, which promotes strong exploration during training. However, computing $\mathrm{msign}(\mathbf{M})$ via SVD incurs significant computational overhead and is impractical for large model training. In practice, Muon instead approximates the matrix sign operator using a small number of NS (Newton--Schulz) iterations. 
 
 The rationale behind the NS iteration is based on the equivalent form
$\mathrm{msign}(\mathbf{M}) = \mathbf{M}(\mathbf{M}^\top \mathbf{M})^{-\frac{1}{2}}$,
which reduces the problem to computing $(\mathbf{M}^\top \mathbf{M})^{-\frac{1}{2}}$. This inverse square root is then approximated via a polynomial iteration derived from a local Taylor expansion around the identity. As a result, NS iteratively applies low-order matrix polynomials to approximate $(\mathbf{M}^\top \mathbf{M})^{-1/2}$, and thus $\mathrm{msign}(\mathbf{M})$, without requiring explicit matrix decomposition. Specifically, for a general matrix $\mathbf{X}$, the matrix sign operator $\mathrm{msign}(\mathbf{X})$ is approximated via  NS iteration of the following form \citep{jordan2024muon}
\begin{align}
    \mathbf{X} \;\leftarrow\; a\,\mathbf{X} + b\,\mathbf{X}\mathbf{X}^\top\mathbf{X} + c\,\mathbf{X}(\mathbf{X}^\top\mathbf{X})^2, ~~\text{with}~~\text{$(a, b, c) = (3.4445, -4.7750, 2.0315)$},
    \label{eq: ns_matrix}
\end{align}
where the input is pre-normalized as $\mathbf{X} \leftarrow \mathbf{X}/(\|\mathbf{X}\|_{\mathrm{F}} + \epsilon)$ (with small $\epsilon \geq 0$) to bound all singular values within $[0,1]$, and $\|\cdot\|_{\mathrm{F}}$ denotes the Frobenius norm.
Setting $\mathbf{X} = \mathbf{M}_t$, the NS iterations are used in place of \eqref{eq:msign} to approximate the $\mathrm{msign}$ operation in the Muon update \eqref{eq:Muon_basic}.

\noindent \textbf{Underexplored regimes for Muon beyond LLM pretraining.}
Muon is widely used for LLM pretraining. We show that Muon-type optimizers also hold significant potential beyond this setting. However, the conventional Muon design exhibits important limitations in these settings (as will be shown in Sec.\,\ref{sec: motivation}), leading to suboptimal performance and hindering its broader adoption.
Throughout our work, we focus on two underexplored training regimes for Muon: (i) multimodal training of VLA (vision-language-action) models, and (ii) post-training via RLVR (reinforcement learning with verifiable rewards), where Muon remains less explored than AdamW.

\noindent \textbf{(i) VLA} trains a policy on offline demonstrations $\mathcal{D} = \{(\mathbf{x}, \mathbf{c}, \mathbf{a})\}$ to map visual observations $\mathbf{x}$ and language instructions $\mathbf{c}$ to continuous robot actions $\mathbf{a}$. Internally, the policy is factorized into a VLM (vision-language model) backbone and an action head, parameterized as $\boldsymbol{\Theta} = \{\boldsymbol{\Theta}_{\mathrm{VLM}}, \boldsymbol{\Theta}_{\mathrm{action}}\}$. We consider two representative designs for the action head $\boldsymbol{\Theta}_{\mathrm{action}}$ (training losses detailed in \textbf{Appendix\,\ref{sec:prelim_vla_heads}}): a \textit{$\ell_1$-regression head} \citep{wang2026vla,kim2025fine}, and a \textit{flow-matching head} \citep{lipman2022flow,black2024pi_0,wu2026vlanext}.

 \textbf{(ii) RLVR} is a \textit{post-training} paradigm in which the supervised fine-tuning (SFT)-initialized policy is further updated by policy gradient against a rule-based, verifiable reward \citep{shao2024deepseekmath}. Unlike SFT, which matches token-level teacher signals on offline demonstrations, RLVR alternates between three stages at every iteration: \textit{rollout}, \textit{scoring}, and \textit{policy update}. We instantiate the policy update via two algorithms, GRPO \citep{shao2024deepseekmath} and GMPO \citep{zhao2025geometric} (training objectives formalized in \textbf{Appendix\,\ref{sec:prelim_rlvr_algs}}).

\section{Rethinking Muon in Heterogeneous and Noisy Training Regimes}
\label{sec: motivation}

In this section, we show that the default Muon design exhibits fundamental limitations in VLA and RLVR, revealing opportunities for improved optimizer design.



\noindent \textbf{Rank adaptiveness in cross-modality VLA training.}
VLA models jointly train three heterogeneous modules, a vision encoder, a language backbone, and an action head \citep{kim2024openvla,black2024pi_0}, whose gradients can differ significantly in their intrinsic dimensionality. 
To quantify this heterogeneity, we use the \textit{effective rank} (\textit{erank}) \citep{roy2007effective} of a gradient matrix $\mathbf{G} \in \mathbb{R}^{m \times n}$ (\textit{w.l.o.g.}, $n \leq m$), defined via the entropy of its singular value spectrum:
\begin{align}
    \mathrm{erank}(\mathbf{G}) \Def \exp \bigl( H(\mathbf p) \bigr),
    \quad
    H(\mathbf p) = -\sum_{i=1}^{n} p_i \log p_i,
    \quad 
     p_i = \frac{\sigma_i(\mathbf{G})}{\sum_{j=1}^{n} \sigma_j(\mathbf{G})},
    \label{eq: erank}
\end{align}
where $\mathbf p = [p_1, \ldots, p_n]^\top$,  and 
$\sigma_i(\mathbf{G})$ denotes the $i$-th singular value of $\mathbf{G}$.
A higher erank indicates that the gradient energy is distributed across many directions.


\textbf{Fig.\,\ref{fig:vla_optimizer_compare}-(a)} reports the average per-module erank along the trajectory of training VLA-Adapter on LIBERO Object. The vision module maintains the highest erank, the language module is intermediate, and the action module consistently exhibits the \textit{lowest} erank. This ordering is stable across training steps, with intra-module variance (column-wise) much smaller than inter-module variance (row-wise). It also aligns with the information capacity of each modality: vision inputs encode rich pixel-level statistics, language tokens use high-dimensional embeddings to disambiguate a large vocabulary, while each action is just a seven-dimensional vector 
encoding the incremental end-effector translation, rotation, and a binary gripper command. Given this low-rank structure, applying Muon uniformly to the action module inflates every normalized singular value toward $1$, making Muon ill-suited for the action module despite its effectiveness on the higher-rank vision and language modules.

\begin{figure}[htbp]
    \centering
    \vspace{-4mm}
    \begin{tabular}{ccc}
    \hspace*{-3mm}  \includegraphics[width=0.29\textwidth]{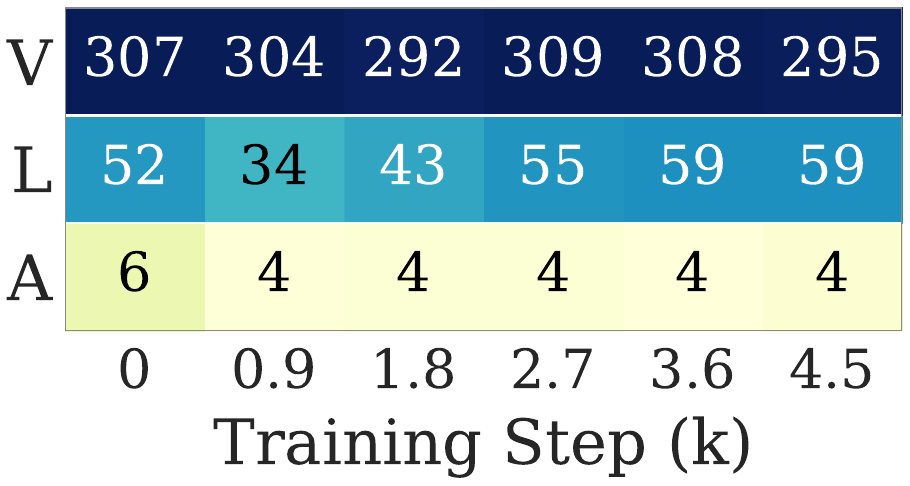} &
    \includegraphics[width=0.29\textwidth]{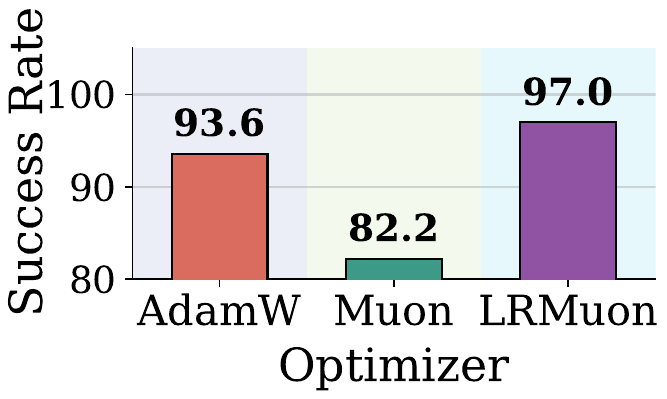} &
    \includegraphics[width=0.29\textwidth]{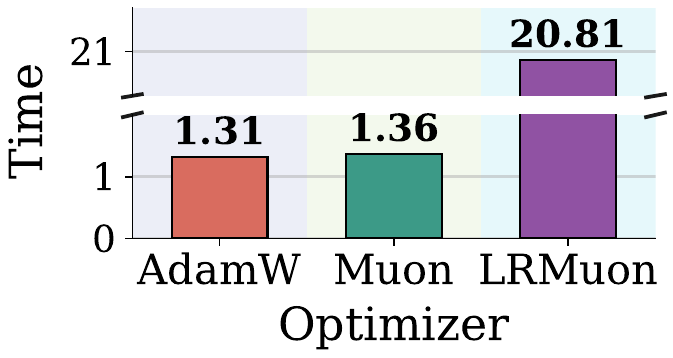}
    \\
   \hspace*{-5mm} \small{(a) Per-module gradient erank}
    &
    \small{(b) Test success rate}
    &
    \small{(c) Total training time (hrs)}
    \end{tabular}
    \vspace{-1mm}
    \caption{\small{Limitations of Muon in VLA training (VLA-Adapter on LIBERO Object). (a) Average per-module gradient erank (V/L/A) along the training trajectory of $4.5$k steps, recorded every $900$ steps. (b)--(c) 
    Test success rates on LIBERO Object for models trained for $4.5$k steps, along with total training time (hours), under three optimizer configurations: AdamW / Muon / LRMuon for the action module, with AdamW fixed for VL modules. 
    }}
    \label{fig:vla_optimizer_compare}
    \vspace{-2mm}
\end{figure}

\textit{Can existing Muon variants address the limitation in VLA training?}
A natural candidate is \textit{Low-rank Muon} (\textbf{LRMuon}) \citep{he2025low,pan2025unbiased,lang2026powering}, which projects the momentum onto a low-rank subspace (via SVD or Gaussian sketching) prior to gradient orthogonalization. This approach can adapt to the low-rank structure of the action-module gradients. However, both SVD and Gaussian sketching incur substantially higher computational cost than NS, leading to slower training.
To validate this, \textbf{Fig.\,\ref{fig:vla_optimizer_compare}-(b,c)} reports the success rate on the LIBERO Object evaluation set together with the total training time, under three optimizer configurations that share the same AdamW updates on the vision and language modules and \textit{differ only in the action module:} (i) \textbf{AdamW}, (ii) \textbf{Muon}, and (iii) \textbf{LRMuon} (see \textbf{Alg.\,\ref{alg:svdmuon_optimizer}} in \textbf{Appendix\,\ref{sec:svdmuon}} for details). We deliberately fix the V/L optimizer to AdamW, 
so that \textit{the comparison isolates the effect of the action-module optimizer}.
As shown, Muon underperforms AdamW, as expected from the rank heterogeneity shown in Fig.\,\ref{fig:vla_optimizer_compare}-(a).
In addition, LRMuon achieves the highest success rate, confirming the benefit of rank-aware optimization for the action module; however, it incurs about $15\times$ higher training cost than AdamW and Muon.

%
%
Motivated by the above, we summarize the first limitation of Muon below.


\begin{tcolorbox}[before skip=2mm, after skip=0.0cm, boxsep=0.0cm, middle=0.0cm, top=0.1cm, bottom=0.1cm, boxrule=0.6pt]
\textit{\textbf{(Limitation 1)}  
Lack of rank adaptiveness:  Conventional Muon is not adaptive to rank heterogeneity across modules, leading to suboptimal performance, while explicit low-rank projection introduces significant computational overhead, limiting scalability.
}
\end{tcolorbox}
\vspace*{2mm}


\noindent \textbf{SNR tolerance for RLVR post-training.}
Despite recent progress applying Muon to SFT-based (pre-)training \citep{liu2025muon,si2025adamuon,li2025normuon}, its effectiveness in post-training, particularly for RLVR, remains largely unexplored.
To understand this gap, we examine how SFT and RLVR, as two post-training paradigms, differ in terms of gradient signal-to-noise ratio (SNR). Unlike LLM pretraining, post-training typically requires only moderate modifications to weights \citep{gan2026neural}, making optimization more sensitive to noise. Meanwhile, as discussed in \textbf{Sec.\,\ref{sec: preliminary}}, a key characteristic of Muon is its strong exploration behavior induced by the uniform spectral sign function \eqref{eq:msign}, which can amplify noise during training.

Motivated by the above, we analyze the per-step gradient SNR of a layer’s weight matrix, defined as
\begin{align}
    \mathrm{SNR}(\mathbf{G}) \Def \frac{\bigl\|\,\mathbb{E}[\mathbf{G}]\,\bigr\|_{\mathrm{F}}^2}{\mathbb{E}\bigl[\,\|\mathbf{G} - \mathbb{E}[\mathbf{G}]\|_{\mathrm{F}}^2\,\bigr]},
    \label{eq: snr}
\end{align}
where $\mathbf{G}$ denotes the stochastic gradient with respect to a layer’s weight matrix, and the expectation is taken over the batch. A higher SNR indicates a cleaner gradient signal.


\begin{wrapfigure}{r}{0.43\textwidth}
    \centering
    \vspace{-4mm}
    \begin{tabular}{cc}
    \hspace*{-3mm}
    \includegraphics[width=0.225\textwidth]{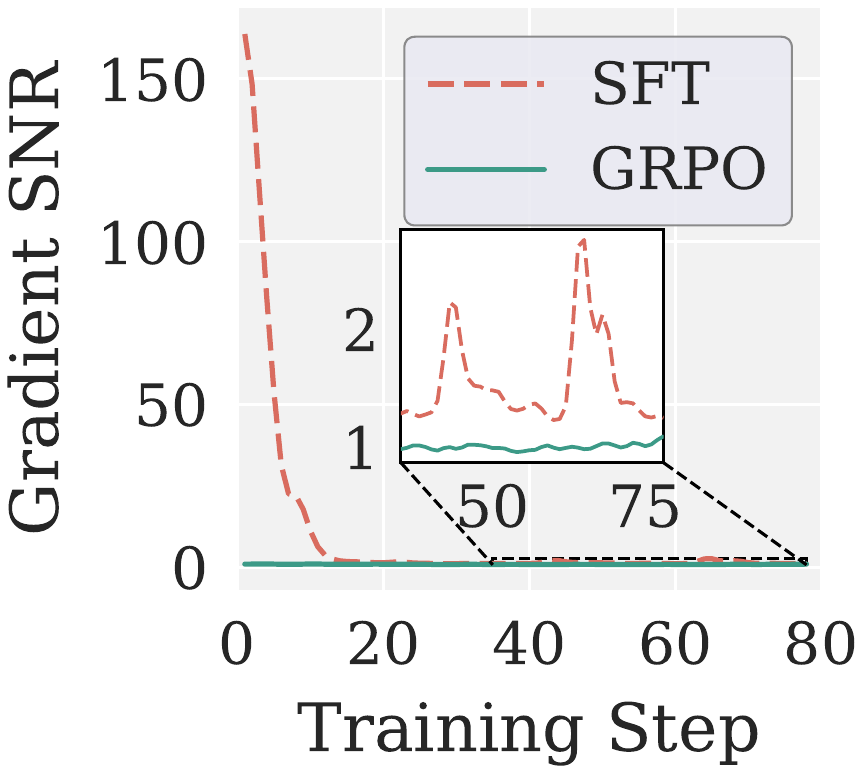}
    &
    \hspace*{-5mm}
    \includegraphics[width=0.215\textwidth]{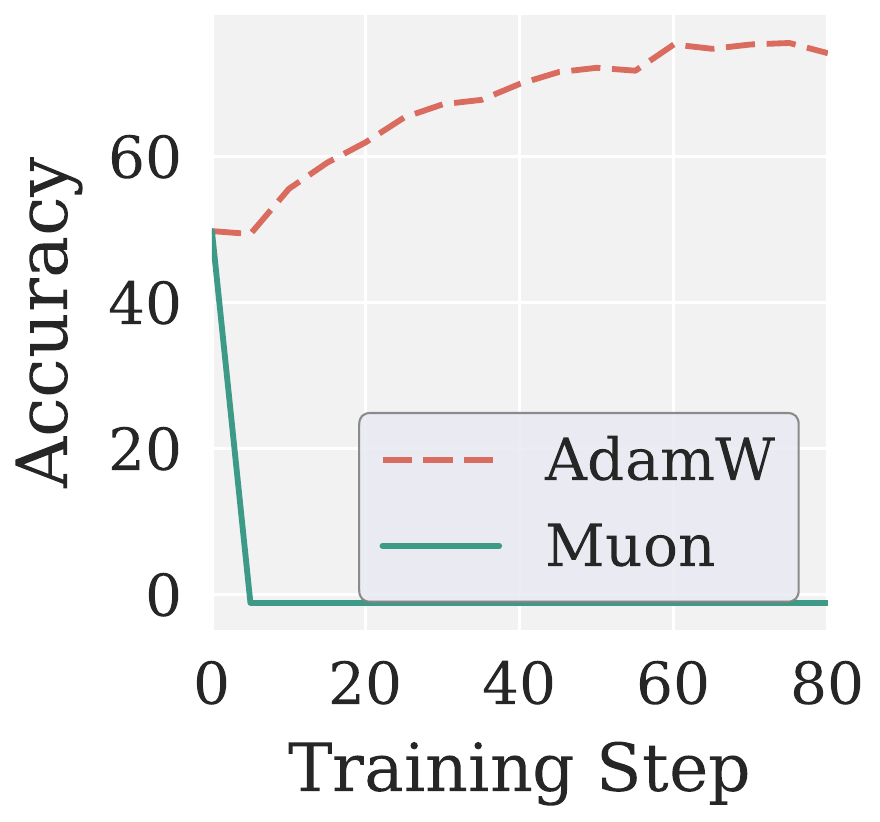}
    \\
    \small{(a)}
    &
    \hspace*{2mm}\small{(b)}
    \end{tabular}
    \vspace*{-2mm}
    \caption{\small{(a) Gradient SNR of SFT vs. GRPO (AdamW, Qwen3-1.7B) on MATH levels 3--5. (b) MATH500 accuracy of Qwen3-1.7B via GRPO (AdamW vs. Muon).}}
    \label{fig:rl_analysis}
    \vspace*{-4mm}
\end{wrapfigure}

We use GRPO \citep{shao2024deepseekmath} as the representative RLVR algorithm, train Qwen3-1.7B on MATH levels 3--5 \citep{liu2025understanding}, and evaluate on MATH500. \textbf{Fig.\,\ref{fig:rl_analysis}-(a)} compares the gradient SNR of SFT and GRPO, both optimized with AdamW. As shown, GRPO consistently exhibits a much \textit{lower} SNR than SFT throughout training. 
We attribute this gap to two primary sources of additional noise in GRPO. First, GRPO has coarser \textit{supervision granularity}: SFT receives token-level teacher signals, whereas GRPO relies on trajectory-level rewards, resulting in a significantly sparser learning signal per token. Second, GRPO relies on \textit{stabilization mechanisms}: Importance sampling, clipping, and group-relative normalization in \eqref{eq:grpo_obj} reweight or suppress portions of per-token gradients, thereby further increasing gradient variance. As a result, GRPO gradients exhibit a \textit{low-SNR} structure, a regime in which Muon’s spectral whitening becomes counterproductive. A detailed derivation is provided in \textbf{Appendix\,\ref{sec:snr}}.

\textbf{Fig.\,\ref{fig:rl_analysis}-(b)} reports the evaluation accuracy of GRPO under AdamW and Muon. As shown, GRPO using AdamW steadily improves accuracy throughout training, whereas GRPO using Muon exhibits a \textit{model collapse}: the accuracy drops from the initial checkpoint and converges to near zero. 
This behavior confirms that Muon’s uniform spectral whitening amplifies noisy directions in low-SNR GRPO gradients to the same magnitude as informative ones, rapidly corrupting the policy. A further limitation is that Muon’s $\mathrm{msign}$ (via NS iterations) operates on each layer-wise weight matrix as a \textit{single} block, ignoring the \textit{per-head specialization} established during pretraining in attention projections.


We summarize the above limitation of Muon as evidenced in RLVR post-training below.
\begin{tcolorbox}[before skip=2mm, after skip=0.0cm, boxsep=0.0cm, middle=0.0cm, top=0.1cm, bottom=0.1cm, boxrule=0.6pt]
\textit{\textbf{(Limitation 2)}  
Lack of noise adaptiveness: Muon’s uniform spectral whitening amplifies noisy directions in low-SNR gradients, making it ill-suited for noise-sensitive post-training regimes.
}
\end{tcolorbox}
\vspace*{2mm}

Both Limitations 1 and 2 stem from the inappropriate spectral exploration induced by the $\mathrm{msign}$ operator (\textit{i.e.}, via NS iterations). This motivates us to improve the design of NS iterations in the next section to enhance Muon's adaptiveness to rank heterogeneity and resilience to low-SNR gradients.

\section{Pion: s\underline{P}ectral h\underline{I}gh-pass \underline{O}ptimization on mome\underline{N}tum}
\label{sec: method}

\noindent \textbf{A unifying spectral view of Muon’s limitations: informative head vs.\ noisy tail.}
Although the two limitations of Sec.\,\ref{sec: motivation} originate from different sources (low erank for VLA, low SNR for RLVR), they share a \textit{common spectral signature}: in the SVD of $\mathbf{M}_t$, the few \textit{leading} singular values carry the informative descent direction, while the long \textit{tail} of small singular values is dominated by noise (spectral floor for low erank, stochastic estimation noise for low SNR). 
Muon's $\mathrm{msign}$, by driving \textit{every} $\sigma_i$ to $1$, lifts this tail to the same magnitude as the head and corrupts the update in both regimes. This motivates a single remedy, a \textit{spectral high-pass} that \textit{retains large singular values} (anchoring them near $1$) and \textit{suppresses small singular values} (contracting them toward $0$), in  contrast to Muon's uniform whitening (\textbf{Fig.\,\ref{fig:pion_iter}-(a)}). We realize this with \textbf{Pion} (s\underline{P}ectral h\underline{I}gh-pass \underline{O}ptimization on mome\underline{N}tum), which inherits Muon's control flow and per-step cost and differs only in the coefficients of its NS iteration. 

\noindent \textbf{A two-stage high-pass NS mechanism as a remedy.}
A single NS step \eqref{eq: ns_matrix} on $\mathbf{X} = \mathbf{U}\boldsymbol{\Sigma}\mathbf{V}^\top$ factors through the SVD as $\mathbf{U}\bigl(a\boldsymbol{\Sigma} + b\boldsymbol{\Sigma}^3 + c\boldsymbol{\Sigma}^5\bigr)\mathbf{V}^\top$ via the identity $\mathbf{X}(\mathbf{X}^\top \mathbf{X})^j = \mathbf{U}\boldsymbol{\Sigma}^{2j+1}\mathbf{V}^\top$. Hence the NS step preserves $(\mathbf{U}, \mathbf{V})$ and \textit{independently} reshapes each $\sigma_i \in [0, 1]$ through the polynomial
\begin{align}
    f(\sigma;\, a, b, c) \Def a\sigma + b\sigma^3 + c\sigma^5.
    \label{eq: pion_poly}
\end{align}
Thus, designing an NS iteration reduces to designing $f$ on $[0, 1]$ (see \textbf{Appendix\,\ref{sec:poly_vs_power}} for the full derivation). A single polynomial $f$ in \eqref{eq: pion_poly} is insufficient to produce a sharp high-pass profile, so we split the NS iteration (with $k=5$ steps by default) into two stages: an early-stage \textbf{Promotion} polynomial $f_{\mathrm{p}}$ (\textbf{Fig.\,\ref{fig:pion_iter}-(b)}) applied for $k_{\mathrm{p}}$ steps to reinforce dominant singular values, and a late-stage \textbf{Suppression} polynomial $f_{\mathrm{s}}$ (\textbf{Fig.\,\ref{fig:pion_iter}-(c)}) applied for $k_{\mathrm{s}} = k - k_{\mathrm{p}}$ steps to attenuate smaller components, each with its own coefficients $(a,b,c)$.

\begin{figure}[htbp]
    \centering
    \begin{tabular}{cccc}
    \hspace{-5mm} \includegraphics[width=0.253\textwidth]{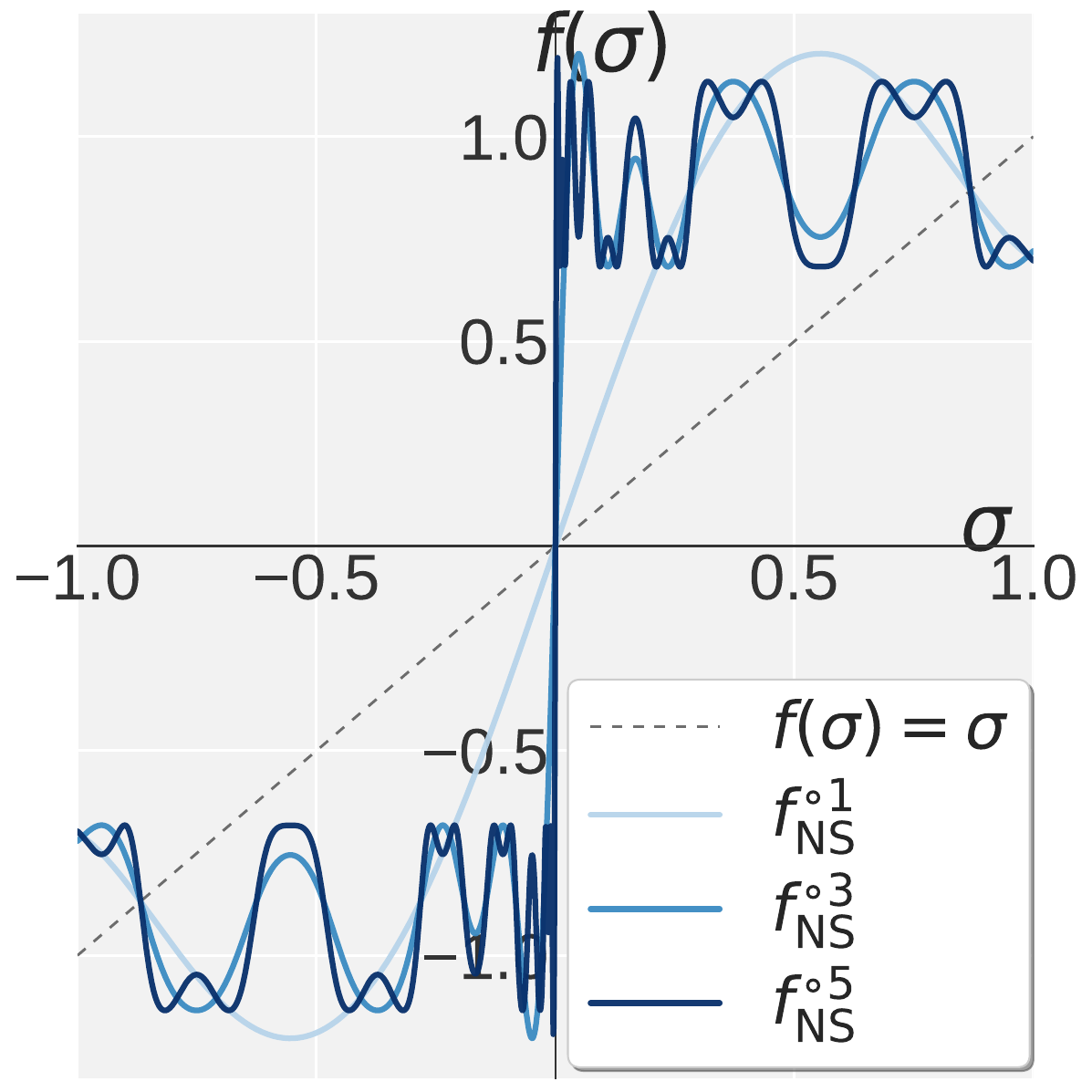} &
    \hspace{-5mm} \includegraphics[width=0.253\textwidth]{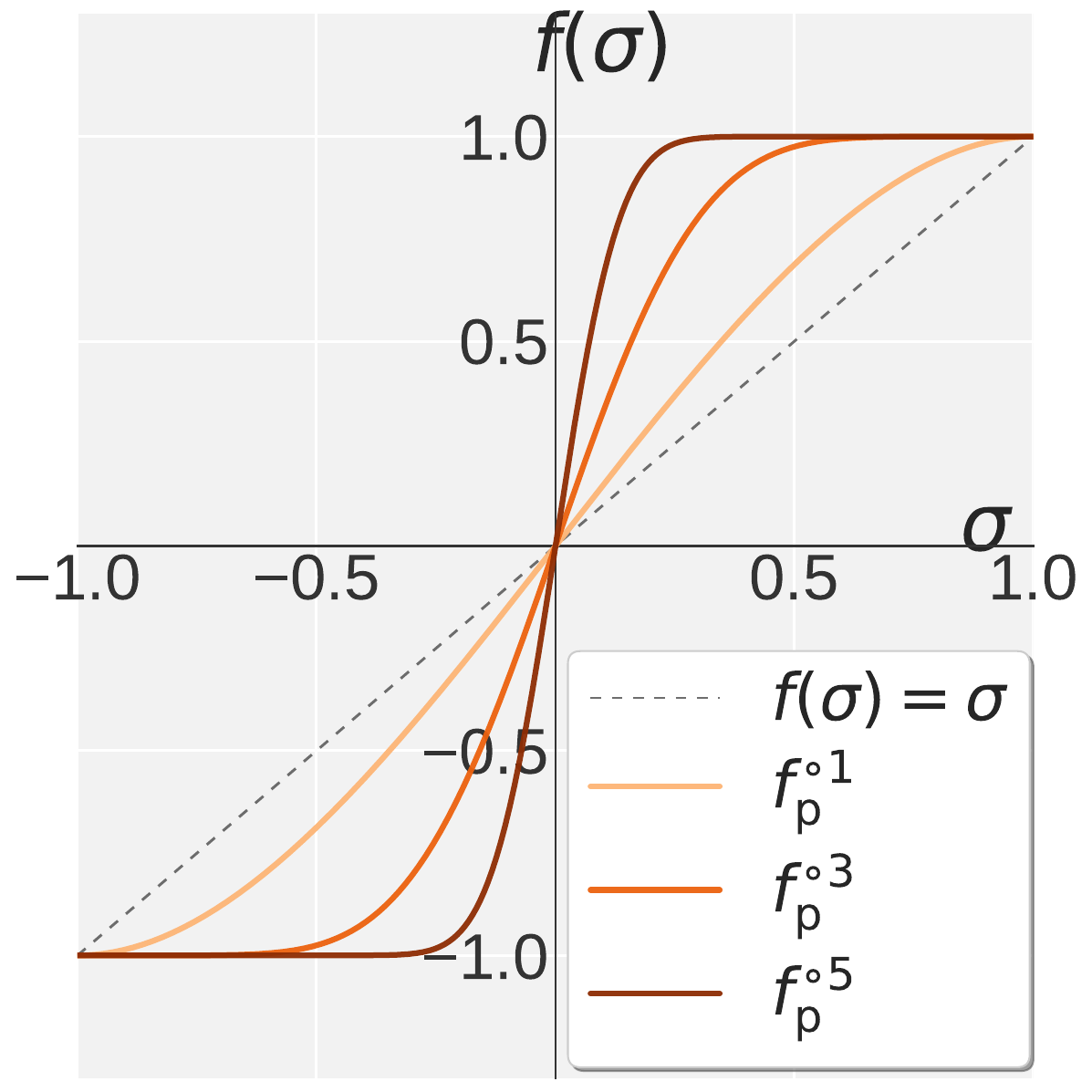} &
    \hspace{-5mm} \includegraphics[width=0.253\textwidth]{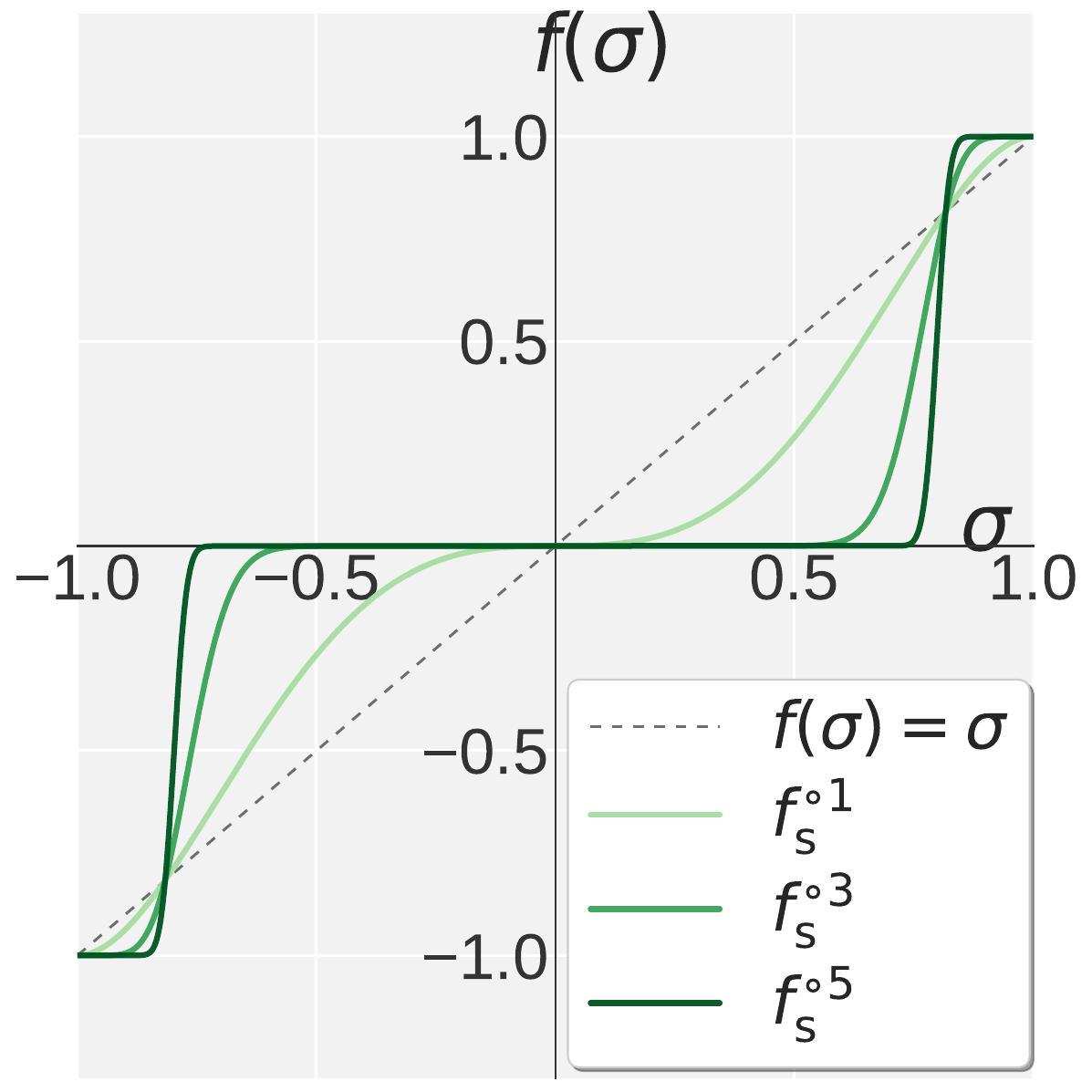} &
    \hspace{-5mm} \includegraphics[width=0.253\textwidth]{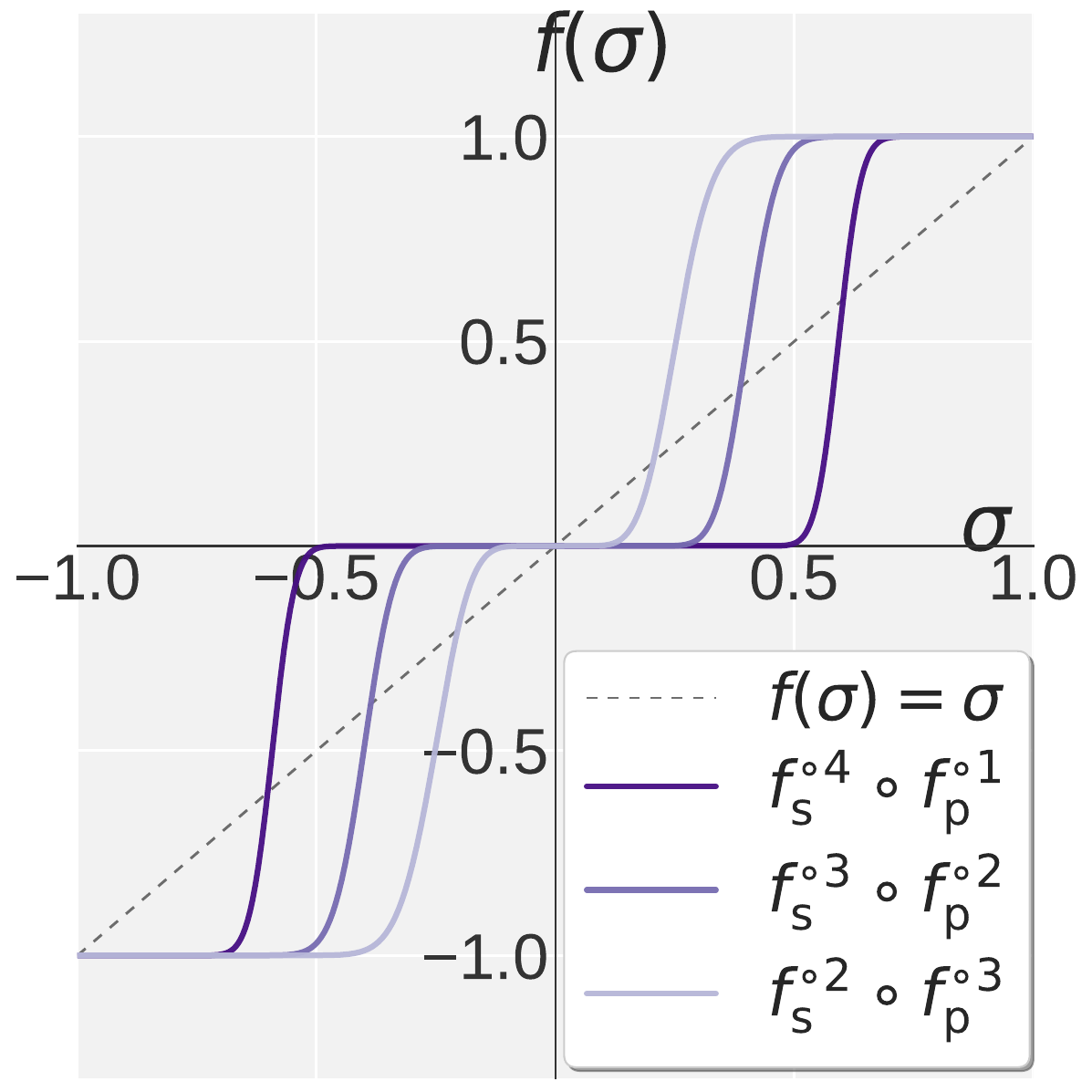} \\
    \small{(a) Muon NS} & \small{(b) Promotion $f_{\mathrm{p}}$} & \small{(c) Suppression $f_{\mathrm{s}}$} & \small{(d) High-pass NS}
    \end{tabular}
    \vspace{-1mm}
    \caption{\small{Visualization of $f(\sigma)$ in \eqref{eq: pion_poly} over $\sigma \in [0,1]$, with $f(\sigma)=\sigma$ shown as the identity reference. (a) $f^{t}_{\mathrm{NS}}$ denotes Muon's NS iteration applied $t$ times. (b) $f^{t}_{\mathrm{p}}$ denotes the Promotion polynomial $f_{\mathrm{p}}$ \eqref{eq: pion_promotion} applied $t$ times. (c) $f^{t}_{\mathrm{s}}$ denotes the Suppression polynomial $f_{\mathrm{s}}$ \eqref{eq: pion_suppression} applied $t$ times. (d) Pion's high-pass NS iteration (Alg.\,\ref{alg:pion_optimizer}): $f^{k_{\mathrm{s}}}_{\mathrm{s}}\circ f^{k_{\mathrm{p}}}_{\mathrm{p}}$ applies $k_{\mathrm{p}}$ Promotion steps followed by $k_{\mathrm{s}} = 5 - k_{\mathrm{p}}$ Suppression steps.}}
    \label{fig:pion_iter}
    \vspace{-1mm}
\end{figure}

The Promotion stage $f_{\mathrm{p}} \Def f(\,\cdot\,;\, a_{\mathrm{p}}, b_{\mathrm{p}}, c_{\mathrm{p}})$ monotonically amplifies all singular values $\sigma$, so as to (i) lift as many singular values as possible above the subsequent suppression threshold and (ii) preserve their relative ordering, ensuring that the later Suppression eventually removes only the smallest. The three coefficients in \eqref{eq: pion_poly} are pinned by two equality constraints and one inequality:
\textbf{(P1)} \textit{fixed point} $f_{\mathrm{p}}(1) = 1$ and \textbf{(P2)} \textit{first-order stationarity} $f_{\mathrm{p}}'(1) = 0$ (both shared with Suppression) anchor any direction already at $\sigma = 1$;
\textbf{(P3)} \textit{boundary concavity} $f_{\mathrm{p}}''(1) \leq 0$, together with (P2), ensures that $\sigma = 1$ is a maximum, \textit{i.e.}, prevents the Promotion from curving upward near $\sigma = 1$. See \textbf{Fig.\,\ref{fig:pion_iter}-(b)} for illustration.
As derived in \textbf{Appendix\,\ref{sec:pion_derivation}}, conditions (P1)--(P3) directly carve out the upper bound $a_{\mathrm{p}} \leq 1.875$, and additionally requiring $f_{\mathrm{p}}$ to be monotonically non-decreasing on $[0,1]$ (so that the relative ordering of singular values is preserved across each Promotion step) tightens the lower bound to $a_{\mathrm{p}} \geq 0$, yielding $a_{\mathrm{p}} \in [0,1.875]$. Since $f_{\mathrm{p}}'(0)=a_{\mathrm{p}}$ determines the slope at the origin, we set $a_{\mathrm{p}}=1.875$ to maximize promotion, thereby amplifying small singular values as strongly as possible. This choice uniquely determines the polynomial coefficients for the Promotion stage:
\begin{align}
    f_{\mathrm{p}}(\sigma) = a_{\mathrm{p}}\,\sigma + b_{\mathrm{p}}\,\sigma^3 + c_{\mathrm{p}} \,\sigma^5, ~~\text{with}~~\text{$(a_{\mathrm{p}}, b_{\mathrm{p}}, c_{\mathrm{p}}) = (1.875, - 1.25,0.375)$}.
    \label{eq: pion_promotion}
\end{align}
With these coefficients, the derivative becomes a perfect square,
$f_{\mathrm{p}}'(\sigma)=1.875\,(1-\sigma^2)^2 \geq 0$, ensuring monotonicity on $[0,1]$, as shown in Fig.\,\ref{fig:pion_iter}-(b).


The Suppression stage $f_{\mathrm{s}} \Def f(\,\cdot\,;\, a_{\mathrm{s}}, b_{\mathrm{s}}, c_{\mathrm{s}})$ pins large singular values at $1$ while contracting smaller ones toward $0$ (\textbf{Fig.\,\ref{fig:pion_iter}-(c)}).  
It inherits $f_{\mathrm{s}}(1)=1$ and $f_{\mathrm{s}}'(1)=0$, and adds the \textit{spectral filtering} condition $f_{\mathrm{s}}'(0)=0$, which removes the linear term near the origin so that small singular values are driven to $0$ by higher-order terms.
%
These constraints give the \textit{Suppression polynomial}:
\begin{align}
    f_{\mathrm{s}}(\sigma) = a_{\mathrm{s}}\,\sigma + b_{\mathrm{s}}\,\sigma^3 + c_{\mathrm{s}} \,\sigma^5, ~~\text{with}~~\text{$(a_{\mathrm{s}}, b_{\mathrm{s}}, c_{\mathrm{s}}) = (0,2.5,- 1.5)$}.
    \label{eq: pion_suppression}
\end{align}

\noindent \textbf{The Pion optimizer and its two application modes.}
Chaining $k_{\mathrm{p}}$ Promotion steps with $k_{\mathrm{s}}$ ($ = k - k_{\mathrm{p}}$) Suppression steps yields a \textbf{high-pass NS} iteration; the resulting Muon variant is termed \textbf{Pion} (see the full algorithm in \textbf{Appendix\,\ref{sec:pion_algorithms}}).
Fixing $k = 5$ preserves Muon's per-step cost, and $k_{\mathrm{p}} \in \{0, 1, \ldots, 5\}$ becomes the single hyperparameter that controls the high-pass cutoff: \textbf{Fig.\,\ref{fig:pion_iter}-(d)} shows that Pion exhibits a sharp transition between the pinned region ($\sigma \mapsto 1$) and the filtered region ($\sigma \mapsto 0$). Empirically, Suppression-dominant allocations with $k_{\mathrm{s}} \geq 3$ consistently perform best for VLA and RLVR training, as they more aggressively suppress noisy tail while preserving the informative head.

The high-pass NS admits \textit{two modes}: (i) the \textbf{default mode} applies the iteration to each weight matrix as a single block, mirroring Muon; (ii) the \textbf{per-head mode} first reshapes each attention projection along its head dimension into multiple per-head sub-matrices and runs the iteration independently on each. We use the default mode for VLA training (\textbf{Sec.\,\ref{sec: exp_vla}}) and the per-head mode for RLVR post-training (\textbf{Sec.\,\ref{sec: exp_rl}}), as explained next.

\begin{wrapfigure}{r}{0.44\textwidth}
    \centering
    \vspace*{-6mm}
    \begin{tabular}{cc}
    \hspace*{-4mm}
    \includegraphics[width=0.218\textwidth]{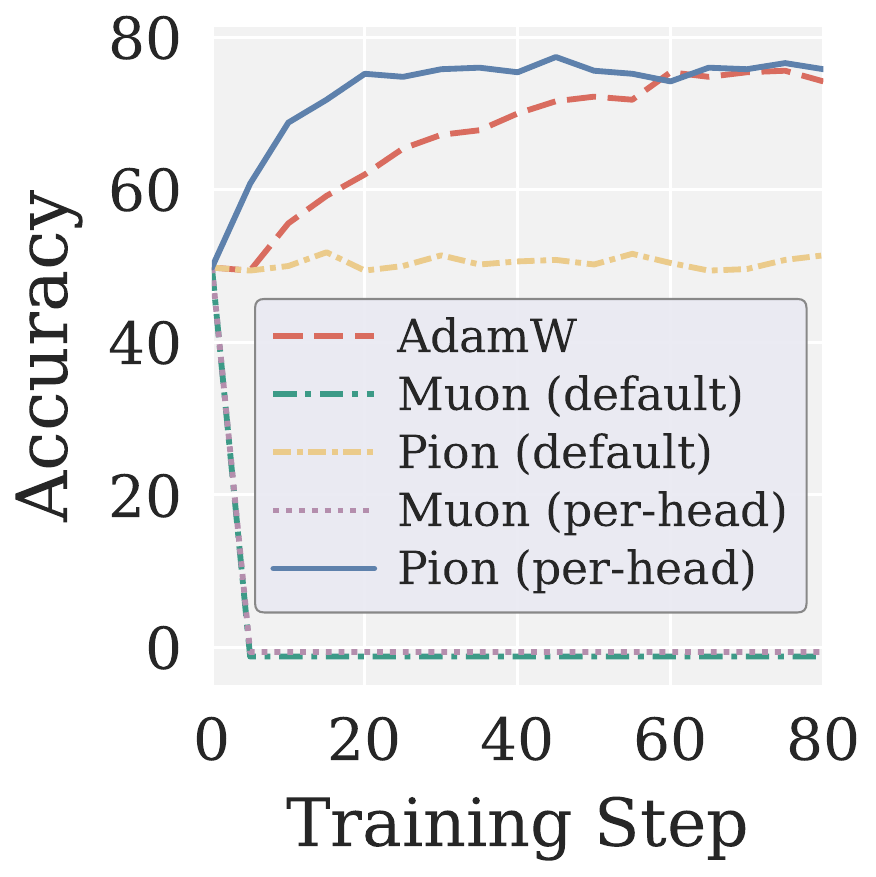}
    &
    \hspace*{-5mm}
    \includegraphics[width=0.232\textwidth]{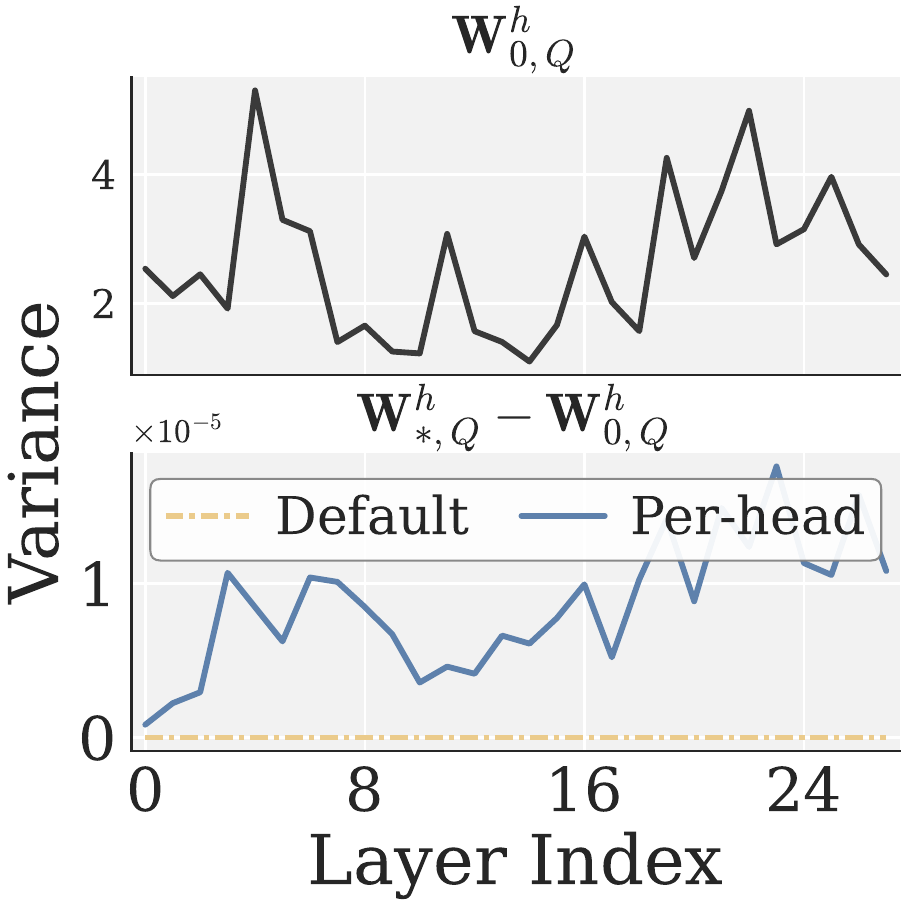}
    \\
    \small{(a)}
    &
    \hspace*{2mm}\small{(b)}
    \end{tabular}
    \vspace{-2mm}
    \caption{\small{Effect of per-head high-pass NS on RLVR (Qwen3-1.7B, GRPO on MATH levels 3--5). (a) MATH500 accuracy of AdamW, Muon (default vs. per-head), and Pion (default vs. per-head). (b) Cross-head Q-projection variance: before-RLVR weight $\mathrm{Var}(\|\mathbf{W}_{0,Q}^h\|_{\mathrm{F}})$ (top) and after-RLVR update $\mathrm{Var}(\|\mathbf{W}_{*,Q}^h-\mathbf{W}_{0,Q}^h\|_{\mathrm{F}})$ for default vs. per-head Pion (bottom).}}
    \label{fig:mh_analysis}
    \vspace{-3mm}
\end{wrapfigure}
\noindent \textbf{Why per-head high-pass NS is needed for RLVR.}
RLVR starts from an already-pretrained model whose attention layers exhibit heterogeneous per-head weight norms.
Such heterogeneity is functionally meaningful: per-head norms govern attention sharpness and gradient magnitudes (\textbf{Appendix\,\ref{sec:head_norm_analysis}}), so different heads naturally require updates at different scales.
However, both default-mode Pion and Muon apply NS iterations to each projection \textit{as a whole}, ignoring this per-head heterogeneity.
As a result, training becomes less effective, as shown in \textbf{Fig.\,\ref{fig:mh_analysis}-(a)}, where default-mode 
Pion underperforms AdamW and (default-mode) Muon collapses. We also observe that enabling the per-head mode for Muon does not improve performance, since the lack of noise adaptiveness (Limitation 2) remains the primary cause of its ineffectiveness in RLVR. The superior performance of per-head Pion suggests that spectral high-pass filtering is the primary driver of RLVR stability, while the per-head reshape serves as an auxiliary mechanism that preserves pretrained head structure.
To further justify per-head awareness in Pion, we analyze the Q projection sub-blocks $\{\mathbf{W}_{Q}^h\}_{h=1}^{H}$ across $H$ attention heads (\textbf{Fig.\,\ref{fig:mh_analysis}-(b)}). Let $\mathbf{W}_0$ and $\mathbf{W}_*$ denote the weights before and after RLVR, respectively. We measure per-head heterogeneity via the cross-head variance $\mathrm{Var}(\left\|\mathbf{W}_{0,Q}^h\right\|_{\mathrm{F}})$. Prior to RLVR, this variance is non-trivial across all 28 layers of Qwen3-1.7B (top). However, the update variance $\mathrm{Var}(\left\|\mathbf{W}_{*,Q}^h - \mathbf{W}_{0,Q}^h\right\|_{\mathrm{F}})$ under default-mode Pion is nearly flat (bottom), indicating uniform updates across heads that fail to reflect heterogeneity. In contrast, the \textit{per-head} mode reshapes projections along the head dimension, enabling heterogeneous, layer-dependent updates.

\section{Experiments}
\label{sec: exp}


\subsection{Experiment setups}
\label{sec: exp_setup}

\noindent \textbf{VLA setups.}
Two models are assessed: $\ell_1$-regression-based \textbf{VLA-Adapter} \citep{wang2026vla} and flow-matching-based \textbf{VLANeXt} \citep{wu2026vlanext}. Both are trained and tested on the four \textbf{LIBERO} suites \citep{liu2023libero}, with VLANeXt additionally evaluated on \textbf{LIBERO-Plus} \citep{fei2025libero}. {We further include a real-robot evaluation by finetuning \textbf{$\pi_{0.5}$} \citep{intelligence2025pi_} under the \textbf{DROID} setup \citep{khazatsky2025droid} on three grasp-and-place tasks.} We compare three optimizers: (i) \textbf{AdamW} globally; (ii) \textbf{Muon} on all 2D matrices (excluding embeddings/output layer), with AdamW elsewhere; and (iii) \textbf{Pion}, applying Pion to the action 2D matrices, Muon to vision/language 2D matrices (excluding embeddings/output layer), and AdamW elsewhere. Performance is measured by \textit{success rate} (\%). 

\noindent \textbf{RLVR setups.}
Experiments utilize \textbf{Qwen3-1.7B} and \textbf{Qwen3-4B} \citep{yang2025qwen3} optimized via \textbf{GRPO} \citep{shao2024deepseekmath} and \textbf{GMPO} \citep{zhao2025geometric}. Models are trained on \textbf{GSM8K} (training split) and \textbf{MATH} levels 3--5, and evaluated on the GSM8K test split \citep{cobbe2021training} and MATH500 \citep{hendrycks2021measuring}, respectively. Optimizer configurations mirror the VLA setups: (i) \textbf{AdamW}, (ii) \textbf{Muon}, and (iii) \textbf{Pion}, which adopts the per-head mode introduced in \textbf{Sec.\,\ref{sec: method}}. The evaluation metric is \textit{accuracy} (\%). 
See \textbf{Appendix\,\ref{sec:training_details}} for details.

\subsection{VLA experiment results}
\label{sec: exp_vla}

\begin{figure}[htb]
    \centering
    \vspace{-1mm}
    \begin{tabular}{cc}
        \multicolumn{2}{c}{\includegraphics[width=0.5\textwidth]{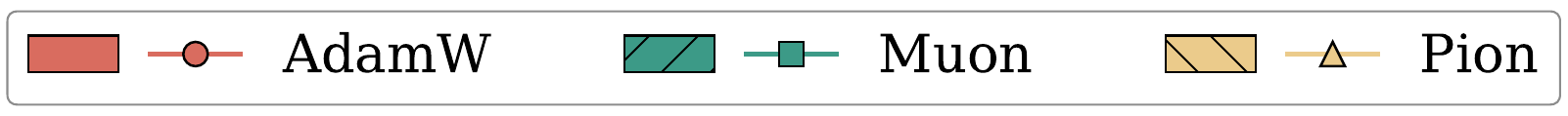}} \\[-1mm]
        \hspace*{-4mm}
        \includegraphics[width=0.455\textwidth]{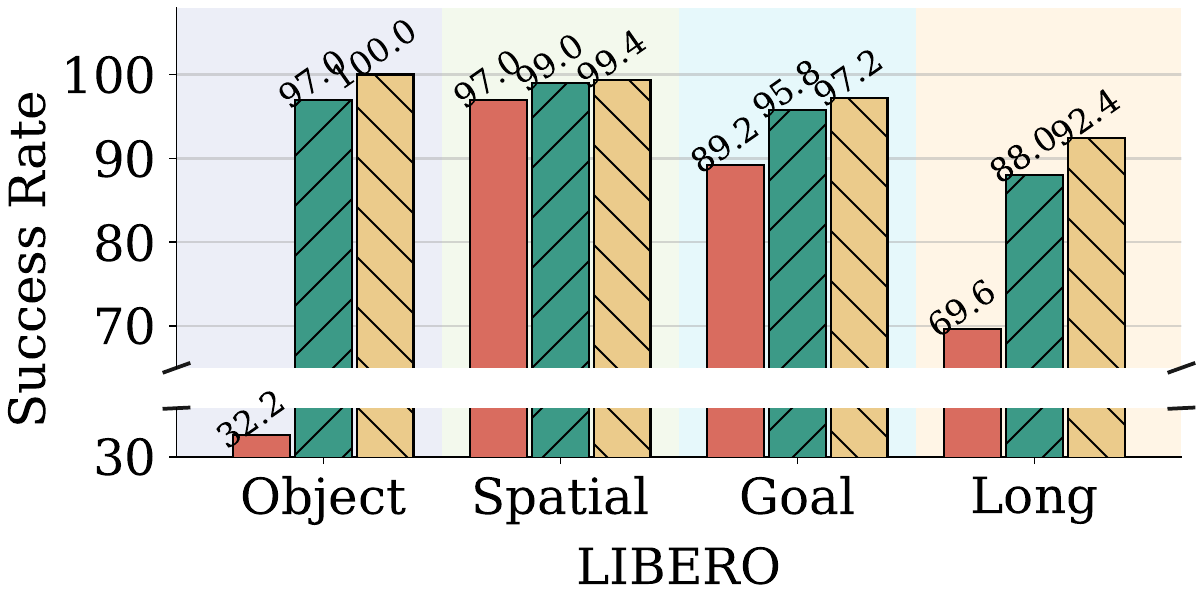}
        &
        \hspace*{-4mm}
        \includegraphics[width=0.223\textwidth]{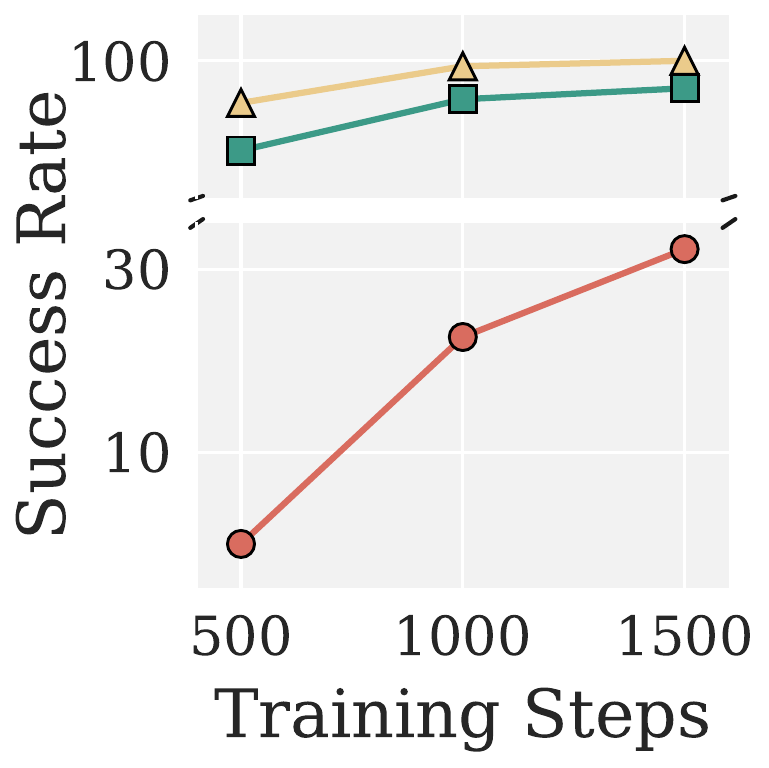}
        \\[-1.5mm]
        \small{(a) Overall Performance}
        &
        \hspace*{-2mm}\small{(b) Object}
    \end{tabular}
    \vspace{-2mm}
    \caption{\small{
    AdamW, Muon and Pion for VLA-Adapter on LIBERO. (a) Test success rates on LIBERO Object, Spatial, Goal and Long at the same training budget ($1{,}500$ steps for Object and $15{,}000$ steps for others). (b) Test success rates vs. training steps on Object.}}
    \label{fig:vlaadapter_main}
    \vspace{-4mm}
\end{figure}

\noindent \textbf{Advantages of Pion over Muon and AdamW for VLA-Adapter on LIBERO.} \textbf{Fig.\,\ref{fig:vlaadapter_main}} presents final success rates of VLA-Adapter on the four LIBERO task suites using AdamW, Muon, and Pion under a fixed budget per suite ($1{,}500$ steps for Object and $15{,}000$ steps for the others), along with learning curves for LIBERO 
Object. As shown in \textbf{Fig.\,\ref{fig:vlaadapter_main}-(a)}, Muon already outperforms AdamW on all four tasks, indicating that spectral steepest descent benefits multimodal training. Pion further improves over Muon on every task. This aligns with the spectral analysis in {Fig.\,\ref{fig:vla_optimizer_compare}}: the action module exhibits near-low-rank gradients, so Pion’s high-pass filter preserves informative singular directions while suppressing tail noise that Muon would otherwise amplify. Furthermore, \textbf{Fig.\,\ref{fig:vlaadapter_main}-(b)} shows that Pion reaches $95.4\%$ success at $500$ steps and saturates at $100\%$ by $1{,}500$ steps, while AdamW requires substantially more steps and Muon consistently lags behind, indicating that Pion’s spectral high-pass yields faster convergence on the action module. This also indicates that Pion improves training efficiency by requiring substantially fewer training steps to reach a high success-rate regime compared to AdamW and Muon.

\begin{table}[htbp]
    \vspace{-4mm}
    \centering
    \caption{\small{AdamW, Muon and Pion for VLANeXt on LIBERO and LIBERO-Plus. Columns 2--3: average test success rate on LIBERO/LIBERO-Plus; Columns 4--10: test success rate on LIBERO-Plus under different perturbations. Best score in each column is in \textbf{bold}.}}
    \vspace{1mm}
    \label{tab:libero_plus}
    \resizebox{\textwidth}{!}{%
    \begin{tabular}{ccc|ccccccc}
    \toprule
    \textbf{Optimizer} & \textbf{LIBERO} & \textbf{LIBERO-Plus} & \textbf{Background} & \textbf{Camera} & \textbf{Language} & \textbf{Layout} & \textbf{Light} & \textbf{Noise} & \textbf{Robot} \\
    \midrule
    AdamW& 79.45 & 64.57 & 68.97 & 70.38 & 54.50 & 61.80 & 76.35 & 66.37 & 47.04 \\
    Muon & 93.65 & 72.34 & 82.72 & 68.00 & 77.53 & 76.21 & 86.17 & 69.98 & 57.36 \\
    \midrule
    \rowcolor{blue!10}
    \textbf{Pion (Ours)} & \textbf{96.35} & \textbf{75.93} & \textbf{84.53} & \textbf{70.88} & \textbf{86.93} & \textbf{76.71} & \textbf{90.67} & \textbf{76.09} & \textbf{63.18} \\
    \bottomrule
    \end{tabular}}
\end{table}

\begin{wraptable}{r}{0.61\textwidth}
    \vspace{-4mm}
    \small
    \centering
    \caption{Qualitative LIBERO-Plus rollout (Object) of VLANeXt trained with AdamW, Muon, and Pion under the instruction \textit{``Grasp the container filled with a citrus-based beverage and deposit it into the woven holder designed.''}}
    \vspace{-1mm}
    \label{tab:vla_qualitative_main}
    \setlength{\tabcolsep}{3pt}
    \renewcommand{\arraystretch}{1.15}
    \begin{tabular}{@{}>{\centering\arraybackslash}m{0.115\textwidth}@{\hspace{4pt}}*{4}{>{\centering\arraybackslash}m{0.11\textwidth}}@{}}
    \toprule
    \multirow{2}{*}{\textbf{Optimizer}} & \multicolumn{4}{c}{\textbf{Frame index}} \\
    \cmidrule(lr){2-5}
    & \textbf{0} & \textbf{3} & \textbf{6} & \textbf{9} \\
    \midrule
    AdamW &
        \includegraphics[width=0.11\textwidth]{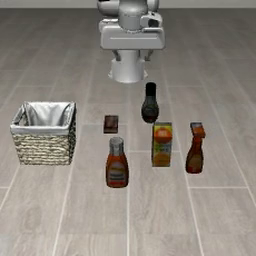} &
        \includegraphics[width=0.11\textwidth]{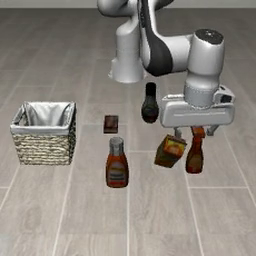} &
        \includegraphics[width=0.11\textwidth]{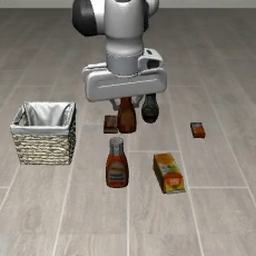} &
        \includegraphics[width=0.11\textwidth]{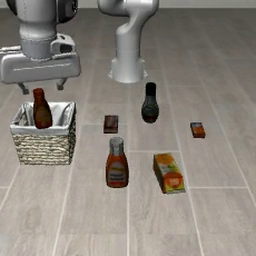} \\
    \midrule
    Muon &
        \includegraphics[width=0.11\textwidth]{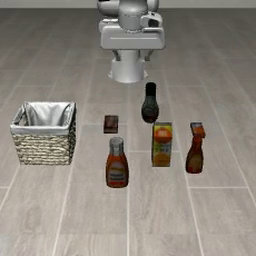} &
        \includegraphics[width=0.11\textwidth]{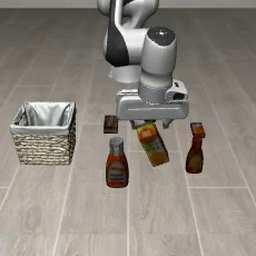} &
        \includegraphics[width=0.11\textwidth]{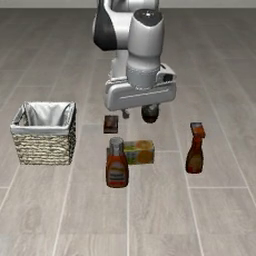} &
        \includegraphics[width=0.11\textwidth]{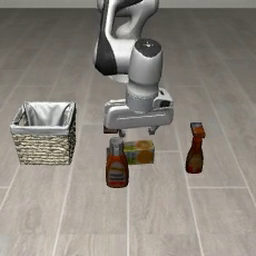} \\
    \midrule
    \textbf{Pion (Ours)} &
        \includegraphics[width=0.11\textwidth]{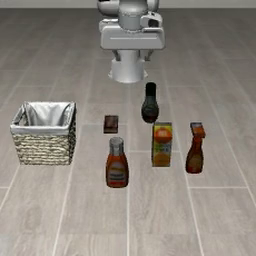} &
        \includegraphics[width=0.11\textwidth]{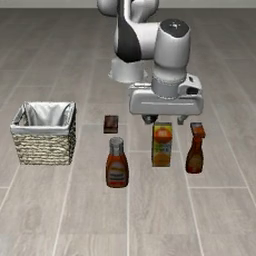} &
        \includegraphics[width=0.11\textwidth]{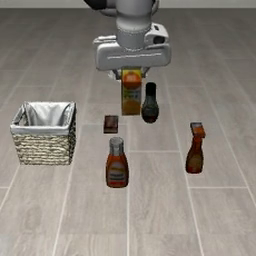} &
        \includegraphics[width=0.11\textwidth]{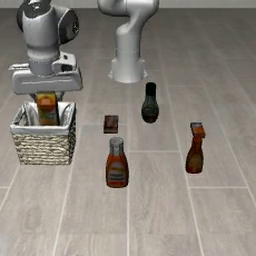} \\
    \bottomrule
    \end{tabular}
    \vspace{-2mm}
\end{wraptable}

\noindent \textbf{The Pion advantage extends to flow-matching VLAs and perturbed scenes.}
To validate that Pion’s benefit is not architecture-specific, we evaluate \textbf{VLANeXt} \citep{wu2026vlanext}, a flow-matching VLA. \textbf{Table\,\ref{tab:libero_plus}} reports success rates on LIBERO and LIBERO-Plus. The first two columns show task-averaged success rates, while the remaining columns break down performance under individual LIBERO-Plus perturbations. As shown, Pion consistently outperforms Muon and AdamW across all settings, confirming that the high-pass mechanism is \textit{model-agnostic} across both regression-based and flow-matching VLAs. Moreover, its advantage is \textit{preserved and amplified} on the more challenging LIBERO-Plus split, notably under Language (+9\%), Noise (+6\%), and Robot (+6\%) perturbations. This suggests that Pion yields more \textit{robust} policies under distribution shifts, tackling the limitation that Muon-style whitening could over-amplify non-generalizable noise directions. \textbf{Table~\ref{tab:vla_qualitative_main}} compares AdamW, Muon, and Pion on a LIBERO-Plus (Object) task (``Grasp the container filled with a citrus-based beverage and deposit it into the woven holder designed.''). \textit{AdamW} mis-grounds the instruction and grasps the wrong bottle. \textit{Muon} grasps the correct target but collides with a neighboring object during transport, corroborating that its uniform whitening over-amplifies noise and yields jittery trajectories (Sec.~\ref{sec: motivation}). \textit{Pion} alone succeeds, executing a clean, collision-free rollout. \textbf{Appendix\,\ref{sec: qualitative_examples}} provides additional examples on the four task suites.

\begin{table}[htbp]
    \centering
    \caption{\small{AdamW, Muon, and Pion on real-robot grasp-and-place tasks using $\pi_{0.5}$ \citep{intelligence2025pi_} backbone under the DROID setup \citep{khazatsky2025droid}. Each entry reports the success rate (\%) over $30$ randomized initial configurations. The best score in each column is in \textbf{bold}.}}
    \vspace{1mm}
    \label{tab:real_robot}
    \small
    \setlength{\tabcolsep}{4pt}
    \begin{tabular}{ccccc}
    \toprule
    \textbf{Optimizer} & \textbf{Cucumber\,$\to$\,Plate} & \textbf{Cube\,$\to$\,Plate} & \textbf{Cube\,$\to$\,Bowl} & \textbf{Average} \\
    \midrule
    AdamW & 40.0 & 33.3 & 20.0 & 31.1 \\
    Muon  & 56.7 & 33.3 & 26.7 & 38.9 \\
    \midrule
    \rowcolor{blue!10}
    \textbf{Pion (Ours)} & \textbf{93.3} & \textbf{83.3} & \textbf{80.0} & \textbf{85.6} \\
    \bottomrule
    \end{tabular}
   \vspace{-3mm}
\end{table}

\noindent \textbf{Real-robot evaluation.}
We further validate Pion on a physical robot by finetuning $\pi_{0.5}$ \citep{intelligence2025pi_} under the DROID setup \citep{khazatsky2025droid} on three grasp-and-place tasks (\textit{Cucumber\,$\to$\,Plate}, \textit{Cube\,$\to$\,Plate}, \textit{Cube\,$\to$\,Bowl}). All three optimizers are trained for the same $20{,}000$ steps under the same training dataset. \textbf{Table\,\ref{tab:real_robot}} reports the trial-level success rate over $30$ randomized trials per task. Pion sharply outperforms both baselines on every task, lifting the average success rate from 31.1\% (AdamW) and 38.9\% (Muon) to 85.6\%. 
Crucially, this substantial performance gain over AdamW and Muon is achieved under a low-budget VLA training regime consisting of only 20,000 training steps, which is much fewer than those typically used in standard AdamW-based VLA training.
This step-efficiency advantage mirrors the margin observed in simulation (Fig.\,\ref{fig:vlaadapter_main}-(b)), confirming that Pion's training-efficiency gain carries over from simulation to real hardware. We attribute this to Pion's high-pass spectral filtering on the action module, whose benefit is further amplified under the tighter precision tolerances of physical manipulation. Qualitative rollouts are provided in \textbf{Appendix\,\ref{sec:real_robot}} (\textbf{Table\,\ref{tab:real_robot_rollouts}}).

\noindent \textbf{Additional results.}
Three studies on VLA-Adapter (\textbf{Appendix\,\ref{sec: vla_ablations}}) show that (i) Pion outperforms LRMuon across all top-$k$ ranks while matching Muon's total training time (\textbf{Fig.\,\ref{fig:vlaadapter_rank_time}}); (ii) per-head Pion on the action head also beats Muon and AdamW but underperforms the default mode (\textbf{Table\,\ref{tab:vlaadapter_libero}}); and (iii) a modality-wise optimizer sweep prefers Muon on vision/language and Pion on action, validating our assignment (\textbf{Table\,\ref{tab:vla_modality_assignment}}).

\subsection{RLVR experiment results}
\label{sec: exp_rl}

\begin{figure*}[htbp]
    \centering
    \vspace*{-2mm}
    \begin{tabular}{cccc}
        \hspace*{-3mm}
        \includegraphics[width=0.21\textwidth]{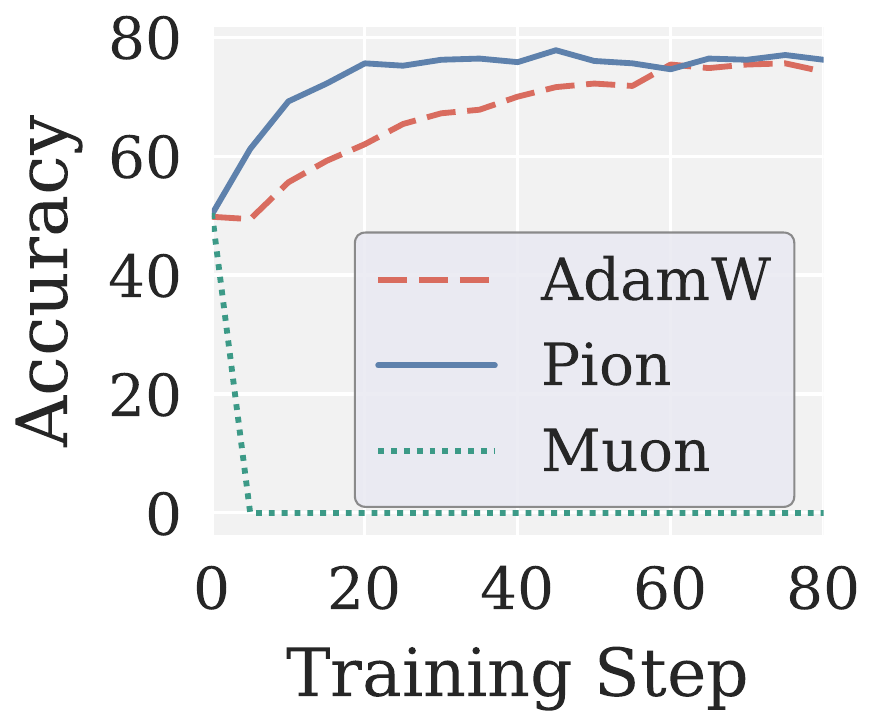}
        & \hspace*{-3mm}
        \includegraphics[width=0.21\textwidth]{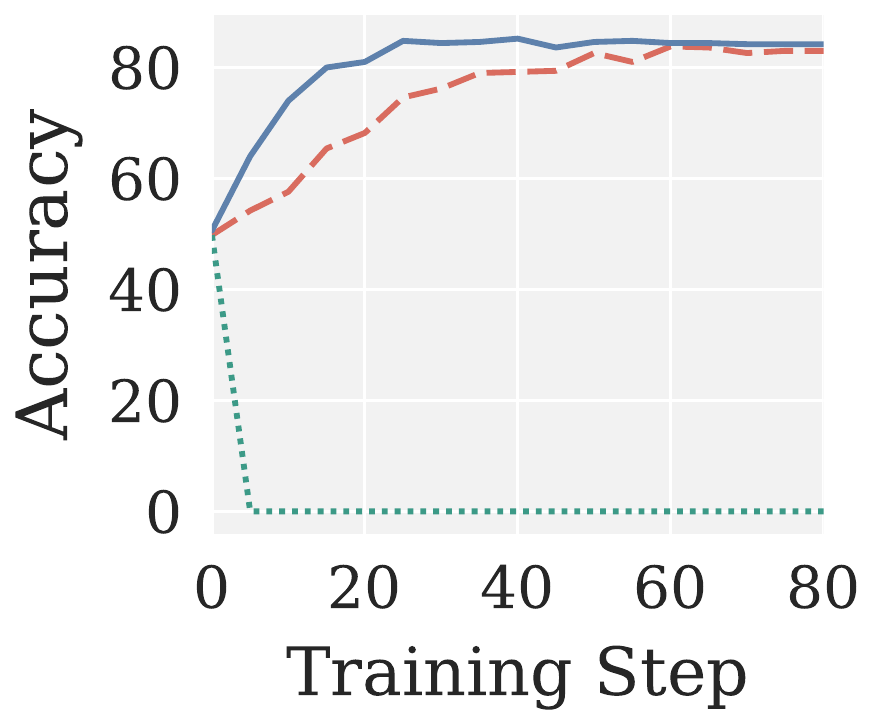}
        & \hspace*{-3mm}
        \includegraphics[width=0.21\textwidth]{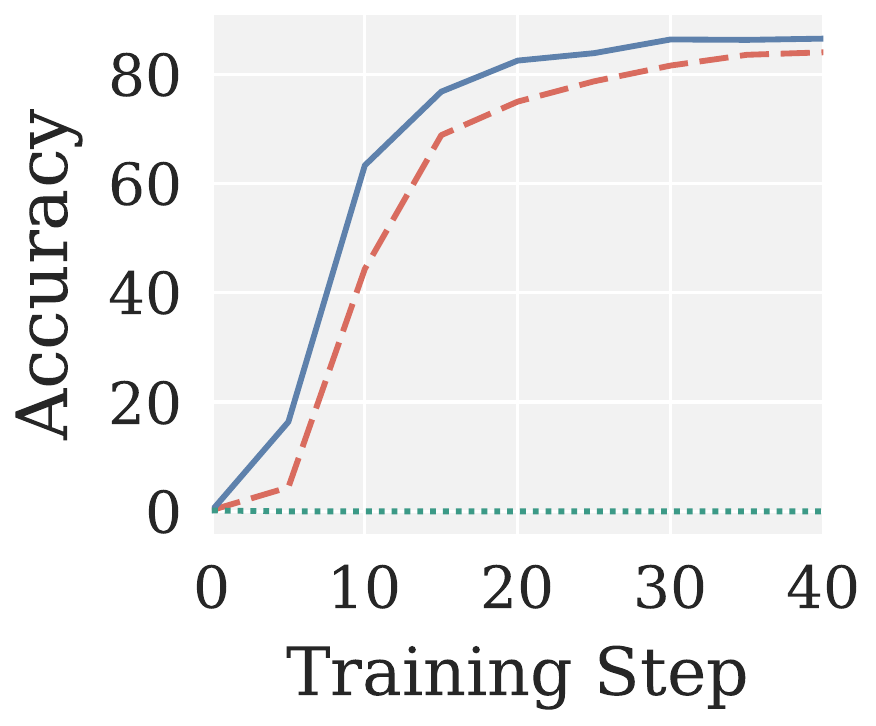}
        & \hspace*{-3mm}
        \includegraphics[width=0.21\textwidth]{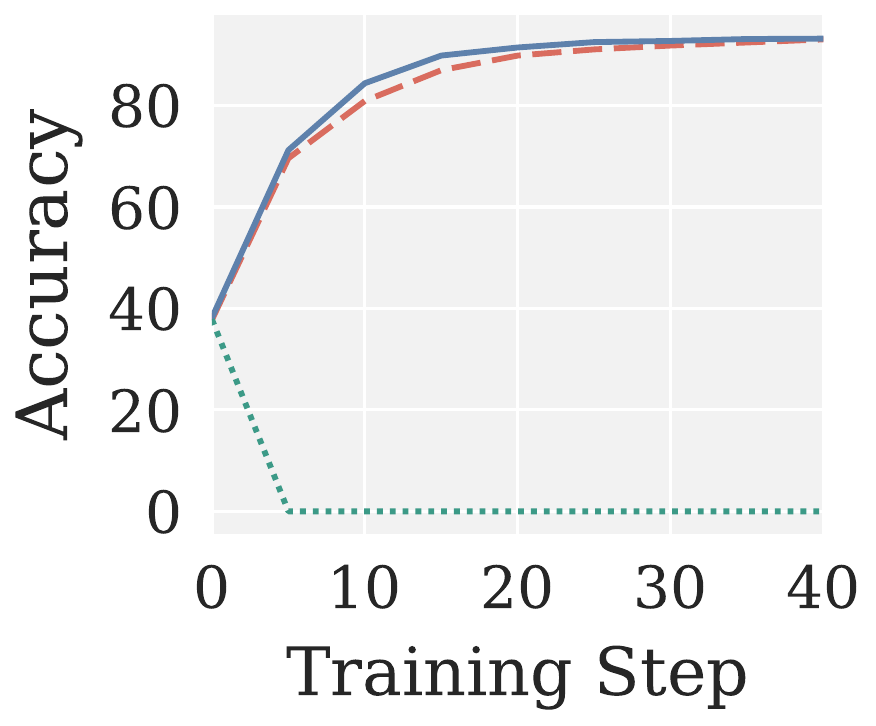}
        \\
        \hspace*{-3mm} \footnotesize{(a) GRPO, 1.7B, MATH}
        & \hspace*{-3mm} \footnotesize{(b) GRPO, 4B, MATH}
        & \hspace*{-3mm} \footnotesize{(c) GRPO, 1.7B, GSM8K}
        & \hspace*{-3mm} \footnotesize{(d) GRPO, 4B, GSM8K}
        \\[2mm]
        \hspace*{-3mm}
        \includegraphics[width=0.21\textwidth]{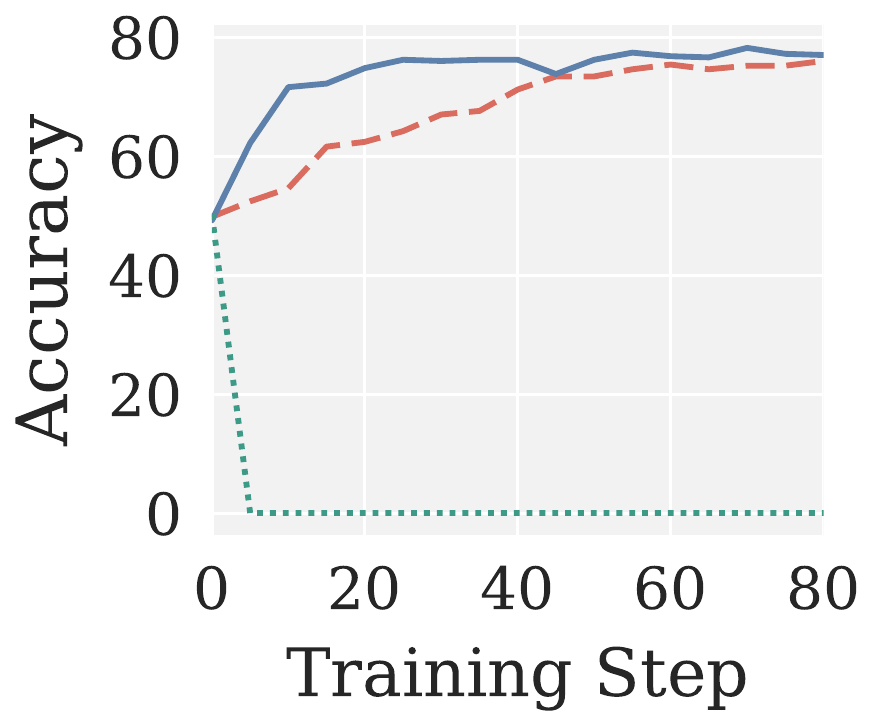}
        & \hspace*{-3mm}
        \includegraphics[width=0.21\textwidth]{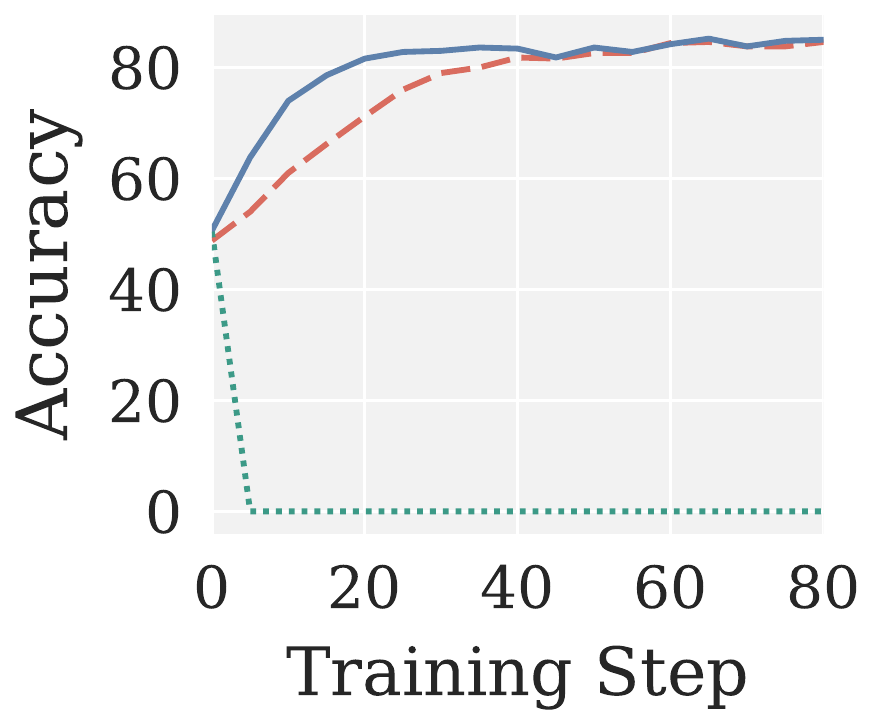}
        & \hspace*{-3mm}
        \includegraphics[width=0.21\textwidth]{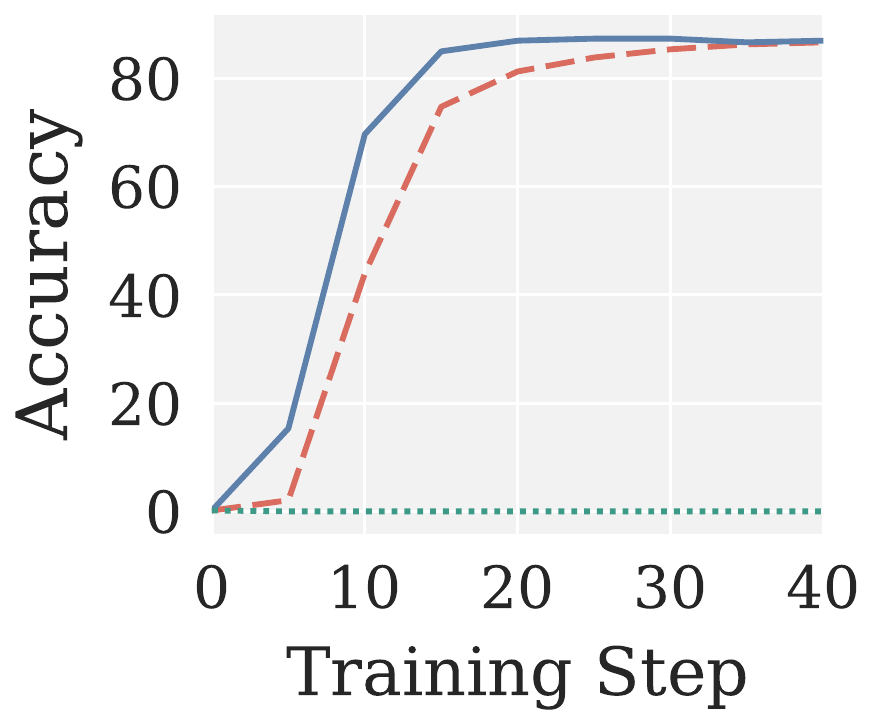}
        & \hspace*{-3mm}
        \includegraphics[width=0.21\textwidth]{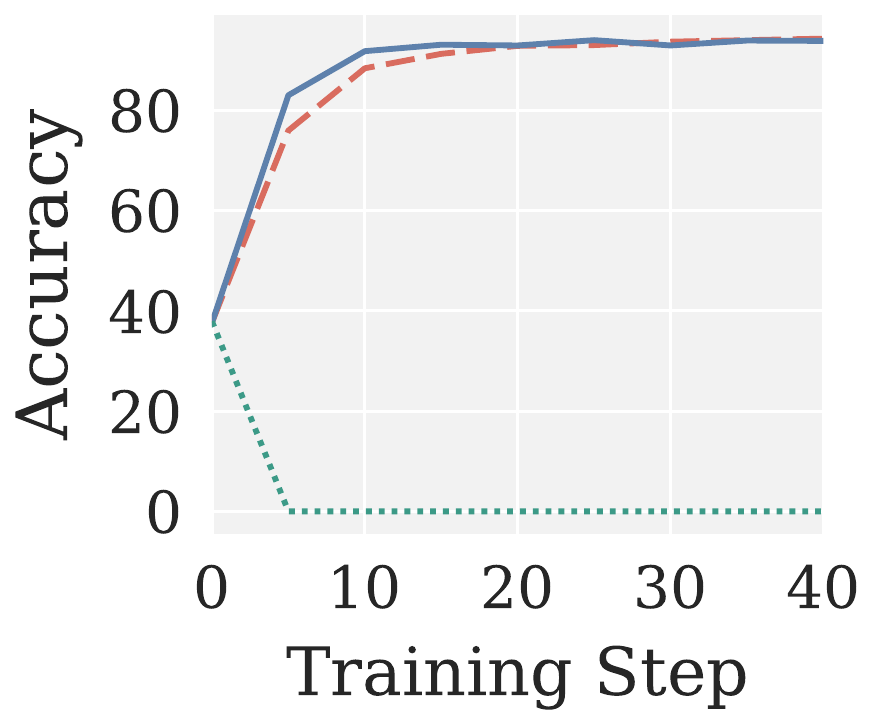}
        \\
        \hspace*{-3mm} \footnotesize{(e) GMPO, 1.7B, MATH}
        & \hspace*{-3mm} \footnotesize{(f) GMPO, 4B, MATH}
        & \hspace*{-3mm} \footnotesize{(g) GMPO, 1.7B, GSM8K}
        & \hspace*{-3mm} \footnotesize{(h) GMPO, 4B, GSM8K}
    \end{tabular}
    \vspace{-2mm}
    \caption{\small{AdamW, Muon and Pion on RLVR: validation accuracy vs. training step across eight settings, spanning two algorithms (GRPO, GMPO), two model sizes (Qwen3-1.7B, Qwen3-4B), and two benchmarks (MATH: train on levels 3--5 / evaluate on MATH500; GSM8K: train/test splits).}}
    \vspace{-3mm}
    \label{fig:rl_results}
\end{figure*}

\begin{wrapfigure}{r}{0.22\textwidth}
    \vspace*{-4mm}
    \centering
    \includegraphics[width=0.22\textwidth]{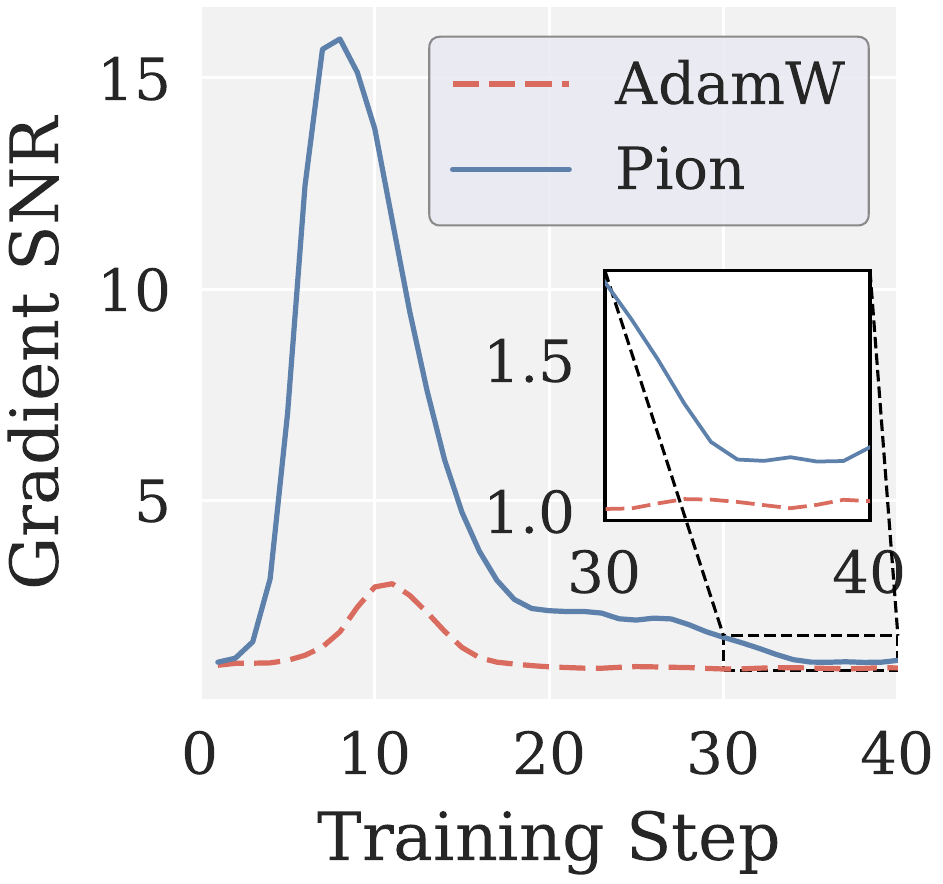}
    \vspace*{-5mm}
    \caption{\small{Gradient SNR of Pion vs. AdamW (Qwen3-1.7B, GRPO on GSM8K).}}
    \label{fig:rl_snr_results}
    \vspace*{-6mm}
\end{wrapfigure}

\noindent \textbf{Pion succeeds while Muon collapses.} \textbf{Fig.\,\ref{fig:rl_results}} shows validation accuracy vs.\ training steps across eight settings (GRPO/GMPO $\times$ Qwen3-1.7B/4B $\times$ MATH/GSM8K) using AdamW, Muon, and Pion. Muon consistently fails: accuracy remains near zero throughout training and often falls below the initial checkpoint. This aligns with our Limitation\,2 analysis in Sec.\,\ref{sec: motivation}: under low-SNR RLVR gradients, Muon’s uniform whitening amplifies noisy directions to the same magnitude as informative ones, leading to rapid policy collapse. In contrast, Pion recovers a meaningful training signal and \textit{outperforms} AdamW, as evidenced by faster convergence across all settings, demonstrating that spectral high-pass filtering is key to stable and effective RLVR. To further verify this, \textbf{Fig.\,\ref{fig:rl_snr_results}} shows that Pion consistently achieves higher SNR than AdamW throughout training.

\noindent \textbf{A reverse ablation: flipping the filter direction collapses on RLVR.}
To isolate that Pion’s gains stem specifically from its \textit{high-pass} NS design, we construct a low-pass counterpart, \textbf{Low-pass Muon} (LPMuon), as a direct mirror of Pion. LPMuon retains the same NS structure and per-step cost, but flips the coefficients to induce a {low-pass} mapping (contracting large singular values and amplifying small ones); see \textbf{Appendix\,\ref{sec:lowpass_muon}} for details. 
\textbf{Fig.\,\ref{fig:lowpassmuon}-(a)} confirms the resulting low-pass profile.
Yet, LPMuon fails to train: as shown in \textbf{Fig.\,\ref{fig:lowpassmuon}-(b)}, its accuracy remains at the initial checkpoint, in stark contrast to Pion. Together with Muon’s failure (no filtering) in Fig.\,\ref{fig:rl_results}, this reverse ablation isolates the \textit{direction} of spectral shaping as the key factor: Pion’s gains arise specifically from high-pass filtering.

\begin{figure}[htbp]
    \centering
    \begin{tabular}{cc}
        \hspace*{-4mm}
        \includegraphics[width=0.26\textwidth]{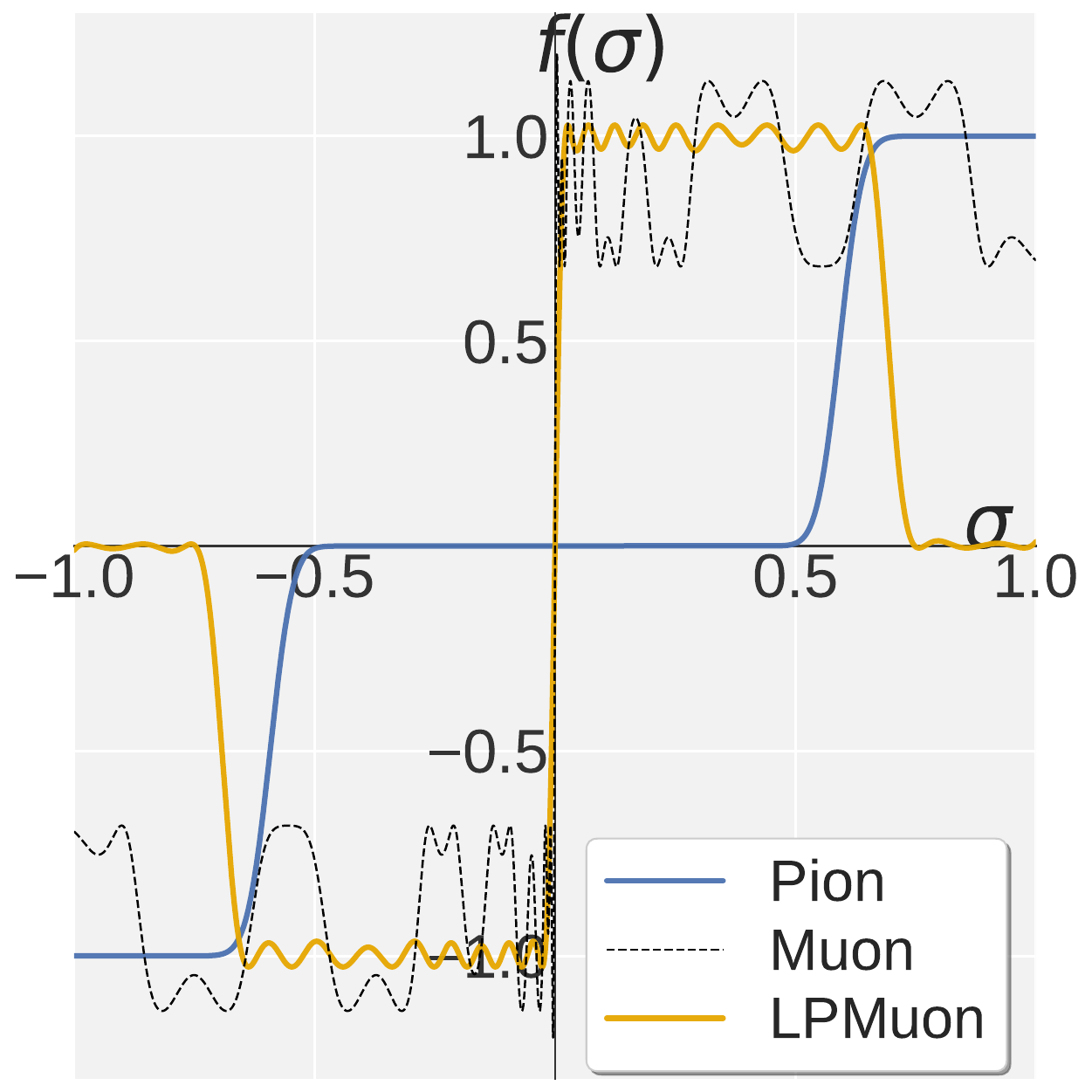}
        &
        \includegraphics[width=0.26\textwidth]{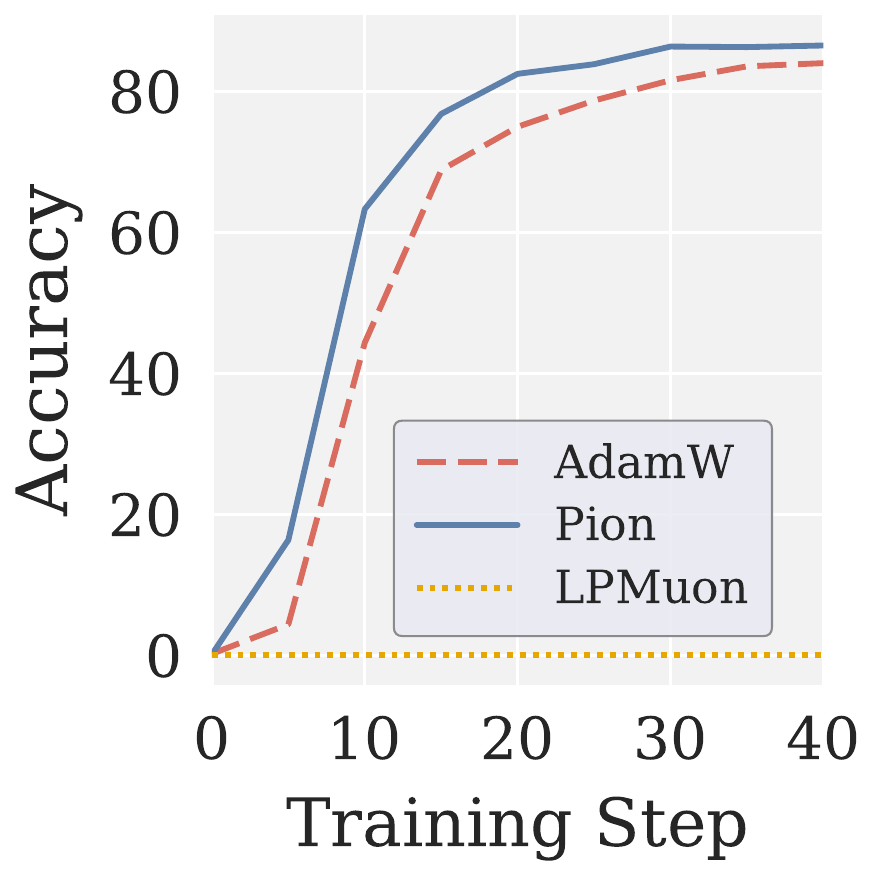}
        \\
        \hspace*{-2mm}\small{(a) Low-pass NS}
        &
        \small{(b) GSM8K accuracy}
    \end{tabular}
    \vspace{-2mm}
    \caption{\small{(a) Scalar map $f(\sigma)$ of LPMuon for $\sigma \in [0, 1]$. (b) Accuracy of AdamW, Pion, and LPMuon (Qwen3-1.7B, GRPO on GSM8K).}}
    \label{fig:lowpassmuon}
    \vspace{-3mm}
\end{figure}

\section{Conclusion}
\label{sec: conclusion}

We identified two limitations of Muon beyond LLM pretraining: lack of rank adaptiveness in cross-modality VLA training, and lack of noise adaptiveness in RLVR post-training. To address them, we proposed Pion, a drop-in replacement for Muon's NS iteration that uses a high-pass NS to preserve leading singular directions while suppressing the noisy tail, at the same per-step cost as Muon. Pion consistently outperforms AdamW and Muon across VLA training on LIBERO/LIBERO-Plus and RLVR post-training on Qwen3-1.7B/4B over MATH and GSM8K, including settings where Muon collapses. We discuss Pion's limitations (\textbf{Appendix\,\ref{sec: limitations}}) and broader impacts (\textbf{Appendix\,\ref{sec: broader_impact}}).

\section*{Acknowledgment}
This project is supported by the Cisco Faculty Research Award. The work of Chongyu Fan and Sijia Liu is also supported in part by the NSF CISE Core Program Award IIS-2504263, the NSF CAREER Award IIS-2338068, and the NSF Cyber-Physical Systems (CPS) Award CNS-2235231. We would also like to thank Gengyu Zhang for helpful discussions and feedback on the real-robot implementation of applying Pion to VLA training.

\bibliography{refs/vla, refs/rl, refs/muon}
\bibliographystyle{iclr2025_conference}


\clearpage
\newpage

\clearpage
\onecolumn
\section*{\Large{Appendix}}
\setcounter{section}{0}
\setcounter{figure}{0}
\setcounter{table}{0}
\makeatletter 
\renewcommand{\thesection}{\Alph{section}}
\renewcommand{\theHsection}{\Alph{section}}
\renewcommand{\thefigure}{A\arabic{figure}}
\renewcommand{\theHfigure}{A\arabic{figure}}
\renewcommand{\thetable}{A\arabic{table}}
\renewcommand{\theHtable}{A\arabic{table}}
\makeatother

\renewcommand{\thetable}{A\arabic{table}}
\setcounter{equation}{0}
\renewcommand{\theequation}{A\arabic{equation}}
\renewcommand{\theHequation}{A\arabic{equation}}

\etocsettocdepth.toc{subsection}
\etocsetnexttocdepth{subsection}
\etocsettocstyle{}{}
\tableofcontents
\clearpage

\section{Additional Preliminaries: VLA  Training and RLVR Training}
\label{sec:prelim_details}

This appendix provides the formal definitions deferred from Sec.\,\ref{sec: preliminary}: the two representative VLA action heads ($\ell_1$-regression and flow-matching) used as our cross-modality testbeds, and the two representative RLVR algorithms (GRPO and GMPO) used for our post-training experiments.

\subsection{VLA action heads and training objectives}
\label{sec:prelim_vla_heads}

We consider two representative designs for the action head of a VLA policy, each instantiating a different way of modeling the action distribution conditioned on the multimodal input $(\mathbf{x}, \mathbf{c})$.
\begin{itemize}[leftmargin=*, topsep=0pt, itemsep=0pt]
    \item 
    \textit{$\ell_1$-regression head} \citep{wang2026vla,kim2025fine}: A deterministic transformer maps the multimodal input $(\mathbf{x}, \mathbf{c})$ to a single action prediction $f_{\boldsymbol{\Theta}}(\mathbf{x}, \mathbf{c})$, trained with
    \begin{align}
        \mathcal{L}_{\mathrm{reg}}(\boldsymbol{\Theta}) \;=\; \mathbb{E}_{(\mathbf{x}, \mathbf{c}, \mathbf{a}) \sim \mathcal{D}}\bigl[\,\|f_{\boldsymbol{\Theta}}(\mathbf{x}, \mathbf{c}) - \mathbf{a}\|_1\,\bigr],
        \label{eq: l1_loss}
    \end{align}
    where $\| \cdot \|_1$ denotes the $\ell_1$ norm.
    \item \textit{Flow-matching head} \citep{lipman2022flow,black2024pi_0}: Rather than producing a single point estimate, the action head models the conditional distribution $p(\mathbf{a} | \mathbf{x}, \mathbf{c})$ via a continuous-time generative process that transports a Gaussian prior to the data action. Concretely, let $\mathbf{a}_1 \Def \mathbf{a}$ be the ground-truth action drawn from $\mathcal{D}$ and $\mathbf{a}_0 \sim \mathcal{N}(\mathbf{0}, \mathbf{I})$ be a noise sample. Along the linear interpolation path $\mathbf{a}_t = t\,\mathbf{a}_1 + (1-t)\,\mathbf{a}_0$ for $t \in [0,1]$, the target velocity field is the constant displacement $\tfrac{d \mathbf{a}_t}{dt} = \mathbf{a}_1 - \mathbf{a}_0$. The action head parameterizes a conditional velocity field $v_{\boldsymbol{\Theta}_{\mathrm{action}}}(\mathbf{a}_t, t | \mathbf{x}, \mathbf{c})$ that predicts this velocity, and is trained to regress the target via
    \begin{align}
        \mathcal{L}_{\mathrm{FM}}(\boldsymbol{\Theta}) \;=\; \mathbb{E}_{t \sim \mathcal{U}(0,1),\, \mathbf{a}_0 \sim \mathcal{N}(\mathbf{0}, \mathbf{I}),\, (\mathbf{x},\mathbf{c},\mathbf{a}_1) \sim \mathcal{D}}\Bigl[\,\bigl\|\,v_{\boldsymbol{\Theta}_{\mathrm{action}}}(\mathbf{a}_t, t | \mathbf{x}, \mathbf{c}) - (\mathbf{a}_1 - \mathbf{a}_0)\bigr\|_2^2\,\Bigr],
        \label{eq: fm_loss}
    \end{align}
    where $t \sim \mathcal{U}(0,1)$ denotes the uniform distribution over the interpolation timestep.
\end{itemize}
In our experiments (\textbf{Sec.\,\ref{sec: exp}}), the $\ell_1$-regression head is instantiated by VLA-Adapter \citep{wang2026vla} and the flow-matching head by VLANeXt \citep{wu2026vlanext}.

\subsection{RLVR training: GRPO and GMPO}
\label{sec:prelim_rlvr_algs}

We expand here on the three-stage RLVR loop sketched in Sec.\,\ref{sec: preliminary}.
At each iteration, RLVR alternates between
(a) \textit{rollout}: for each prompt $\mathbf{q}$, a group of $g$ responses $\{\mathbf{o}_i\}_{i=1}^{g}$ is sampled from the old policy $\pi_{\mathrm{old}}$;
(b) \textit{scoring}: each response $\mathbf{o}_i$ is assigned a scalar reward by a programmatic verifier, and the rewards within a group are normalized into a \textit{group-relative advantage} $\hat{a}_i \in \mathbb{R}$;
(c) \textit{policy update}: $\pi_{\boldsymbol{\Theta}}$ is optimized through a clipped importance-ratio objective.
We study two representative policy-gradient algorithms, GRPO and GMPO, which differ in how they aggregate the per-token importance ratio $r_{i,t}$. Throughout, we denote by $\clip(x, l, u) \Def \min\bigl(\max(x, l),\, u\bigr)$ the standard clipping operator that confines a scalar $x \in \mathbb{R}$ to the interval $[l, u]$.

\begin{itemize}[leftmargin=*, topsep=0pt, itemsep=0pt]
    \item 
\textit{GRPO} \citep{shao2024deepseekmath} aggregates the ratio at the token level via arithmetic averaging $r_{i,t}(\boldsymbol{\Theta}) \Def \pi_{\boldsymbol{\Theta}}(o_{i,t} | \mathbf{q}, \mathbf{o}_{i,<t}) \,/\, \pi_{\mathrm{old}}(o_{i,t} | \mathbf{q}, \mathbf{o}_{i,<t})$, where $o_{i,t}$ denotes the $t$-th token of $\mathbf{o}_i$ and $\mathbf{o}_{i,<t}$ its preceding prefix:
\begin{align}
    \mathcal{J}_{\mathrm{GRPO}}(\boldsymbol{\Theta}) = \mathbb{E}_{\mathbf{q},\,\{\mathbf{o}_i\}}\!\left[\frac{1}{g}\sum_{i=1}^{g} \frac{1}{|\mathbf{o}_i|}\sum_{t=1}^{|\mathbf{o}_i|} \min\Bigl(\,r_{i,t}(\boldsymbol{\Theta})\,\hat{a}_i,\;\, \clip\bigl(r_{i,t}(\boldsymbol{\Theta}),\; 1-\epsilon,\; 1+\epsilon\bigr)\,\hat{a}_i\,\Bigr)\right].
    \label{eq:grpo_obj}
\end{align}
\item
\textit{GMPO} \citep{zhao2025geometric} replaces the token-level arithmetic mean with a sequence-level geometric mean. Denoting the sequence product $p_i(\boldsymbol{\Theta}) \Def \prod_{t=1}^{|\mathbf{o}_i|} r_{i,t}(\boldsymbol{\Theta})$, GMPO optimizes
\begin{align}
    \mathcal{J}_{\mathrm{GMPO}}(\boldsymbol{\Theta}) = \mathbb{E}_{\mathbf{q},\,\{\mathbf{o}_i\}}\!\left[\frac{1}{g}\sum_{i=1}^{g} \Bigl|\,\min\Bigl(\,p_i(\boldsymbol{\Theta})\,\hat{a}_i,\;\, \clip\bigl(p_i(\boldsymbol{\Theta}),\; 1-\epsilon,\; 1+\epsilon\bigr)\,\hat{a}_i\,\Bigr)\,\Bigr|^{1/|\mathbf{o}_i|} \cdot \sign(\hat{a}_i)\right].
    \label{eq:gmpo_obj}
\end{align}
\end{itemize}

\clearpage
\newpage
\section{Low-rank Muon (LRMuon) Algorithm}
\label{sec:svdmuon}

We provide the full pseudocode for \textbf{Low-rank Muon} (LRMuon), used as a baseline in Sec.\,\ref{sec: motivation} and Sec.\,\ref{sec: exp_vla}. LRMuon follows the standard Muon optimization loop \eqref{eq:Muon_basic}, but replaces the NS approximation to $\mathrm{msign}(\mathbf{M}_t)$ \eqref{eq:msign} with an exact SVD-based \textbf{top-$k$ polar factor}. Concretely, given the compact SVD $\mathbf{M}_t = \mathbf{U}\boldsymbol{\Sigma}\mathbf{V}^\top$, LRMuon truncates to the top-$k$ singular subspace $(\mathbf{U}_k, \mathbf{V}_k)$ and uses the partial-isometry update $\mathbf{U}_k\mathbf{V}_k^\top$. The full procedure is summarized in Alg.\,\ref{alg:svdmuon_optimizer}.

\begin{algorithm}[H]
\caption{LRMuon Optimizer}
\label{alg:svdmuon_optimizer}
\begin{algorithmic}[1]
\REQUIRE Learning rate $\eta$, momentum coefficient $\mu$, target rank $k$
\STATE $\mathbf{M}_0 \leftarrow \mathbf{0}$
\FOR{$t = 1, 2, \dots$}
    \STATE $\mathbf{G}_t \leftarrow \nabla_{\boldsymbol{\Theta}} \mathcal{L}_t(\boldsymbol{\Theta}_{t-1})$;\quad $\mathbf{M}_t \leftarrow \mu\, \mathbf{M}_{t-1} + \mathbf{G}_t$ \quad ($\mathbf{M}_t \in \mathbb{R}^{m \times n}$)
    \STATE $\mathbf{U},\, \boldsymbol{\Sigma},\, \mathbf{V}^\top \leftarrow \mathrm{SVD}(\mathbf{M}_t)$
    \STATE $k_{\mathrm{eff}} \leftarrow \min\bigl(k,\, \mathrm{rank}(\mathbf{M}_t)\bigr)$
    \STATE $\mathbf{U}_k \leftarrow \mathbf{U}_{:,\, 1:k_{\mathrm{eff}}}$;\quad $\mathbf{V}_k^\top \leftarrow (\mathbf{V}^\top)_{1:k_{\mathrm{eff}},\, :}$
    \STATE $\mathbf{X} \leftarrow \mathbf{U}_k \mathbf{V}_k^\top$
    \STATE $\boldsymbol{\Theta}_t \leftarrow \boldsymbol{\Theta}_{t-1} - \eta\, \mathbf{X}$
\ENDFOR
\RETURN $\boldsymbol{\Theta}_t$
\end{algorithmic}
\end{algorithm}

\clearpage
\newpage
\section{SNR Analysis for SFT and RLVR}
\label{sec:snr}

This appendix justifies the empirical observation in Sec.\,\ref{sec: motivation} that RLVR has a much lower gradient SNR than SFT. We derive closed-form expressions for the per-step SNR of both estimators under matched batch size, and then account for the additional noise sources that are unique to RLVR.

\subsection{Setup and gradient estimators}

We adopt the notation of Sec.\,\ref{sec: preliminary} and Appendix\,\ref{sec:prelim_rlvr_algs}. For a prompt $\mathbf{q}$, the old policy $\pi_{\mathrm{old}}$ produces a group of $g$ responses $\{\mathbf{o}_i\}_{i=1}^{g}$ with lengths $|\mathbf{o}_i|$, and a verifier assigns binary rewards $R_i \in \{0, 1\}$. Throughout, denote
\begin{align}
\ell_{i,t}(\boldsymbol{\Theta}) = \log\pi_{\boldsymbol{\Theta}}(o_{i,t}\mid \mathbf{q}, \mathbf{o}_{i,<t}),
\qquad
r_{i,t}(\boldsymbol{\Theta}) = \frac{\pi_{\boldsymbol{\Theta}}(o_{i,t}\mid \mathbf{q}, \mathbf{o}_{i,<t})}{\pi_{\mathrm{old}}(o_{i,t}\mid \mathbf{q}, \mathbf{o}_{i,<t})},
\qquad
\hat{a}_i = \frac{R_i - \bar R}{\std(R) + \epsilon},
\label{eq:snr_defs}
\end{align}
where $\bar R = \tfrac{1}{g}\sum_j R_j$ and $\std(R)^2 = \tfrac{1}{g}\sum_j (R_j - \bar R)^2$. Throughout this appendix, for a random vector $\mathbf{X}$ we write $\Var(\mathbf{X}) \Def \mathbb{E}\|\mathbf{X} - \mathbb{E}[\mathbf{X}]\|^2 = \mathrm{tr}(\mathrm{Cov}(\mathbf{X}))$ for its total scalar variance, which coincides with the Frobenius-based denominator of the main-text SNR \eqref{eq: snr} once the gradient matrix is vectorized. Two identities that we invoke repeatedly follow directly from \eqref{eq:snr_defs}:
\begin{align}
\sum_{i=1}^{g} \hat{a}_i \;=\; 0,
\qquad
\frac{1}{g}\sum_{i=1}^{g} \hat{a}_i^2 \;=\; \frac{\std(R)^2}{(\std(R)+\epsilon)^2} \;\approx\; 1.
\label{eq:snr_advantage_identities}
\end{align}

\paragraph{SFT estimator.} On a labelled pair $(\mathbf{q}, \mathbf{o}^\star)$ of length $T$, the per-sample loss is $-\log\pi_{\boldsymbol{\Theta}}(\mathbf{o}^\star\mid \mathbf{q}) = -\sum_{t=1}^{T} \ell_t(\boldsymbol{\Theta})$, so the batch estimator over $g$ i.i.d.\ examples is
\begin{align}
\hat{\mathbf{g}}_{\mathrm{SFT}}
\;=\; -\frac{1}{g}\sum_{j=1}^{g}\sum_{t=1}^{T} \nabla_{\boldsymbol{\Theta}} \ell_{j,t}(\boldsymbol{\Theta}),
\label{eq:snr_sft_grad}
\end{align}
with a \emph{deterministic} coefficient $-1$ on every token.

\paragraph{GRPO estimator.} Differentiating the GRPO objective \eqref{eq:grpo_obj} inside the unclipped branch of the $\min$, and applying the log-derivative identity $\nabla_{\boldsymbol{\Theta}} r_{i,t}(\boldsymbol{\Theta}) = r_{i,t}(\boldsymbol{\Theta})\,\nabla_{\boldsymbol{\Theta}} \ell_{i,t}(\boldsymbol{\Theta})$, gives
\begin{align}
\nabla_{\boldsymbol{\Theta}} \mathcal{J}_{\mathrm{GRPO}}(\boldsymbol{\Theta})
\;=\; \mathbb{E}_{\mathbf{q},\{\mathbf{o}_i\}}\!\left[\frac{1}{g}\sum_{i=1}^{g}\frac{1}{|\mathbf{o}_i|}\sum_{t=1}^{|\mathbf{o}_i|} \mathbb{1}_{i,t}\,\hat{a}_i\,r_{i,t}(\boldsymbol{\Theta})\,\nabla_{\boldsymbol{\Theta}} \ell_{i,t}(\boldsymbol{\Theta}) \right],
\label{eq:snr_grpo_full}
\end{align}
where $\mathbb{1}_{i,t}\in\{0,1\}$ is the active indicator picking out tokens for which the $\min$ in \eqref{eq:grpo_obj} selects the unclipped branch. In the \emph{on-policy regime} $\boldsymbol{\Theta} = \boldsymbol{\Theta}_{\mathrm{old}}$ we have $r_{i,t}\equiv 1$ and $\mathbb{1}_{i,t}\equiv 1$, so \eqref{eq:snr_grpo_full} reduces to
\begin{align}
\hat{\mathbf{g}}_{\mathrm{GRPO}}
\;=\; \frac{1}{g}\sum_{i=1}^{g} \hat{a}_i\, \bar{\mathbf{S}}_i,
\qquad
\bar{\mathbf{S}}_i \;\Def\; \frac{1}{|\mathbf{o}_i|}\sum_{t=1}^{|\mathbf{o}_i|} \nabla_{\boldsymbol{\Theta}} \ell_{i,t}(\boldsymbol{\Theta}).
\label{eq:snr_grpo_grad}
\end{align}

\paragraph{Regularity assumptions.} For notational simplicity we treat all responses as having a common representative length $T$ (the lengths $|\mathbf{o}_i|$ are replaced by $T$ throughout; the analysis goes through when $T$ is interpreted as the average length, as long as $\|\bar{\mathbf{s}}\|^2 \ll \sigma_s^2$). We assume throughout that (i) per-token scores have constant variance $\Var(\nabla_{\boldsymbol{\Theta}} \ell_{i,t}) = \sigma_s^2$ across token positions and are uncorrelated across time steps; (ii) the rewards are i.i.d., $R_i\sim\Bern(p)$ with $p = p(\mathbf{q})\in(0,1)$; and (iii) conditional on $\{R_i\}$, the trajectories $\{\mathbf{o}_i\}$ are independent and the residual $\bar{\mathbf{S}}_i - \mathbb{E}[\bar{\mathbf{S}}_i \mid R_i]$ has mean zero with second moment $\sigma_s^2/T$.

\subsection{SFT variance and SNR}

\paragraph{Signal.} Taking expectation in \eqref{eq:snr_sft_grad} with $\bar{\mathbf{s}} \Def \mathbb{E}[\nabla_{\boldsymbol{\Theta}} \ell_{i,t}]$,
\begin{align}
\mathbb{E}[\hat{\mathbf{g}}_{\mathrm{SFT}}]
\;=\; -\frac{1}{g}\cdot g \cdot T\,\bar{\mathbf{s}}
\;=\; -T\,\bar{\mathbf{s}},
\qquad
\|\mathbb{E}[\hat{\mathbf{g}}_{\mathrm{SFT}}]\|^2
\;=\; T^2\|\bar{\mathbf{s}}\|^2.
\label{eq:snr_sft_signal}
\end{align}

\paragraph{Variance.} Samples are independent and tokens within a sample are uncorrelated, so
\begin{align}
\Var(\hat{\mathbf{g}}_{\mathrm{SFT}})
\;=\; \frac{1}{g^2}\sum_{j=1}^{g}\Var\!\Bigl(\sum_{t=1}^{T}\nabla_{\boldsymbol{\Theta}} \ell_{j,t}\Bigr)
\;=\; \frac{1}{g}\cdot \sum_{t=1}^{T}\Var(\nabla_{\boldsymbol{\Theta}} \ell_{1,t})
\;=\; \frac{T\,\sigma_s^2}{g},
\label{eq:snr_sft_var}
\end{align}
where the last equality treats $\sigma_s^2$ as the per-token noise scale. Combining \eqref{eq:snr_sft_signal}--\eqref{eq:snr_sft_var},
\begin{align}
\mathrm{SNR}_{\mathrm{SFT}}
\;\Def\; \frac{\|\mathbb{E}[\hat{\mathbf{g}}_{\mathrm{SFT}}]\|^2}{\Var(\hat{\mathbf{g}}_{\mathrm{SFT}})}
\;=\; \frac{T^2\|\bar{\mathbf{s}}\|^2}{T\sigma_s^2/g}
\;=\; g\,T\,\frac{\|\bar{\mathbf{s}}\|^2}{\sigma_s^2}.
\label{eq:snr_sft}
\end{align}

\subsection{GRPO variance and SNR (on-policy)}

To isolate the reward-dependent part of $\bar{\mathbf{S}}_i$, decompose
\begin{align}
\bar{\mathbf{S}}_i \;=\; \underbrace{\mathbb{E}[\bar{\mathbf{S}}_i \mid R_i]}_{\mathbf{u}_i} \;+\; \underbrace{\bar{\mathbf{S}}_i - \mathbb{E}[\bar{\mathbf{S}}_i \mid R_i]}_{\mathbf{v}_i},
\qquad
\mathbf{u}_i = \boldsymbol{\mu}_S^- + R_i\,\boldsymbol{\Delta},
\label{eq:snr_decomp}
\end{align}
where $\boldsymbol{\mu}_S^+ \Def \mathbb{E}[\bar{\mathbf{S}}_i \mid R_i = 1]$, $\boldsymbol{\mu}_S^- \Def \mathbb{E}[\bar{\mathbf{S}}_i \mid R_i = 0]$, and $\boldsymbol{\Delta} \Def \boldsymbol{\mu}_S^+ - \boldsymbol{\mu}_S^-$ is the expected score gap between successful and failed trajectories. Assumption (iii) gives $\mathbb{E}\|\mathbf{v}_i\|^2 = \sigma_s^2/T$.

\paragraph{Signal.} Substituting \eqref{eq:snr_decomp} into \eqref{eq:snr_grpo_grad},
\begin{align}
\hat{\mathbf{g}}_{\mathrm{GRPO}}
\;=\; \frac{1}{g}\sum_{i=1}^{g} \hat{a}_i\,\mathbf{u}_i \;+\; \frac{1}{g}\sum_{i=1}^{g} \hat{a}_i\,\mathbf{v}_i.
\label{eq:snr_grpo_split}
\end{align}
Using $\sum_i \hat{a}_i = 0$ from \eqref{eq:snr_advantage_identities}, the reward-dependent term becomes
\begin{equation}
    \frac{1}{g}\sum_{i=1}^{g} \hat{a}_i \mathbf{u}_i
    \;=\;
    \left(\frac{1}{g}\sum_{i=1}^{g} \hat{a}_i R_i\right)\boldsymbol{\Delta}.
\end{equation}
For finite group size, this coefficient depends on the number of successful responses
$K \Def \sum_{i=1}^{g} R_i$. Ignoring the small $\epsilon$ in the normalization, degenerate groups with $K \in \{0,g\}$ have zero advantage and hence contribute no signal. For non-degenerate groups, $1 \leq K \leq g-1$,
\begin{equation}
    \frac{1}{g}\sum_{i=1}^{g} \hat{a}_i R_i
    \;=\;
    \sqrt{\frac{K}{g}\left(1-\frac{K}{g}\right)}.
\end{equation}
Thus, with $K\sim\mathrm{Binomial}(g,p)$ and
\begin{equation}
    q_{\mathrm{nd}} \Def \Pr(0<K<g)=1-p^g-(1-p)^g,
    \qquad
    \rho_g(p) \Def
    \mathbb{E}\!\left[
    \sqrt{\frac{K}{g}\left(1-\frac{K}{g}\right)}
    \,\middle|\, 0<K<g
    \right],
\end{equation}
the first term has expectation $q_{\mathrm{nd}}\rho_g(p)\boldsymbol{\Delta}$, while the second term has zero expectation by assumption (iii). In the large-$g$ regime with $p$ bounded away from $0$ and $1$, $q_{\mathrm{nd}}\to1$ and $\rho_g(p)\to\sqrt{p(1-p)}$, recovering the simpler approximation used in the main text. Therefore
\begin{align}
\mathbb{E}[\hat{\mathbf{g}}_{\mathrm{GRPO}}] \;\approx\; q_{\mathrm{nd}}\rho_g(p)\,\boldsymbol{\Delta},
\qquad
\|\mathbb{E}[\hat{\mathbf{g}}_{\mathrm{GRPO}}]\|^2 \;\approx\; q_{\mathrm{nd}}^2\rho_g(p)^2\,\|\boldsymbol{\Delta}\|^2.
\label{eq:snr_grpo_signal}
\end{align}

\paragraph{Variance.} The first term in \eqref{eq:snr_grpo_split} contributes $O(\|\boldsymbol{\Delta}\|^2/g)$, while the second contributes $O(\sigma_s^2/(gT))$; in the low-SNR regime $T\|\boldsymbol{\Delta}\|^2 \ll \sigma_s^2$ the second term dominates. Conditioning on $\{R_i\}$ (which fixes the $\hat{a}_i$) and using assumption (iii),
\begin{align}
\mathbb{E}\!\left\|\frac{1}{g}\sum_{i=1}^{g} \hat{a}_i\,\mathbf{v}_i\right\|^2
\;=\; \frac{1}{g^2}\sum_{i=1}^{g} \hat{a}_i^2\,\mathbb{E}\|\mathbf{v}_i\|^2
\;=\; \frac{\sigma_s^2}{g\,T}\cdot\underbrace{\frac{1}{g}\sum_{i=1}^{g} \hat{a}_i^2}_{\approx\,1\text{ for non-degenerate groups, }0\text{ otherwise}}
\;\approx\; \frac{q_{\mathrm{nd}}\,\sigma_s^2}{g\,T}.
\label{eq:snr_grpo_var}
\end{align}
Combining \eqref{eq:snr_grpo_signal} and \eqref{eq:snr_grpo_var},
\begin{align}
\mathrm{SNR}_{\mathrm{GRPO}}
\;=\; \frac{\|\mathbb{E}[\hat{\mathbf{g}}_{\mathrm{GRPO}}]\|^2}{\Var(\hat{\mathbf{g}}_{\mathrm{GRPO}})}
\;\approx\; g\,T\,\frac{\kappa_g(p)\,\|\boldsymbol{\Delta}\|^2}{\sigma_s^2},
\qquad
\kappa_g(p) \Def q_{\mathrm{nd}}\rho_g(p)^2.
\label{eq:snr_grpo}
\end{align}

\subsection{On-policy SNR comparison}

Dividing \eqref{eq:snr_sft} by \eqref{eq:snr_grpo},
\begin{align}
\frac{\mathrm{SNR}_{\mathrm{SFT}}}{\mathrm{SNR}_{\mathrm{GRPO}}}
\;\approx\; \frac{\|\bar{\mathbf{s}}\|^2}{\kappa_g(p)\,\|\boldsymbol{\Delta}\|^2}.
\label{eq:snr_ratio_onpolicy}
\end{align}
Two regimes drive this ratio large: (i) extreme difficulty $p\to 0$ or $p\to 1$, where the effective reward signal $\kappa_g(p)$ vanishes because many groups become degenerate and the within-group success/failure contrast disappears; and (ii) low distinctiveness $\|\boldsymbol{\Delta}\|\ll\|\bar{\mathbf{s}}\|$, where successful and failed rollouts produce nearly identical score directions. In the large-$g$ non-degenerate approximation, $\kappa_g(p)\approx p(1-p)$, recovering the simpler ratio $\|\bar{\mathbf{s}}\|^2/(p(1-p)\|\boldsymbol{\Delta}\|^2)$. These are exactly the failure modes that dynamic sampling \citep{yu2025dapo} and mean-only normalization \citep{liu2025understanding} are designed to mitigate.

\subsection{Additional SNR degradation in GRPO}

The on-policy bound \eqref{eq:snr_grpo_var} is optimistic: practical GRPO runs deviate from on-policy and lose signal through clipping and degenerate reward groups, neither of which has an SFT counterpart.

\paragraph{Importance-sampling amplification.} When $\boldsymbol{\Theta}\neq\boldsymbol{\Theta}_{\mathrm{old}}$, each token gradient in \eqref{eq:snr_grpo_full} is weighted by $r_{i,t}$. Assuming that magnitudes and directions of $r_{i,t}$ and $\nabla_{\boldsymbol{\Theta}} \ell_{i,t}$ factorize in second moment,
\begin{align}
\mathbb{E}\|r_{i,t}\,\nabla_{\boldsymbol{\Theta}} \ell_{i,t}\|^2
\;=\; \mathbb{E}_{\pi_{\mathrm{old}}}[r_{i,t}^2]\cdot\mathbb{E}\|\nabla_{\boldsymbol{\Theta}} \ell_{i,t}\|^2
\;=\; (1+\chi^2)\,\sigma_s^2,
\qquad
\chi^2 \;\Def\; \mathbb{E}_{\pi_{\mathrm{old}}}[r_{i,t}^2] - 1,
\label{eq:snr_is}
\end{align}
where $\chi^2$ is the per-token chi-squared divergence between $\pi_{\boldsymbol{\Theta}}$ and $\pi_{\mathrm{old}}$ (equivalently, $e^{D_2(\pi_{\boldsymbol{\Theta}}\|\pi_{\mathrm{old}})} - 1$, where $D_2$ denotes the Rényi-2 divergence); it equals zero on-policy and grows with every inner gradient step. Thus, off-policy updates multiply the variance term in \eqref{eq:snr_grpo_var} by $(1+\chi^2)$.

\paragraph{Clipping-induced signal loss.} Let $\alpha = \Pr(\mathbb{1}_{i,t}=0)$ be the clip fraction. Modeling $\mathbb{1}_{i,t}$ as a Bernoulli$(1-\alpha)$ mask independent of the per-token score (valid in the mean-field sense), under random masking the conditional expectation of the per-response score $\bar{\mathbf{S}}_i$ scales by $(1-\alpha)$ while its variance scales by $(1-\alpha)$ as well, so signal-squared contributes a factor $(1-\alpha)^2$ and variance contributes $(1-\alpha)$, giving a net $(1-\alpha)$ attenuation on the GRPO SNR (equivalently, substituting the effective length $T \to (1-\alpha)T$ into \eqref{eq:snr_grpo}). This attenuation appears in the SFT/GRPO SNR ratio as
\begin{align}
\frac{1}{1-\alpha}.
\label{eq:snr_clip}
\end{align}

\paragraph{Degenerate reward groups.} For binary rewards, group normalization provides a useful advantage only when a group contains both successes and failures. As reflected in $\kappa_g(p)$ above, the probability of such a non-degenerate group is
\begin{align}
q_{\mathrm{nd}} \;=\; 1 - p^{g} - (1-p)^{g}.
\label{eq:snr_non_degenerate}
\end{align}
When $p\to0$ or $p\to1$, $q_{\mathrm{nd}}$ becomes small: many groups have zero reward variance, hence zero normalized advantage and no learning signal. This reduces the effective batch size and weakens the GRPO signal, beyond the large-$g$ approximation $p(1-p)$.

\subsection{Combined bound}

Combining the on-policy variance \eqref{eq:snr_grpo_var} with the importance-sampling factor \eqref{eq:snr_is} and the clipping attenuation \eqref{eq:snr_clip}, the off-policy variance and SNR ratio satisfy
\begin{align}
\Var(\hat{\mathbf{g}}_{\mathrm{GRPO}}^{\mathrm{full}})
\;\gtrsim\;
\frac{q_{\mathrm{nd}}\sigma_s^2}{g\,T}
\cdot \underbrace{(1+\chi^2)}_{\text{IS}\eqref{eq:snr_is}},
\label{eq:snr_full_var}
\end{align}
where the clipping attenuation \eqref{eq:snr_clip} is not absorbed into this variance bound because it scales signal-squared and variance simultaneously; instead, its net SNR effect is folded directly into the SFT/GRPO ratio:
\begin{align}
\frac{\mathrm{SNR}_{\mathrm{SFT}}}{\mathrm{SNR}_{\mathrm{GRPO}}^{\mathrm{full}}}
\;\gtrsim\;
\underbrace{\frac{\|\bar{\mathbf{s}}\|^2}{\kappa_g(p)\,\|\boldsymbol{\Delta}\|^2}}_{\text{credit assignment}}
\cdot (1+\chi^2)
\cdot \frac{1}{1-\alpha}.
\label{eq:snr_full_ratio}
\end{align}

\clearpage
\newpage
\section{SVD Factorization of Newton--Schulz Polynomial Iteration}
\label{sec:poly_vs_power}

This appendix provides the detailed derivation behind the claim in Sec.\,\ref{sec: method} that designing Pion's spectral high-pass at the matrix level reduces, via the SVD, to designing a scalar polynomial $f$ on $[0, 1]$. We show that the odd matrix polynomial used by a single Newton--Schulz (NS) step factors through the SVD as a scalar polynomial acting entrywise on the singular values, so that designing the matrix filter is equivalent to designing three scalar coefficients $(a, b, c)$. The chaining of multiple NS steps further composes these scalar polynomials, while leaving the singular vectors $(\mathbf{U}, \mathbf{V})$ unchanged throughout.

\paragraph{Setup.}
Let $\mathbf{X} \in \mathbb{R}^{m \times n}$ with $r \Def \mathrm{rank}(\mathbf{X})$, and let its \textit{compact} singular value decomposition (consistent with \eqref{eq:msign}) be
\begin{equation}
    \mathbf{X} \;=\; \mathbf{U}\,\boldsymbol{\Sigma}\,\mathbf{V}^\top,
    \qquad
    \mathbf{U} \in \mathbb{R}^{m \times r}, \;\; \mathbf{V} \in \mathbb{R}^{n \times r}, \;\; \mathbf{U}^\top \mathbf{U} = \mathbf{V}^\top \mathbf{V} = \mathbf{I}_r,
    \label{eq: svd}
\end{equation}
where $\boldsymbol{\Sigma} = \diag(\sigma_1, \ldots, \sigma_r) \succ 0$ collects the strictly positive singular values.

\paragraph{Polynomial iteration factors through the SVD.}
Consider the odd matrix polynomial used by a single quintic NS step:
\begin{equation}
    \mathcal{P}(\mathbf{X};\, a, b, c) \;\Def\; a\mathbf{X} + b\,\mathbf{X}\mathbf{X}^\top\mathbf{X} + c\,\mathbf{X}(\mathbf{X}^\top\mathbf{X})^2.
    \label{eq: poly_matrix}
\end{equation}
Using the SVD \eqref{eq: svd}, we have $\mathbf{X}^\top \mathbf{X} = \mathbf{V}\boldsymbol{\Sigma}^2\mathbf{V}^\top$. Since the thin right singular vector matrix satisfies $\mathbf{V}\mathbf{V}^\top \neq \mathbf{I}_n$ in general, the Gram-power identity should be read for positive powers:
\begin{equation}
    (\mathbf{X}^\top\mathbf{X})^k \;=\; \mathbf{V}\,\boldsymbol{\Sigma}^{2k}\,\mathbf{V}^\top
    \quad\text{for all } k\in\mathbb{N}_{\geq 1}.
    \label{eq: gram_power}
\end{equation}
For $k \geq 1$, left-multiplying the Gram power by $\mathbf{X} = \mathbf{U}\boldsymbol{\Sigma}\mathbf{V}^\top$ and using $\mathbf{V}^\top \mathbf{V} = \mathbf{I}_r$ yields
$\mathbf{X}(\mathbf{X}^\top\mathbf{X})^k = \mathbf{U}\boldsymbol{\Sigma}^{2k+1}\mathbf{V}^\top$; the same identity is immediate for $k=0$. Hence the key identity is
\begin{equation}
    \mathbf{X}(\mathbf{X}^\top\mathbf{X})^k \;=\; \mathbf{U}\,\boldsymbol{\Sigma}^{2k+1}\,\mathbf{V}^\top.
    \label{eq: odd_power}
\end{equation}
Substituting \eqref{eq: odd_power} into \eqref{eq: poly_matrix}, the matrix iteration collapses to
\begin{equation}
    \mathcal{P}(\mathbf{X};\, a, b, c) \;=\; \mathbf{U}\,\underbrace{\bigl(a\,\boldsymbol{\Sigma} + b\,\boldsymbol{\Sigma}^3 + c\,\boldsymbol{\Sigma}^5\bigr)}_{f(\boldsymbol{\Sigma};\, a, b, c)}\,\mathbf{V}^\top
    \;=\; \mathbf{U}\,f(\boldsymbol{\Sigma};\, a, b, c)\,\mathbf{V}^\top,
    \label{eq: svd_factorization}
\end{equation}
where $f(\sigma; a, b, c) = a\sigma + b\sigma^3 + c\sigma^5$ is the scalar polynomial from \eqref{eq: pion_poly} and $f(\boldsymbol{\Sigma})$ is understood as applying $f$ entrywise to the diagonal of $\boldsymbol{\Sigma}$. Equation~\eqref{eq: svd_factorization} has three important consequences:
\begin{itemize}[leftmargin=*, topsep=0pt, itemsep=0pt]
    \item \textbf{Per-singular-value control.} The matrix map $\mathbf{X} \mapsto \mathcal{P}(\mathbf{X})$ is exactly equivalent to the scalar map $\sigma_i \mapsto f(\sigma_i)$ applied independently to each singular value.
    \item \textbf{Invariance of singular vectors.} The left and right singular vectors $\mathbf{U}$ and $\mathbf{V}$ are preserved unchanged; only the singular values are reshaped.
    \item \textbf{Reduction to a 3-dim. coefficient design.} Specifying the full matrix-level filter reduces to specifying the three scalar coefficients $(a, b, c)$ that encode the desired shape of $f$ on $[0, 1]$.
\end{itemize}

\paragraph{Composition of NS steps.}
Composing $t$ NS steps $\mathcal{P}_t \circ \cdots \circ \mathcal{P}_1$ simply composes the scalar polynomials. If step $\mathcal{P}_i$ uses coefficients $(a_i, b_i, c_i)$ and induces the scalar map $f_i$, then by repeatedly applying \eqref{eq: svd_factorization},
\begin{equation}
    \bigl(\mathcal{P}_t \circ \cdots \circ \mathcal{P}_1\bigr)(\mathbf{X}) \;=\; \mathbf{U}\,\bigl(f_t \circ \cdots \circ f_1\bigr)(\boldsymbol{\Sigma})\,\mathbf{V}^\top.
    \label{eq: chained_factorization}
\end{equation}
This is exactly the chaining mechanism exploited by Pion to compose Promotion \eqref{eq: pion_promotion} for $k_{\mathrm{p}}$ steps and Suppression \eqref{eq: pion_suppression} for $k_{\mathrm{s}}$ steps into a single composite high-pass $f_{\mathrm{s}}^{\circ k_{\mathrm{s}}} \circ f_{\mathrm{p}}^{\circ k_{\mathrm{p}}}$ acting entrywise on $\boldsymbol{\Sigma}$, while leaving $(\mathbf{U}, \mathbf{V})$ untouched throughout.

\paragraph{Conclusion.}
The SVD factorization \eqref{eq: svd_factorization} reduces the problem of designing a matrix-level spectral filter to the problem of designing a scalar polynomial $f$ on $[0, 1]$. This justifies the treatment in Sec.\,\ref{sec: method}, where the entire Pion design (Promotion plus Suppression) is specified through scalar coefficients $(a_{\mathrm{p}}, b_{\mathrm{p}}, c_{\mathrm{p}})$ and $(a_{\mathrm{s}}, b_{\mathrm{s}}, c_{\mathrm{s}})$ acting on the normalized singular spectrum, with the singular vectors $(\mathbf{U}, \mathbf{V})$ of the gradient preserved exactly throughout the iteration.

\clearpage
\newpage
\section{Derivation of the Promotion and Suppression Polynomials}
\label{sec:pion_derivation}

\paragraph{Setup.}
Recall from \eqref{eq: pion_poly} the odd quintic scalar map that any single NS step induces on each normalized singular value $\sigma \in [0, 1]$:
\begin{equation}
    f(\sigma;\, a, b, c) \;=\; a\,\sigma + b\,\sigma^{3} + c\,\sigma^{5}, 
    \qquad
    f'(\sigma) \;=\; a + 3b\,\sigma^{2} + 5c\,\sigma^{4}, 
    \qquad
    f''(\sigma) \;=\; 6b\,\sigma + 20c\,\sigma^{3}.
    \label{eq: pion_poly_derivs}
\end{equation}
The Pion design problem is to choose two sets of coefficients $(a_{\mathrm{p}}, b_{\mathrm{p}}, c_{\mathrm{p}})$ and $(a_{\mathrm{s}}, b_{\mathrm{s}}, c_{\mathrm{s}})$ such that the chained iteration $f_{\mathrm{s}}^{\circ k_{\mathrm{s}}} \circ f_{\mathrm{p}}^{\circ k_{\mathrm{p}}}$ realizes a high-pass on $[0,1]$.

\subsection{Promotion polynomial $f_{\mathrm{p}}$}
\label{sec:promotion_derivation}

\paragraph{Design constraints.}
The Promotion stage must satisfy three constraints:
\begin{itemize}
    \setlength{\itemsep}{2pt}
    \item \textbf{(P1)} \textit{Fixed point}: $f_{\mathrm{p}}(1) = 1$, i.e., any singular value already at $1$ is left unchanged.
    \item \textbf{(P2)} \textit{First-order stationarity}: $f_{\mathrm{p}}'(1) = 0$, so that small perturbations around the fixed point $\sigma = 1$ are not amplified.
    \item \textbf{(P3)} \textit{Boundary concavity}: $f_{\mathrm{p}}''(1) \leq 0$, which prevents the Promotion map from curving upward near the anchored fixed point and pushing nearby singular values outside the normalized spectral range.
\end{itemize}

We motivate (P3) as follows. Since (P2) makes $\sigma = 1$ a stationary point of $f_{\mathrm{p}}$, the sign of $f_{\mathrm{p}}''(1)$ controls the local shape of $f_{\mathrm{p}}$ near $\sigma = 1$. If $f_{\mathrm{p}}''(1) > 0$, then $f_{\mathrm{p}}'$ is strictly increasing through $0$ at $\sigma = 1$ and hence strictly negative just to the left of $\sigma = 1$; consequently $f_{\mathrm{p}}$ is locally decreasing as $\sigma \uparrow 1$, so values $\sigma$ slightly below $1$ are mapped to $f_{\mathrm{p}}(\sigma) > f_{\mathrm{p}}(1) = 1$, leaving the spectral budget $[0, 1]$. Imposing $f_{\mathrm{p}}''(1) \leq 0$ rules out this upward curving; the strict case $f_{\mathrm{p}}''(1) < 0$ already gives a local maximum at $\sigma = 1$ via the standard second-derivative test, while the boundary case $f_{\mathrm{p}}''(1) = 0$ is degenerate at the second order and its consequences are pinned down by the global monotonicity analysis below. As we verify below, restricted to the one-parameter family fixed by (P1)--(P2), the boundary concavity (P3) together with a matching lower bound is in fact equivalent to global monotonicity of $f_{\mathrm{p}}$ on $[0, 1]$, so it preserves the relative ordering of singular values throughout. As an immediate corollary, the Promotion stage stays inside the spectral budget: $f_{\mathrm{p}}(\sigma) \leq f_{\mathrm{p}}(1) = 1$ for all $\sigma \in [0, 1]$.

\paragraph{Step 1: reduction to a one-parameter family via (P1)--(P2).}
By \eqref{eq: pion_poly_derivs}, conditions (P1) and (P2) yield
\begin{equation}
    a_{\mathrm{p}} + b_{\mathrm{p}} + c_{\mathrm{p}} \;=\; 1, 
    \qquad
    a_{\mathrm{p}} + 3 b_{\mathrm{p}} + 5 c_{\mathrm{p}} \;=\; 0.
    \label{eq: promotion_linear_system}
\end{equation}
Solving for $a_{\mathrm{p}}$ and $b_{\mathrm{p}}$ in terms of $c_{\mathrm{p}}$,
\begin{equation}
    b_{\mathrm{p}} \;=\; -\,\frac{1 + 4 c_{\mathrm{p}}}{2}, 
    \qquad 
    a_{\mathrm{p}} \;=\; \frac{3 + 2 c_{\mathrm{p}}}{2}.
    \label{eq: promotion_param}
\end{equation}

\paragraph{Step 2: applying (P3) to obtain feasible ranges of $(a_{\mathrm{p}}, b_{\mathrm{p}}, c_{\mathrm{p}})$.}
Substituting \eqref{eq: promotion_param} into the second-order derivative gives
\begin{equation}
    f_{\mathrm{p}}''(1) \;=\; 6 b_{\mathrm{p}} + 20 c_{\mathrm{p}} \;=\; -\,3 + 8 c_{\mathrm{p}},
    \label{eq: promotion_curvature_explicit}
\end{equation}
so (P3) is equivalent to
\begin{equation}
    c_{\mathrm{p}} \;\leq\; 0.375.
    \label{eq: cp_bound}
\end{equation}

Next, we derive the conditions that ensure $f_{\mathrm{p}}$ is monotonically non-decreasing on $[0, 1]$. Setting $u \Def \sigma^{2} \in [0, 1]$, define
\begin{equation}
    g(u) \;\Def\; f_{\mathrm{p}}'(\sigma) \;=\; a_{\mathrm{p}} + 3 b_{\mathrm{p}}\, u + 5 c_{\mathrm{p}}\, u^{2}.
\end{equation}
Then $g$ is a quadratic in $u$ with $g(1) = a_{\mathrm{p}} + 3 b_{\mathrm{p}} + 5 c_{\mathrm{p}} = 0$ by (P2). For $c_{\mathrm{p}} \neq 0$, this lets us factor $g$ as
\begin{equation}
    g(u) \;=\; 5\, c_{\mathrm{p}}\, (u - 1)(u - r), 
    \qquad 
    r \;=\; \frac{a_{\mathrm{p}}}{5\, c_{\mathrm{p}}} \;=\; \frac{3 + 2 c_{\mathrm{p}}}{10\, c_{\mathrm{p}}}.
    \label{eq: gp_factorization}
\end{equation}

Since $u - 1 \leq 0$ for all $u \in [0, 1]$, the inequality $g(u) \geq 0$ is equivalent to $5 c_{\mathrm{p}}\,(u - r) \leq 0$ on $[0, 1]$. We split on the sign of $c_{\mathrm{p}}$:
\begin{itemize}
    \setlength{\itemsep}{2pt}
    \item If $c_{\mathrm{p}} > 0$, we need $u \leq r$ for all $u \in [0, 1]$, i.e.\ $r \geq 1$. From \eqref{eq: gp_factorization}, $r \geq 1 \iff 3 + 2 c_{\mathrm{p}} \geq 10 c_{\mathrm{p}} \iff c_{\mathrm{p}} \leq 0.375$.
    \item If $c_{\mathrm{p}} < 0$, we need $u \geq r$ for all $u \in [0, 1]$, i.e.\ $r \leq 0$. From \eqref{eq: gp_factorization}, $r \leq 0 \iff 3 + 2 c_{\mathrm{p}} \geq 0 \iff c_{\mathrm{p}} \geq -1.5$.
    \item If $c_{\mathrm{p}} = 0$, then $g(u) = \tfrac{3}{2} - \tfrac{3}{2} u \geq 0$ on $[0, 1]$.
\end{itemize}
Combining the three cases yields the feasible range
\begin{equation}
    -1.5 \;\leq\; c_{\mathrm{p}} \;\leq\; 0.375 
    \quad \Longrightarrow \quad
    g(u) \geq 0 \text{ for all } u \in [0, 1].
    \label{eq: cp_global_bound}
\end{equation}
The upper bound in \eqref{eq: cp_global_bound} coincides with the local condition \eqref{eq: cp_bound} from (P3), and the lower bound corresponds (via \eqref{eq: promotion_param}) exactly to $a_{\mathrm{p}} \geq 0$. Hence, within the family pinned by (P1)--(P2), the boundary concavity (P3) together with $a_{\mathrm{p}} \geq 0$ is necessary and sufficient for global monotonicity of $f_{\mathrm{p}}$ on $[0, 1]$.

Combining \eqref{eq: cp_global_bound} with \eqref{eq: promotion_param}, we obtain the feasible coefficient ranges
\begin{equation}
    0 \;\leq\; a_{\mathrm{p}} \;\leq\; 1.875, \qquad
    -1.25 \;\leq\; b_{\mathrm{p}} \;\leq\; 2.5, \qquad
    -1.5 \;\leq\; c_{\mathrm{p}} \;\leq\; 0.375.
    \label{eq: pcoeff_ranges}
\end{equation}

\paragraph{Step 3: choosing the largest feasible slope at the origin.}
The slope $a_{\mathrm{p}} = f_{\mathrm{p}}'(0)$ controls how aggressively a single Promotion step lifts small singular values $\sigma \approx 0$ into the regime where Suppression eventually anchors them at $1$: since $f_{\mathrm{p}}(\sigma) \approx a_{\mathrm{p}}\,\sigma$ near the origin, small singular values are amplified by a factor of approximately $a_{\mathrm{p}}$ per step. We therefore choose $a_{\mathrm{p}}$ at its maximal feasible value, $a_{\mathrm{p}} = 1.875$, to promote rapid growth under a fixed budget of $k = 5$ NS iterations. This achieves equality in \eqref{eq: cp_bound}, yielding $c_{\mathrm{p}} = 0.375$ and, by \eqref{eq: promotion_curvature_explicit}, $f_{\mathrm{p}}''(1) = 0$. Substituting back into \eqref{eq: promotion_param} fixes
\begin{equation}
    (a_{\mathrm{p}}, b_{\mathrm{p}}, c_{\mathrm{p}}) \;=\; (1.875,\, -1.25,\, 0.375),
    \label{eq: promotion_coeffs}
\end{equation}
which recovers exactly \eqref{eq: pion_promotion}. At these coefficients, the derivative simplifies to a perfect square,
\begin{equation}
    f_{\mathrm{p}}'(\sigma) \;=\; 1.875 - 3.75\,\sigma^{2} + 1.875\,\sigma^{4} \;=\; 1.875\,\bigl(1 - \sigma^{2}\bigr)^{2} \;\geq\; 0
    \qquad \forall\, \sigma \in [0, 1],
    \label{eq: promotion_monotone}
\end{equation}
making $f_{\mathrm{p}}$ monotone non-decreasing on $[0, 1]$ with $f_{\mathrm{p}}'$ vanishing only at the boundary $\sigma = 1$.

\subsection{Suppression polynomial $f_{\mathrm{s}}$}
\label{sec:suppression_derivation}

\paragraph{Design constraints.}
The Suppression stage inherits the fixed-point and first-order stationarity conditions at $\sigma = 1$ from Promotion in order to anchor the leading singular values at $1$. In addition, it imposes a \emph{spectral filtering} condition at the origin that strips the linear term, so that small singular values are driven toward $0$ by the higher-order ($\sigma^{3}, \sigma^{5}$) terms. Concretely:
\begin{itemize}
    \setlength{\itemsep}{2pt}
    \item \textbf{(S1)} \textit{Fixed point}: $f_{\mathrm{s}}(1) = 1$.
    \item \textbf{(S2)} \textit{First-order stationarity}: $f_{\mathrm{s}}'(1) = 0$.
    \item \textbf{(S3)} \textit{Spectral filtering at the origin}: $f_{\mathrm{s}}'(0) = 0$, eliminating the linear term so that small singular values $\sigma \approx 0$ are pushed toward $0$ by the higher-order terms.
\end{itemize}
By \eqref{eq: pion_poly_derivs}, (S3) is equivalent to $a_{\mathrm{s}} = 0$. Substituting into (S1) and (S2) gives a $2 \times 2$ linear system in $(b_{\mathrm{s}}, c_{\mathrm{s}})$:
\begin{equation}
    b_{\mathrm{s}} + c_{\mathrm{s}} \;=\; 1, 
    \qquad 
    3 b_{\mathrm{s}} + 5 c_{\mathrm{s}} \;=\; 0,
\end{equation}
whose unique solution is $b_{\mathrm{s}} = 2.5$ and $c_{\mathrm{s}} = -1.5$. Combined with $a_{\mathrm{s}} = 0$, this yields
\begin{equation}
    (a_{\mathrm{s}}, b_{\mathrm{s}}, c_{\mathrm{s}}) \;=\; (0,\, 2.5,\, -1.5),
    \label{eq: suppression_coeffs}
\end{equation}
which recovers exactly \eqref{eq: pion_suppression}. Unlike the Promotion stage, the Suppression coefficients are determined uniquely by (S1)--(S3) and admit no remaining degree of freedom. At these coefficients, the derivative factors as
\begin{equation}
    f_{\mathrm{s}}'(\sigma) \;=\; 7.5\,\sigma^{2} - 7.5\,\sigma^{4} \;=\; 7.5\,\sigma^{2}\bigl(1 - \sigma^{2}\bigr) \;\geq\; 0
    \qquad \forall\, \sigma \in [0, 1],
    \label{eq: suppression_monotone}
\end{equation}
so $f_{\mathrm{s}}$ is monotone non-decreasing on $[0, 1]$ with $f_{\mathrm{s}}'$ vanishing only at the endpoints $\sigma \in \{0, 1\}$. Hence Suppression also preserves the relative ordering of singular values, and the chained iteration $f_{\mathrm{s}}^{\circ k_{\mathrm{s}}} \circ f_{\mathrm{p}}^{\circ k_{\mathrm{p}}}$ is monotone on $[0, 1]$.
\clearpage
\newpage
\section{The Pion Optimizer: Full Algorithmic Description}
\label{sec:pion_algorithms}

We provide the full pseudocode for Pion deferred from Sec.\,\ref{sec: method}. Pion is a drop-in replacement for Muon: the only change is that the per-step Newton--Schulz orthogonalization \eqref{eq: ns_matrix} is replaced by our \textbf{high-pass NS}, which chains the Promotion polynomial $f_{\mathrm{p}}$ \eqref{eq: pion_promotion} and the Suppression polynomial $f_{\mathrm{s}}$ \eqref{eq: pion_suppression}. The total iteration count is fixed to $k = 5$, split by $k_{\mathrm{p}} \in \{0, 1, \ldots, 5\}$ with $k_{\mathrm{s}} = k - k_{\mathrm{p}}$. The high-pass NS has two modes: a \textbf{default} mode applied to each weight matrix $\mathbf{M}_t \in \mathbb{R}^{m \times n}$ as a whole (Alg.\,\ref{alg:pion_optimizer}), used for VLA training, and a \textbf{per-head} mode that splits each attention projection along the head dimension into sub-blocks $\{\mathbf{M}_t^h\}_{h=1}^{H}$ and runs the iteration independently per head (Alg.\,\ref{alg:multihead_pion_optimizer}), used for RLVR post-training; the per-head mode adds only a single reshape on top of the default mode.

\vspace{-1mm}
\begin{algorithm}[H]
    \caption{Pion Optimizer (\textbf{default mode}: high-pass NS on the whole matrix)}
    \label{alg:pion_optimizer}
    \begin{algorithmic}[1]
    \renewcommand{\algorithmiccomment}[1]{\hfill #1}
    \REQUIRE Learning rate $\eta$, momentum coefficient $\mu$, promotion steps $k_{\mathrm{p}}$
    \STATE $k_{\mathrm{s}} \leftarrow 5 - k_{\mathrm{p}}$;\quad $\mathbf{M}_0 \leftarrow \mathbf{0}$ \COMMENT{Total iterations strictly fixed to $k = 5$}
    \FOR{$t = 1, 2, \dots$}
        \STATE $\mathbf{G}_t \leftarrow \nabla_{\boldsymbol{\Theta}} \mathcal{L}_t(\boldsymbol{\Theta}_{t-1})$;\quad $\mathbf{M}_t \leftarrow \mu\, \mathbf{M}_{t-1} + \mathbf{G}_t$
        \STATE $\mathbf{X} \leftarrow \mathbf{M}_t / (\lVert \mathbf{M}_t \rVert_{\mathrm{F}} + \epsilon)$ \COMMENT{Spectral pre-normalization, cf.\ \eqref{eq: ns_matrix}}
        \FOR[\textbf{Stage 1: Promotion} \eqref{eq: pion_promotion}, $(a_{\mathrm{p}}, b_{\mathrm{p}}, c_{\mathrm{p}}) = (1.875, -1.25, 0.375)$]{$i = 1, \dots, k_{\mathrm{p}}$}
            \STATE $\mathbf{X} \leftarrow a_{\mathrm{p}}\,\mathbf{X} + b_{\mathrm{p}}\,\mathbf{X}\mathbf{X}^\top\mathbf{X} + c_{\mathrm{p}}\,\mathbf{X}(\mathbf{X}^\top\mathbf{X})^2$
        \ENDFOR
        \FOR[\textbf{Stage 2: Suppression} \eqref{eq: pion_suppression}, $(a_{\mathrm{s}}, b_{\mathrm{s}}, c_{\mathrm{s}}) = (0, 2.5, -1.5)$]{$j = 1, \dots, k_{\mathrm{s}}$}
            \STATE $\mathbf{X} \leftarrow a_{\mathrm{s}}\,\mathbf{X} + b_{\mathrm{s}}\,\mathbf{X}\mathbf{X}^\top\mathbf{X} + c_{\mathrm{s}}\,\mathbf{X}(\mathbf{X}^\top\mathbf{X})^2$
        \ENDFOR
        \STATE $\boldsymbol{\Theta}_t \leftarrow \boldsymbol{\Theta}_{t-1} - \eta\, \mathbf{X}$
    \ENDFOR
    \RETURN $\boldsymbol{\Theta}_t$
    \end{algorithmic}
\end{algorithm}

\vspace{-2mm}
\begin{algorithm}[H]
    \caption{Pion Optimizer (\textbf{per-head mode}: per-head high-pass NS on attention projections)}
    \label{alg:multihead_pion_optimizer}
    \begin{algorithmic}[1]
    \renewcommand{\algorithmiccomment}[1]{\hfill #1}
    \REQUIRE Learning rate $\eta$, momentum coefficient $\mu$, promotion steps $k_{\mathrm{p}}$, number of heads $H$
    \STATE $k_{\mathrm{s}} \leftarrow 5 - k_{\mathrm{p}}$;\quad $\mathbf{M}_0 \leftarrow \mathbf{0}$ \COMMENT{Total iterations strictly fixed to $k = 5$}
    \FOR{$t = 1, 2, \dots$}
        \STATE $\mathbf{G}_t \leftarrow \nabla_{\boldsymbol{\Theta}} \mathcal{L}_t(\boldsymbol{\Theta}_{t-1})$;\quad $\mathbf{M}_t \leftarrow \mu\, \mathbf{M}_{t-1} + \mathbf{G}_t$
        \STATE $\{\mathbf{M}_t^h\}_{h=1}^{H} \leftarrow \mathrm{Reshape}(\mathbf{M}_t)$ \COMMENT{Split the attention projection along the head dim}
        \STATE $\mathbf{X}^h \leftarrow \mathbf{M}_t^h / (\lVert \mathbf{M}_t^h \rVert_{\mathrm{F}} + \epsilon),\;\; \forall\, h \in \{1, \dots, H\}$ \COMMENT{Per-head pre-normalization}
        \FOR[\textbf{Stage 1: Promotion} \eqref{eq: pion_promotion}, batched over $H$]{$i = 1, \dots, k_{\mathrm{p}}$}
            \STATE $\mathbf{X}^h \leftarrow a_{\mathrm{p}}\,\mathbf{X}^h + b_{\mathrm{p}}\,\mathbf{X}^h(\mathbf{X}^h)^\top\mathbf{X}^h + c_{\mathrm{p}}\,\mathbf{X}^h\bigl((\mathbf{X}^h)^\top\mathbf{X}^h\bigr)^2,\;\; \forall\, h \in \{1, \dots, H\}$
        \ENDFOR
        \FOR[\textbf{Stage 2: Suppression} \eqref{eq: pion_suppression}, batched over $H$]{$j = 1, \dots, k_{\mathrm{s}}$}
            \STATE $\mathbf{X}^h \leftarrow a_{\mathrm{s}}\,\mathbf{X}^h + b_{\mathrm{s}}\,\mathbf{X}^h(\mathbf{X}^h)^\top\mathbf{X}^h + c_{\mathrm{s}}\,\mathbf{X}^h\bigl((\mathbf{X}^h)^\top\mathbf{X}^h\bigr)^2,\;\; \forall\, h \in \{1, \dots, H\}$
        \ENDFOR
        \STATE $\mathbf{X} \leftarrow \mathrm{Reshape}^{-1}\bigl(\{\mathbf{X}^h\}_{h=1}^{H}\bigr) \in \mathbb{R}^{m \times n}$
        \STATE $\boldsymbol{\Theta}_t \leftarrow \boldsymbol{\Theta}_{t-1} - \eta\, \mathbf{X}$
    \ENDFOR
    \RETURN $\boldsymbol{\Theta}_t$
    \end{algorithmic}
\end{algorithm}

\clearpage
\newpage
\section{Per-Head Norm Heterogeneity Affects Forward and Backward Computation}
\label{sec:head_norm_analysis}

We analyze how per-head norm heterogeneity, an empirical property of trained transformers (Fig.\,\ref{fig:mh_analysis}-(b)), affects both forward computation and gradient flow. This motivates per-head spectral filtering in place of whole-matrix filtering.

\paragraph{Notation.}
For clarity, we write the analysis for a standard multi-head attention layer. For grouped-query or multi-query attention, the same argument applies to each Q, K, and V projection along its own head dimension.
Let $\mathbf{X}\!\in\!\mathbb{R}^{n\times d}$ denote the input sequence and let $d_k$ be the head dimension. For head $h$, define $\mathbf{W}_Q^h,\mathbf{W}_K^h,\mathbf{W}_V^h\in\mathbb{R}^{d\times d_k}$ and $\mathbf{W}_O^h\in\mathbb{R}^{d_k\times d}$. The head computes $\mathbf{S}^h=\mathbf{X}\mathbf{W}_Q^h(\mathbf{X}\mathbf{W}_K^h)^\top/\sqrt{d_k}$, $\mathbf{A}^h=\mathrm{softmax}(\mathbf{S}^h)$ row-wise, $\mathbf{O}^h=\mathbf{A}^h\mathbf{X}\mathbf{W}_V^h$, and the layer output is $\mathbf{Z}=\sum_h\mathbf{O}^h\mathbf{W}_O^h$.

\begin{proposition}[Per-head norms modulate attention and gradients]
\label{prop:head_norm_coupling}
For each head $h$, the following forward and backward norm couplings hold.
\begin{enumerate}[label=(\alph*),leftmargin=*]
    \item \textbf{Forward.}
    The Q/K norms control attention \emph{sharpness}: the logits admit the factorization
    \begin{equation}
    \label{eq:logit_temperature}
        \mathbf{S}^h \;=\; \underbrace{\frac{\|\mathbf{W}_Q^h\|_F\,\|\mathbf{W}_K^h\|_F}{\sqrt{d_k}}}_{\text{effective inverse temperature}~\beta_h}\;\cdot\; \mathbf{X}\,\widetilde{\mathbf{W}}^h\,\mathbf{X}^\top, \qquad \widetilde{\mathbf{W}}^h \Def \frac{\mathbf{W}_Q^h(\mathbf{W}_K^h)^\top}{\|\mathbf{W}_Q^h\|_F\,\|\mathbf{W}_K^h\|_F},
    \end{equation}
    so at fixed normalized shape $\widetilde{\mathbf{W}}^h$, larger $\|\mathbf{W}_Q^h\|_F\|\mathbf{W}_K^h\|_F$ gives a larger softmax inverse temperature and a sharper attention pattern. The V/O norms control the head's \emph{output magnitude}:
    \begin{equation}
    \label{eq:output_scale}
        \|\mathbf{O}^h \mathbf{W}_O^h\|_F \;\leq\; \|\mathbf{A}^h\|_2\,\|\mathbf{X}\|_F\,\|\mathbf{W}_V^h\|_2\,\|\mathbf{W}_O^h\|_2,
    \end{equation}
    so heads with larger $\|\mathbf{W}_V^h\|_2\|\mathbf{W}_O^h\|_2$ tend to contribute more to the layer output.

    \item \textbf{Backward.}
    Let $\mathbf{G}=\partial\mathcal{L}/\partial\mathbf{Z}$. Then
    \begin{align}
        \left\|\frac{\partial \mathcal{L}}{\partial \mathbf{W}_O^h}\right\|_F
        &\leq \|\mathbf{A}^h\|_2\,\|\mathbf{X}\|_2\,\|\mathbf{W}_V^h\|_2\,\|\mathbf{G}\|_F, \label{eq:bound_O}\\
        \left\|\frac{\partial \mathcal{L}}{\partial \mathbf{W}_V^h}\right\|_F
        &\leq \|\mathbf{A}^h\|_2\,\|\mathbf{X}\|_2\,\|\mathbf{W}_O^h\|_2\,\|\mathbf{G}\|_F, \label{eq:bound_V}\\
        \left\|\frac{\partial \mathcal{L}}{\partial \mathbf{W}_Q^h}\right\|_F
        &\leq C_X\,\|\mathbf{W}_K^h\|_2\,\|\mathbf{W}_V^h\|_2\,\|\mathbf{W}_O^h\|_2\,\|\mathbf{G}\|_F, \label{eq:bound_Q}\\
        \left\|\frac{\partial \mathcal{L}}{\partial \mathbf{W}_K^h}\right\|_F
        &\leq C_X\,\|\mathbf{W}_Q^h\|_2\,\|\mathbf{W}_V^h\|_2\,\|\mathbf{W}_O^h\|_2\,\|\mathbf{G}\|_F, \label{eq:bound_K}
    \end{align}
    where $C_X \Def 2\|\mathbf{X}\|_2^{\,3}/\sqrt{d_k}$.
\end{enumerate}
\end{proposition}

\begin{proof}
The logit factorization follows by substituting the definition of $\widetilde{\mathbf{W}}^h$. The sharpness claim is the standard temperature-scaling property of softmax: for non-constant $\boldsymbol{\ell}$, the entropy of $\mathrm{softmax}(\beta\boldsymbol{\ell})$ decreases with $\beta>0$. The output bound \eqref{eq:output_scale} follows from $\mathbf{O}^h\mathbf{W}_O^h=\mathbf{A}^h\mathbf{X}\mathbf{W}_V^h\mathbf{W}_O^h$ and submultiplicativity ($\|MN\|_F\leq\|M\|_2\|N\|_F$ applied left-to-right, then $\|MN\|_F\leq\|M\|_F\|N\|_2$ on $\mathbf{X}\mathbf{W}_V^h$).

For the backward bounds, the chain rule gives
$\partial\mathcal{L}/\partial\mathbf{W}_O^h=(\mathbf{X}\mathbf{W}_V^h)^\top(\mathbf{A}^h)^\top\mathbf{G}$ and
$\partial\mathcal{L}/\partial\mathbf{W}_V^h=\mathbf{X}^\top(\mathbf{A}^h)^\top\mathbf{G}(\mathbf{W}_O^h)^\top$, which imply \eqref{eq:bound_O} and \eqref{eq:bound_V}.
For Q and K, first note that $\partial\mathcal{L}/\partial\mathbf{A}^h=\mathbf{G}(\mathbf{W}_O^h)^\top(\mathbf{W}_V^h)^\top\mathbf{X}^\top$, so $\|\partial\mathcal{L}/\partial\mathbf{A}^h\|_F\leq\|\mathbf{G}\|_F\|\mathbf{W}_O^h\|_2\|\mathbf{W}_V^h\|_2\|\mathbf{X}\|_2$. The row-wise softmax Jacobian has spectral norm at most $2$, giving $\|\partial\mathcal{L}/\partial\mathbf{S}^h\|_F\leq2\|\partial\mathcal{L}/\partial\mathbf{A}^h\|_F$. Combining this with $\partial\mathcal{L}/\partial\mathbf{W}_Q^h=\mathbf{X}^\top(\partial\mathcal{L}/\partial\mathbf{S}^h)\mathbf{X}\mathbf{W}_K^h/\sqrt{d_k}$ yields \eqref{eq:bound_Q}; \eqref{eq:bound_K} follows symmetrically.
\end{proof}

\begin{remark}[Implications for optimizer design]
\label{rmk:optimizer}
Proposition~\ref{prop:head_norm_coupling} shows that the per-head norms inherited from prior training modulate both attention behavior and gradient scale. Since these norms vary substantially across heads in trained models (Fig.\,\ref{fig:mh_analysis}-(b)), different heads naturally receive updates of different magnitudes. A whole-matrix spectral optimizer applies one Newton-Schulz orthogonalization to a concatenated projection matrix, which tends to equalize update scale across heads and mix head-specific directions. Per-head spectral filtering avoids this by filtering each head independently.
\end{remark}
\section{Detailed Training Setups for VLA and RLVR Experiments}
\label{sec:training_details}

In this section, we report the hyperparameter configurations for the VLA and RLVR experiments in Sec.\,\ref{sec: exp}. Within each setting, the three optimizer configurations (AdamW, Muon, and Pion) share identical training setups, hardware, and evaluation protocols; the only altered variable is the optimizer assignment. For Pion, we use Suppression-dominant high-pass NS schedules with $k_{\mathrm{s}}\geq3$ (equivalently, $k_{\mathrm{p}}\leq2$ under the fixed total $k=5$). \textbf{Table\,\ref{tab:vla_training_config}} lists the VLA training hyperparameters for VLA-Adapter \citep{wang2026vla} and VLANeXt \citep{wu2026vlanext} on LIBERO \citep{liu2023libero}, with VLANeXt additionally evaluated on the perturbed LIBERO-Plus split \citep{fei2025libero}; the \textit{Object} suite converges faster and is allocated fewer training steps. \textbf{Table\,\ref{tab:rlvr_training_config}} summarizes the RLVR hyperparameters, reused across both RL algorithms (GRPO/GMPO) and both model scales (Qwen3-1.7B/4B); only the prompt/response length, train batch, rollout group size, and total steps differ between MATH and GSM8K. \textbf{Table\,\ref{tab:real_robot_training_config}} summarizes the real-robot setup, where $\pi_{0.5}$ \citep{intelligence2025pi_} is finetuned under the DROID hardware platform \citep{khazatsky2025droid,wang2026remac} and evaluated on three grasp-and-place tasks.

\begin{table}[htbp]
\centering
\caption{Training hyperparameters for the VLA experiments on the LIBERO benchmark. The three optimizer configurations (i)--(iii) are applied identically to both models, and share all other hyperparameters listed in this table.}
\label{tab:vla_training_config}
\begin{tabular}{lcc}
\toprule
\textbf{Item} & \textbf{VLA-Adapter} & \textbf{VLANeXt} \\
\midrule
Backbone & Prismatic-Qwen2.5-0.5B & Qwen3-VL-2B-Instruct \\
Train dataset & LIBERO & LIBERO \\
Test dataset & LIBERO & LIBERO and LIBERO-Plus \\
Global batch size & $64$ & $256$ \\
Learning rate & $1 \times 10^{-4}$ & $1 \times 10^{-4}$ \\
Weight decay & $1 \times 10^{-2}$ & $1 \times 10^{-2}$ \\
Max steps (\textit{Object}) & $1{,}500$ & $4{,}000$ \\
Max steps (\textit{Spatial} / \textit{Goal} / \textit{Long}) & $15{,}000$ & $10{,}000$ \\
Compute & $8\,\times$ NVIDIA RTX A6000 & $8\,\times$ NVIDIA RTX A6000 \\
\midrule
\multicolumn{3}{l}{\textbf{Optimizer configurations}\,$^{\dagger}$ (applied to action (A), vision (V), and language (L) modules):} \\
\multicolumn{3}{l}{\quad (i)~~\textbf{AdamW} on all modules.} \\
\multicolumn{3}{l}{\quad (ii)~\textbf{Muon} on the 2D matrices of A, V, and L; AdamW on all remaining parameters.} \\
\multicolumn{3}{l}{\quad (iii) \textbf{Pion} on the 2D matrices of A, \textbf{Muon} on those of V and L; AdamW elsewhere.} \\
\multicolumn{3}{l}{\footnotesize $^{\dagger}$ The 2D weight matrices exclude token embeddings and the output (LM-head) layer.} \\
\bottomrule
\end{tabular}
\end{table}

\begin{table}[htbp]
\centering
\caption{Training and rollout hyperparameters for the RLVR experiments. The three optimizer configurations (i)--(iii) are reused across the two RL algorithms (GRPO and GMPO) and the two model scales (Qwen3-1.7B and Qwen3-4B), and share all other hyperparameters listed in this table within each benchmark.}
\label{tab:rlvr_training_config}
\begin{tabular}{lcc}
\toprule
\textbf{Item} & \textbf{MATH} & \textbf{GSM8K} \\
\midrule
Base model & Qwen3-1.7B and Qwen3-4B & Qwen3-1.7B and Qwen3-4B \\
Algorithm & GRPO and GMPO & GRPO and GMPO \\
Train dataset & MATH levels 3--5 & GSM8K (train split) \\
Test dataset & MATH500 & GSM8K (test split) \\
Max prompt / response length & $1{,}024$ / $3{,}000$ & $512$ / $1{,}024$ \\
Train batch (prompts) & $128$ & $1{,}024$ \\
Rollout group size $n$ & $8$ & $5$ \\
Rollout temperature / Top-$p$ & $1.0$ / $1.0$ & $1.0$ / $1.0$ \\
Learning rate & $1 \times 10^{-6}$ & $1 \times 10^{-6}$ \\
Total training steps & $80$ & $40$ \\
Compute & $2\,\times$ NVIDIA H100 & $2\,\times$ NVIDIA H100 \\
\midrule
\multicolumn{3}{l}{\textbf{Optimizer configurations}\,$^{\dagger}$:} \\
\multicolumn{3}{l}{\quad (i)~~\textbf{AdamW} on all parameters.} \\
\multicolumn{3}{l}{\quad (ii)~\textbf{Muon} on all 2D weight matrices; AdamW elsewhere.} \\
\multicolumn{3}{l}{\quad (iii) \textbf{Pion} (per-head mode) on all 2D weight matrices; AdamW elsewhere.} \\
\multicolumn{3}{l}{\footnotesize $^{\dagger}$ The 2D weight matrices exclude token embeddings and the output (LM-head) layer.} \\
\bottomrule
\end{tabular}
\end{table}

\begin{table}[htbp]
\centering
\caption{Hardware, training, and rollout configuration for the real-robot evaluation. The three optimizer configurations (i)--(iii) share all other settings listed in this table; the only altered variable is the optimizer assignment.}
\label{tab:real_robot_training_config}
\begin{tabular}{lc}
\toprule
\textbf{Item} & \textbf{Real-robot ($\pi_{0.5}$ on three grasp-and-place tasks)} \\
\midrule
Backbone VLA & $\pi_{0.5}$ \\
Robot platform & Franka Research 3 (7-DoF) \\
Hardware setup & DROID setup\\
Cameras (input) & one third-view camera $+$ one wrist-mounted camera \\
Tasks & \textit{Cucumber\,$\to$\,Plate}, \textit{Cube\,$\to$\,Plate}, \textit{Cube\,$\to$\,Bowl} \\
Demonstrations & $200$ teleoperated trajectories \\
Total training steps & $20{,}000$ \\
Trials per (optimizer, task) & $30$ (randomized initial pose), $\leq 300$ control steps each \\
Evaluation metric & trial-level success rate (\#successes / 30) \\
\midrule
\multicolumn{2}{l}{\textbf{Optimizer configurations}\,$^{\dagger}$ (applied to action (A), vision (V), and language (L) modules):} \\
\multicolumn{2}{l}{\quad (i)~~\textbf{AdamW} on all parameters.} \\
\multicolumn{2}{l}{\quad (ii)~\textbf{Muon} on the 2D matrices of A, V, and L; AdamW on all remaining parameters.} \\
\multicolumn{2}{l}{\quad (iii) \textbf{Pion} on the 2D matrices of A, \textbf{Muon} on those of V and L; AdamW elsewhere.} \\
\multicolumn{2}{l}{\footnotesize $^{\dagger}$ The 2D weight matrices exclude token embeddings and the output (LM-head) layer.} \\
\bottomrule
\end{tabular}
\end{table}

\clearpage
\newpage

\section{Qualitative rollouts}\label{sec: qualitative_examples}
\subsection{LIBERO Object}\label{sec: examples_object}

\vspace{-1mm}
\noindent\includegraphics[width=0.9\textwidth]{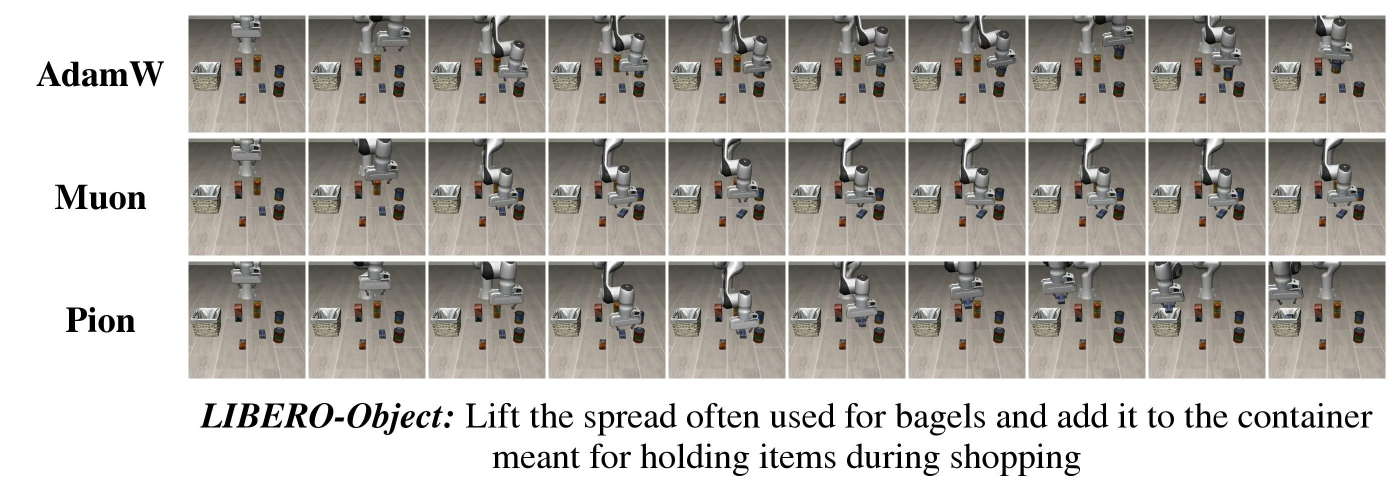}
\vspace{-1mm}

\vspace{-1mm}
\noindent\includegraphics[width=0.9\textwidth]{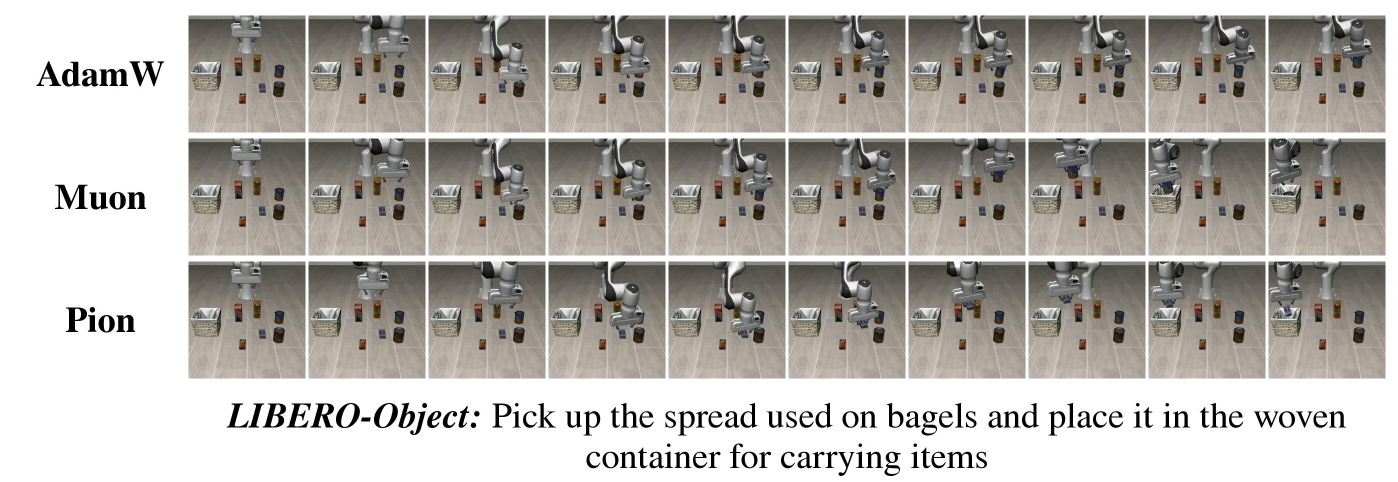}
\vspace{-1mm}

\vspace{-1mm}
\noindent\includegraphics[width=0.9\textwidth]{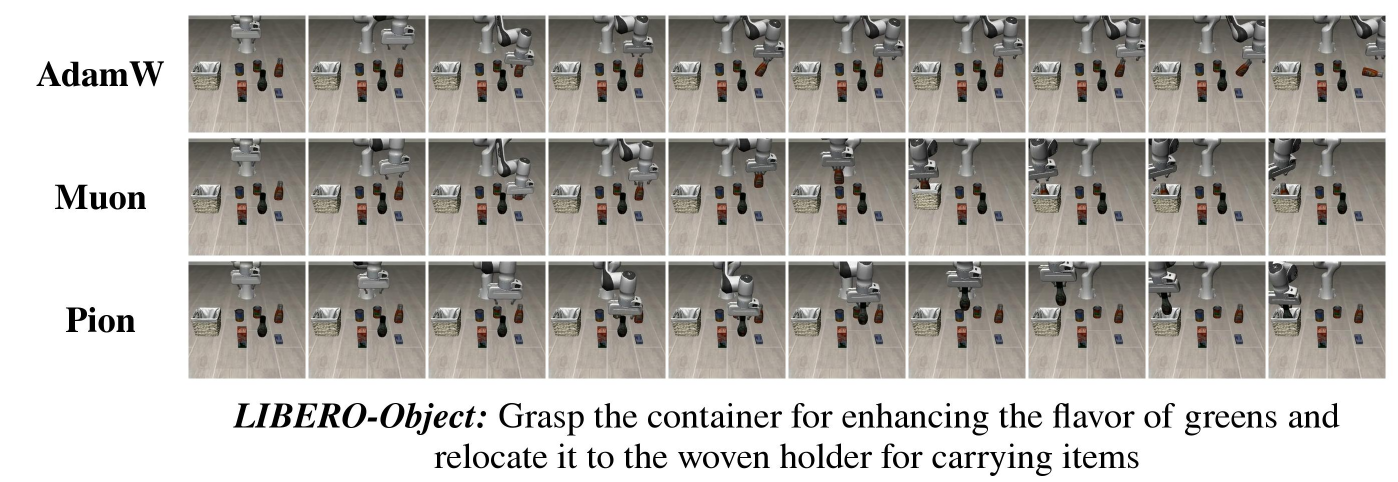}
\vspace{-1mm}

\vspace{-1mm}
\noindent\includegraphics[width=0.9\textwidth]{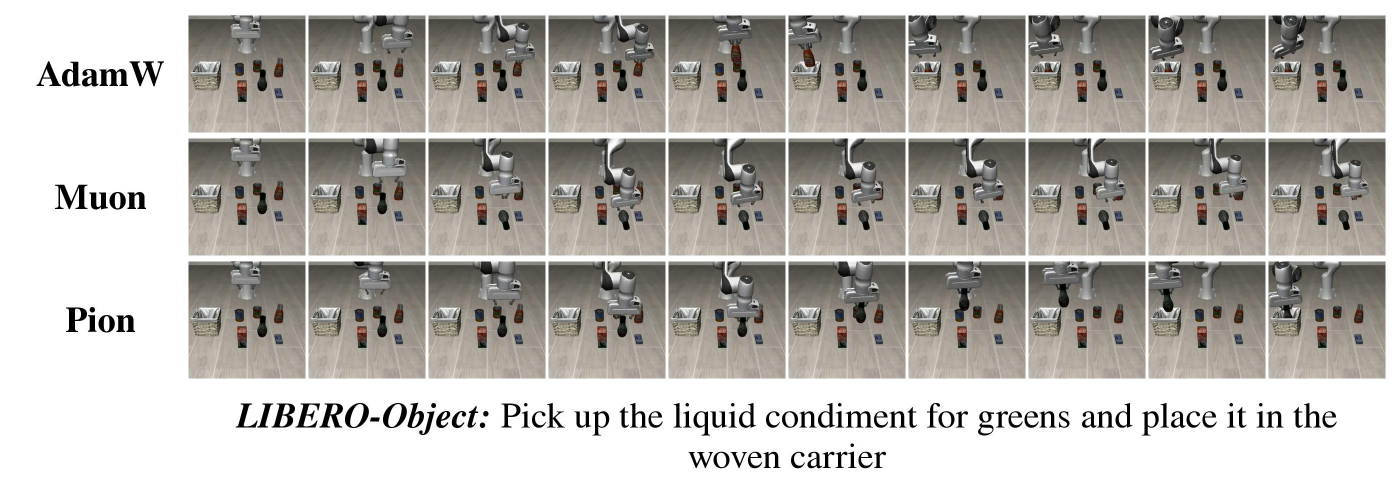}
\vspace{-1mm}

\vspace{-1mm}
\noindent\includegraphics[width=0.9\textwidth]{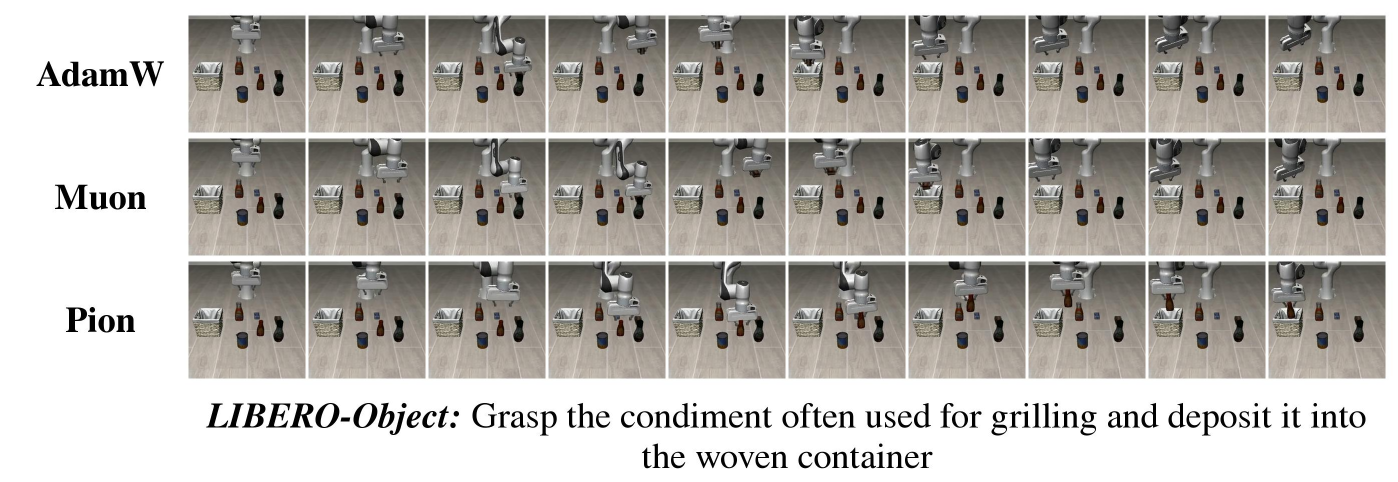}
\vspace{-1mm}

\vspace{-1mm}
\noindent\includegraphics[width=0.9\textwidth]{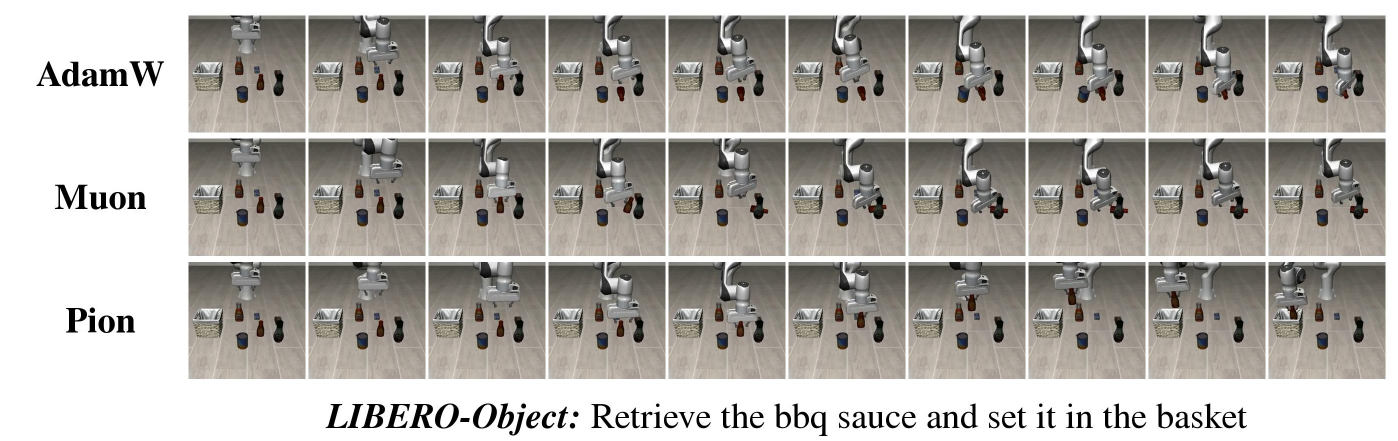}
\vspace{-1mm}

\vspace{-1mm}
\noindent\includegraphics[width=0.9\textwidth]{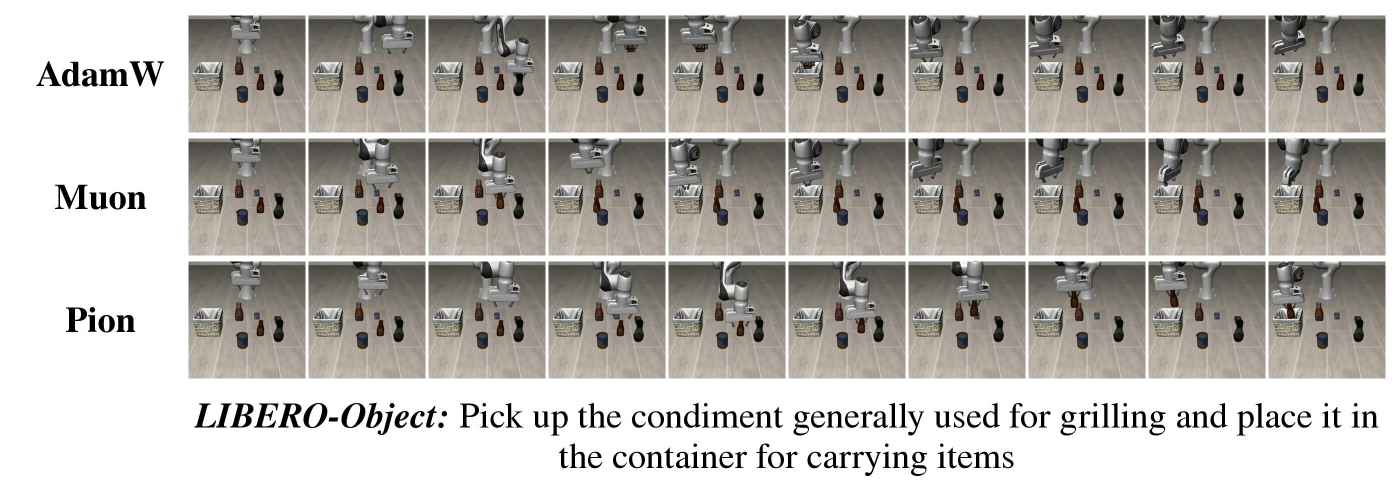}
\vspace{-1mm}

\vspace{-1mm}
\noindent\includegraphics[width=0.9\textwidth]{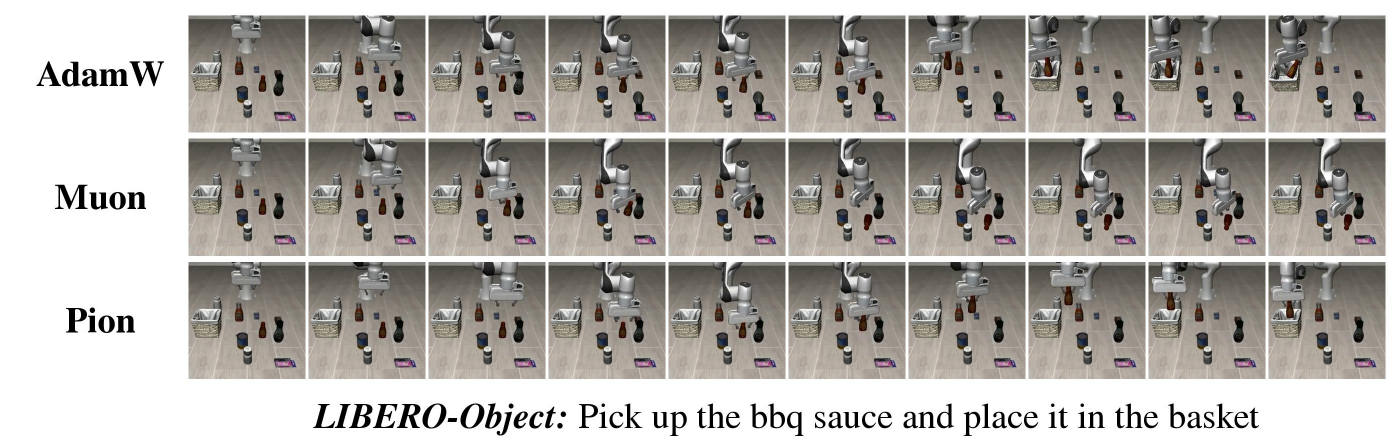}
\vspace{-1mm}

\clearpage

\subsection{LIBERO Spatial}\label{sec: examples_spatial}

\vspace{-1mm}
\noindent\includegraphics[width=0.9\textwidth]{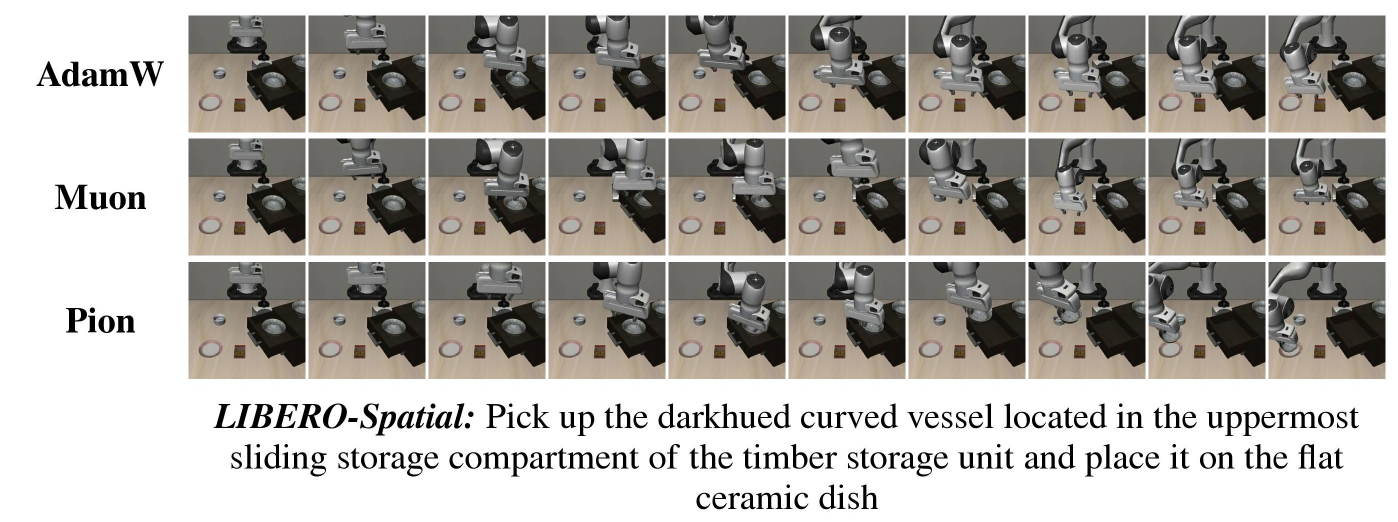}
\vspace{-1mm}

\vspace{-1mm}
\noindent\includegraphics[width=0.9\textwidth]{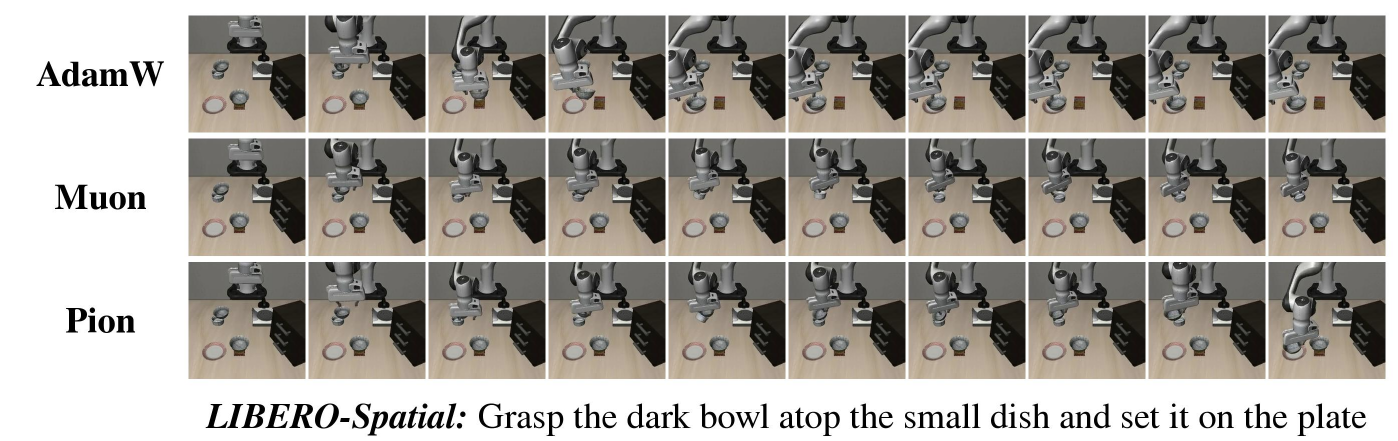}
\vspace{-1mm}

\vspace{-1mm}
\noindent\includegraphics[width=0.9\textwidth]{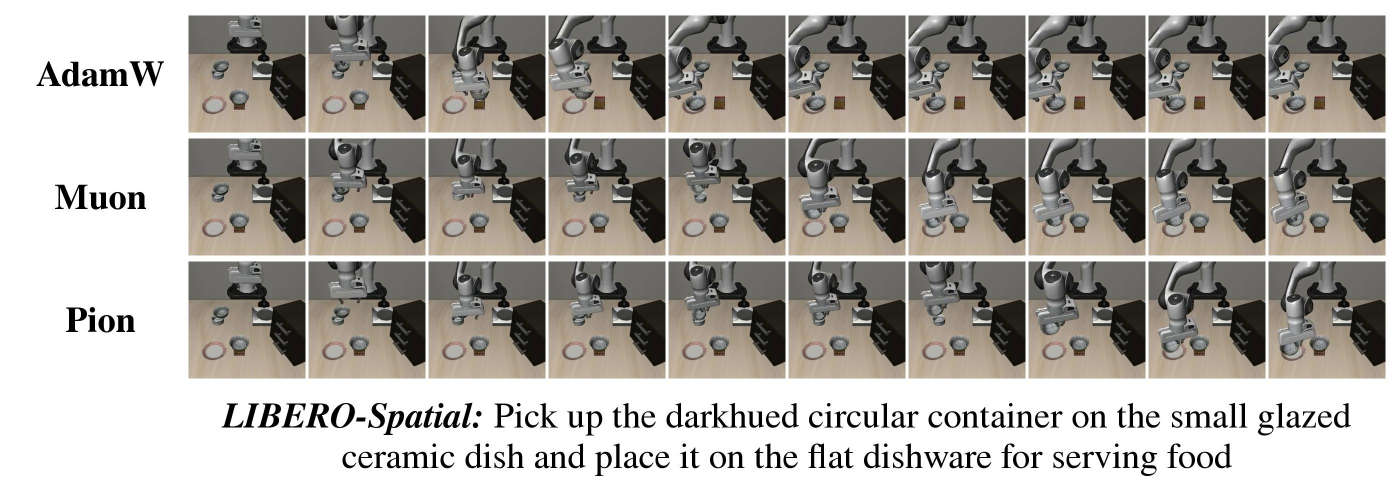}
\vspace{-1mm}

\vspace{-1mm}
\noindent\includegraphics[width=0.9\textwidth]{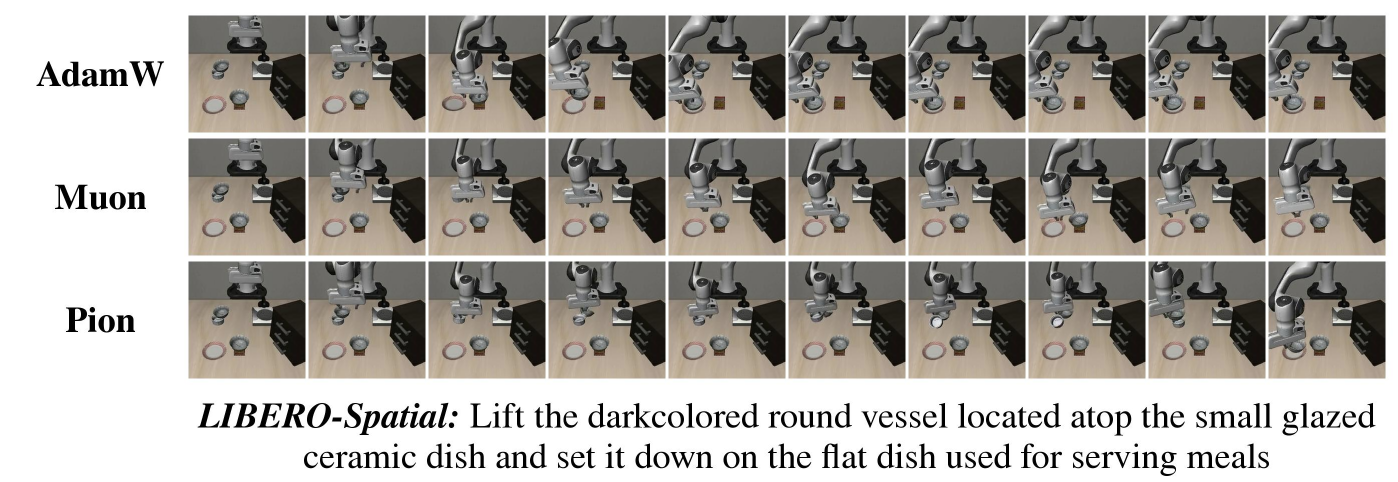}
\vspace{-1mm}

\vspace{-1mm}
\noindent\includegraphics[width=0.9\textwidth]{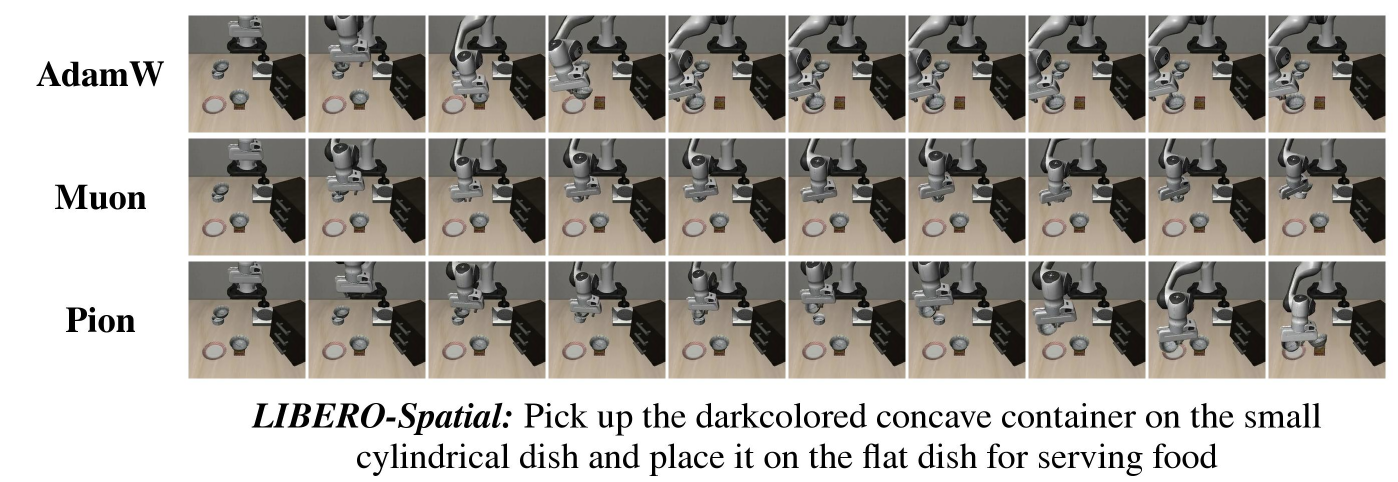}
\vspace{-1mm}

\vspace{-1mm}
\noindent\includegraphics[width=0.9\textwidth]{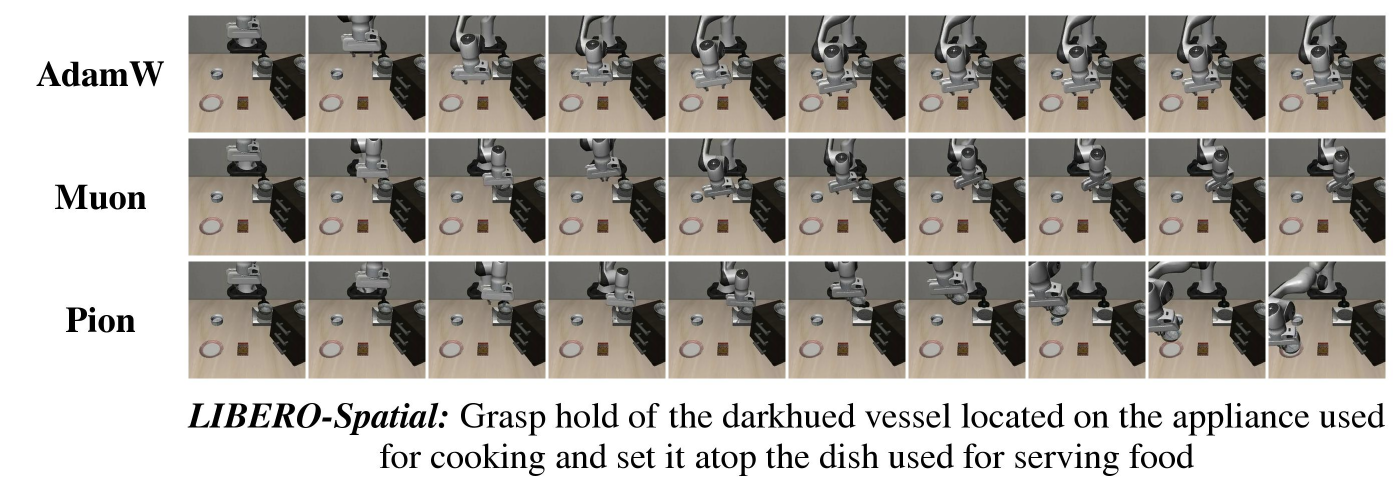}
\vspace{-1mm}

\vspace{-1mm}
\noindent\includegraphics[width=0.9\textwidth]{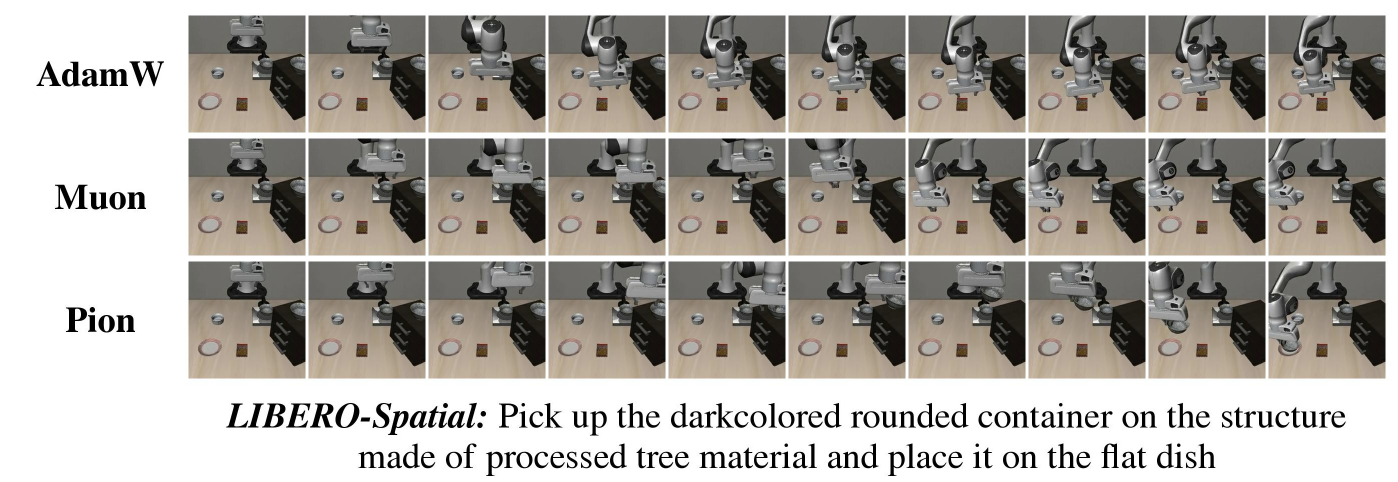}
\vspace{-1mm}

\vspace{-1mm}
\noindent\includegraphics[width=0.9\textwidth]{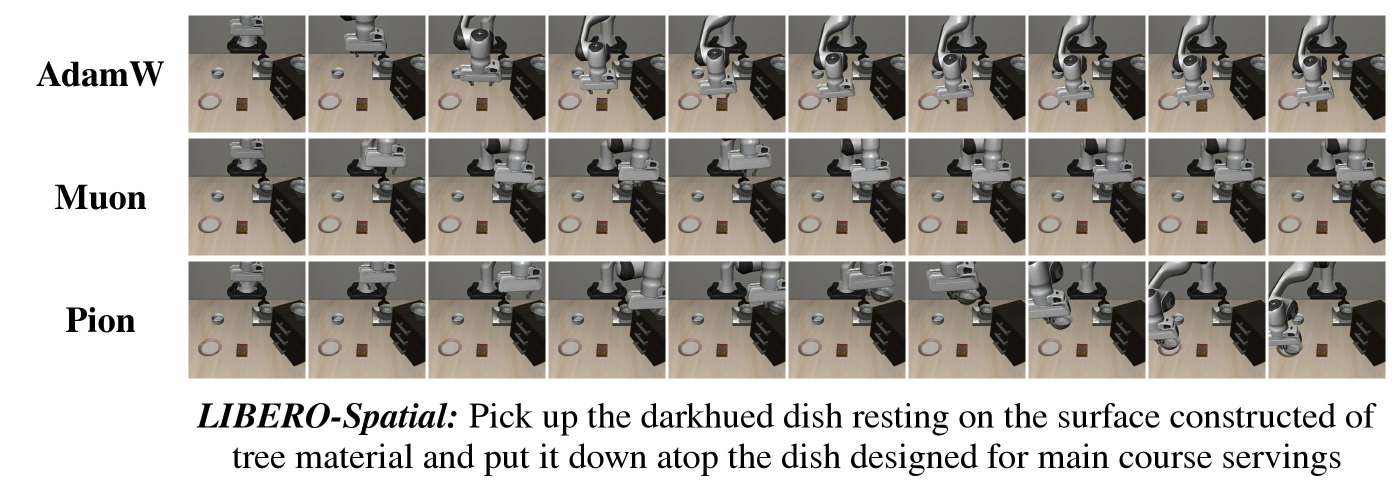}
\vspace{-1mm}

\clearpage

\subsection{LIBERO Goal}\label{sec: examples_goal}

\vspace{-1mm}
\noindent\includegraphics[width=0.9\textwidth]{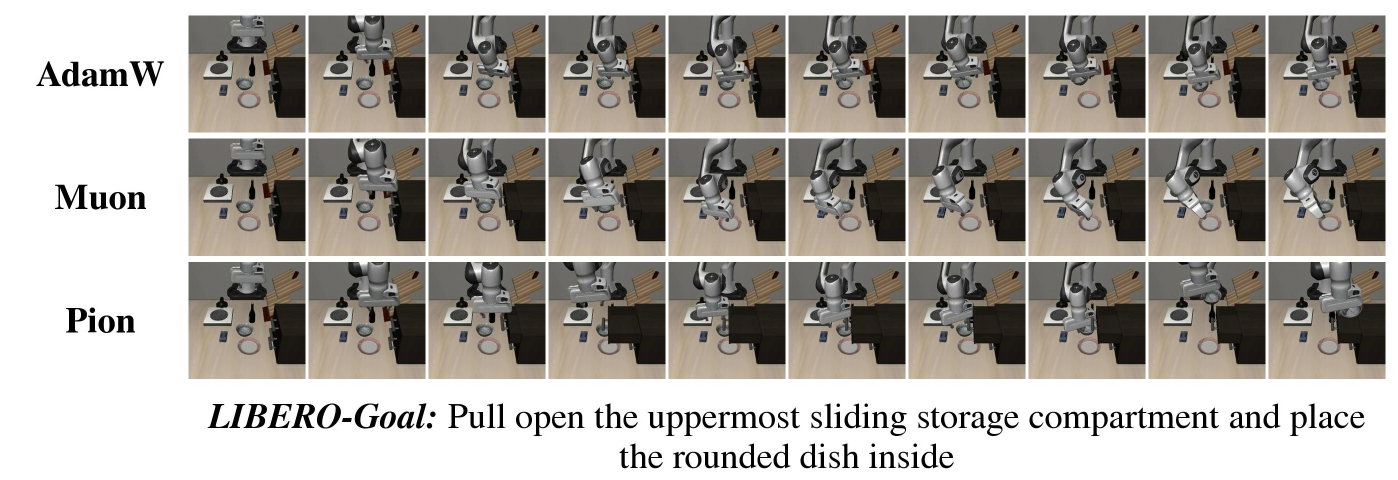}
\vspace{-1mm}

\vspace{-1mm}
\noindent\includegraphics[width=0.9\textwidth]{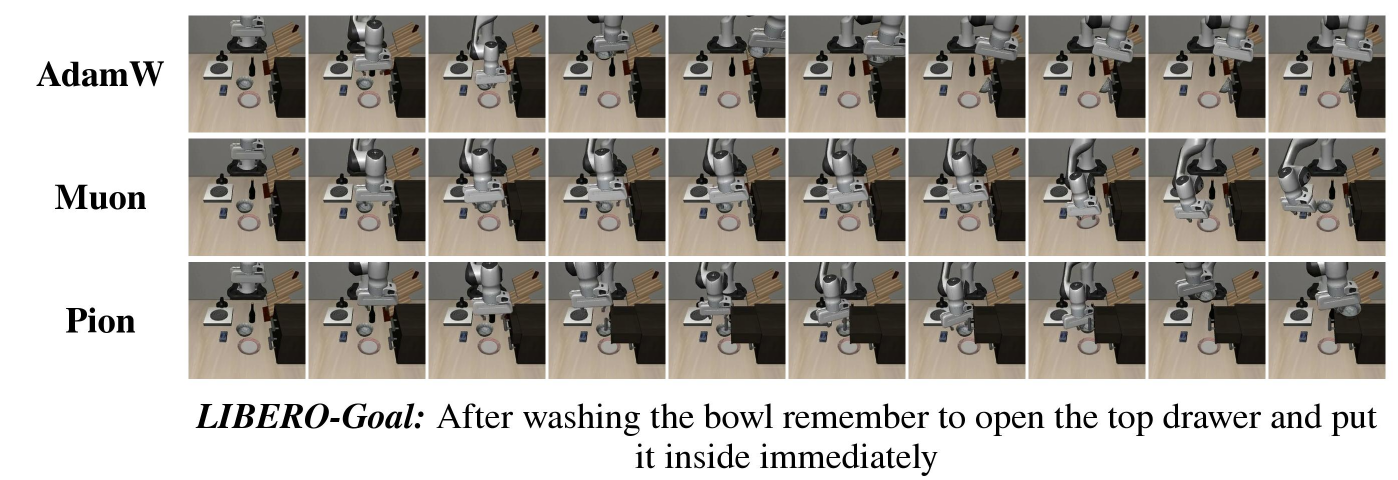}
\vspace{-1mm}

\vspace{-1mm}
\noindent\includegraphics[width=0.9\textwidth]{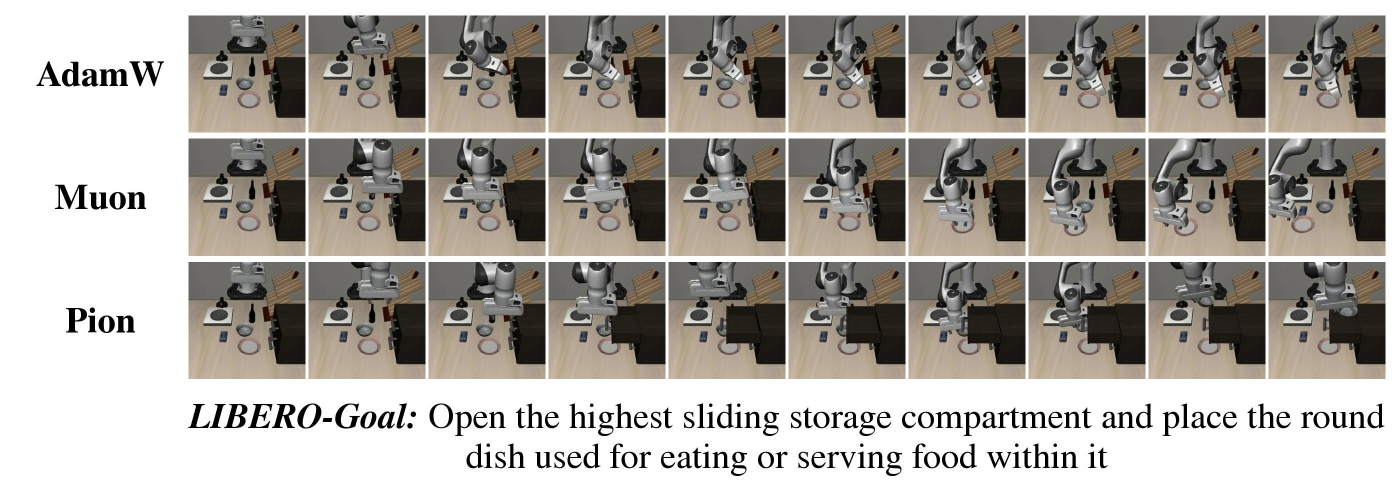}
\vspace{-1mm}

\vspace{-1mm}
\noindent\includegraphics[width=0.9\textwidth]{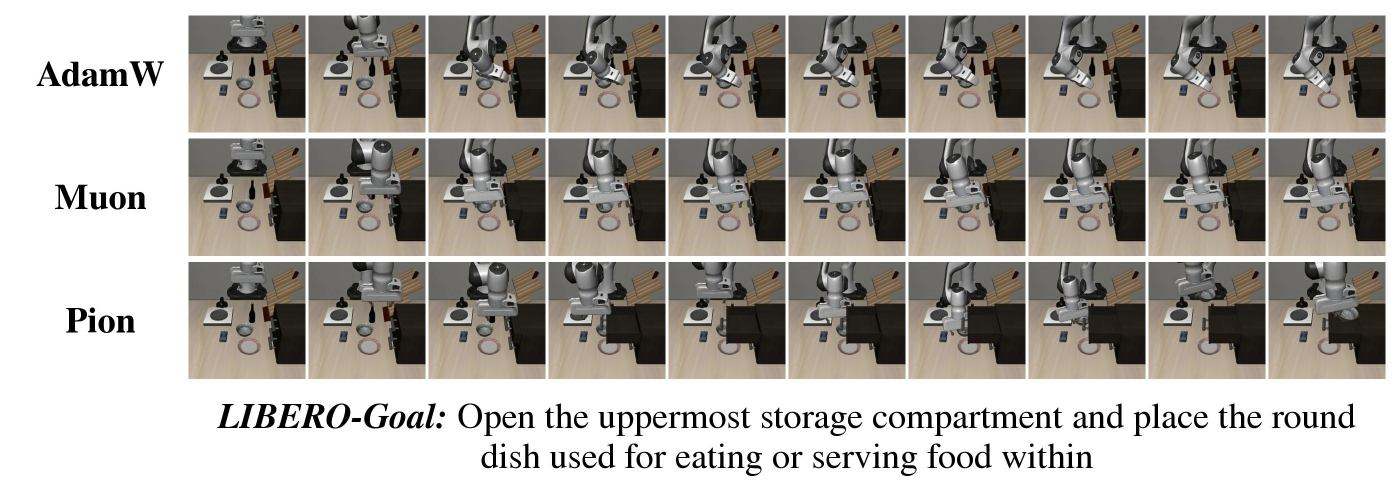}
\vspace{-1mm}

\vspace{-1mm}
\noindent\includegraphics[width=0.9\textwidth]{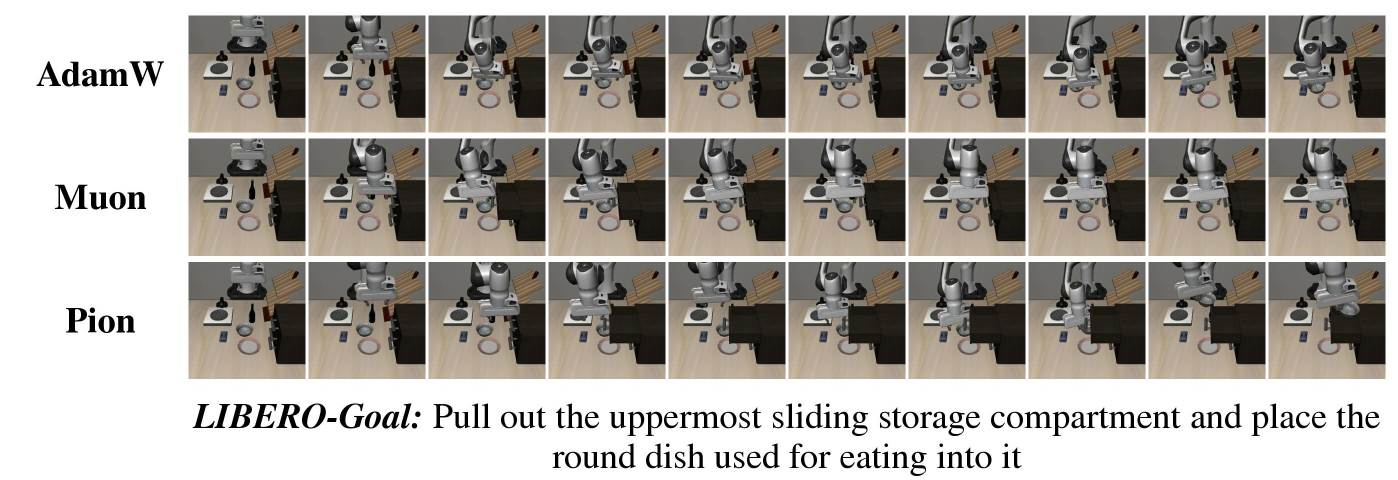}
\vspace{-1mm}

\vspace{-1mm}
\noindent\includegraphics[width=0.9\textwidth]{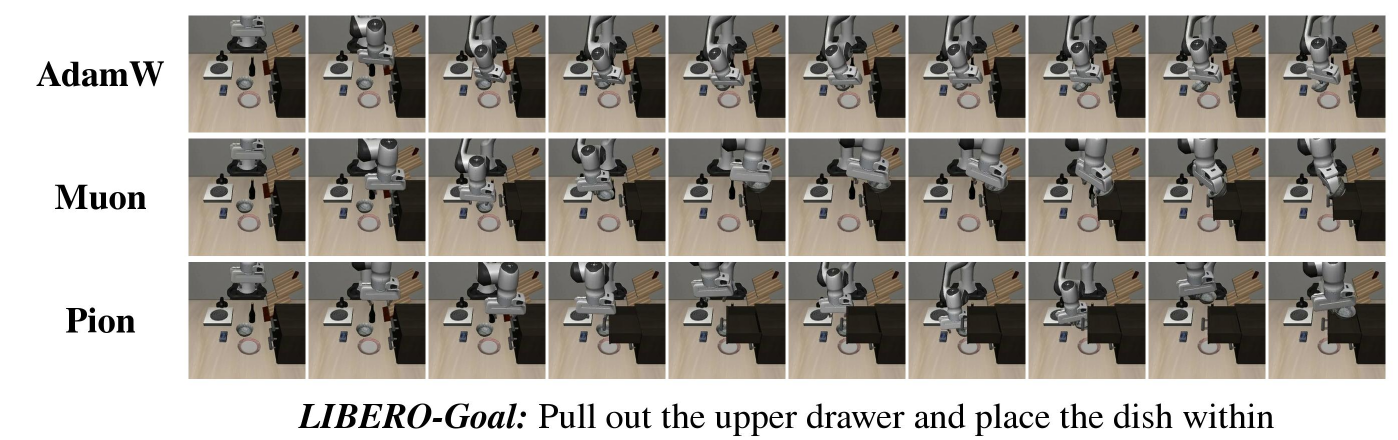}
\vspace{-1mm}

\vspace{-1mm}
\noindent\includegraphics[width=0.9\textwidth]{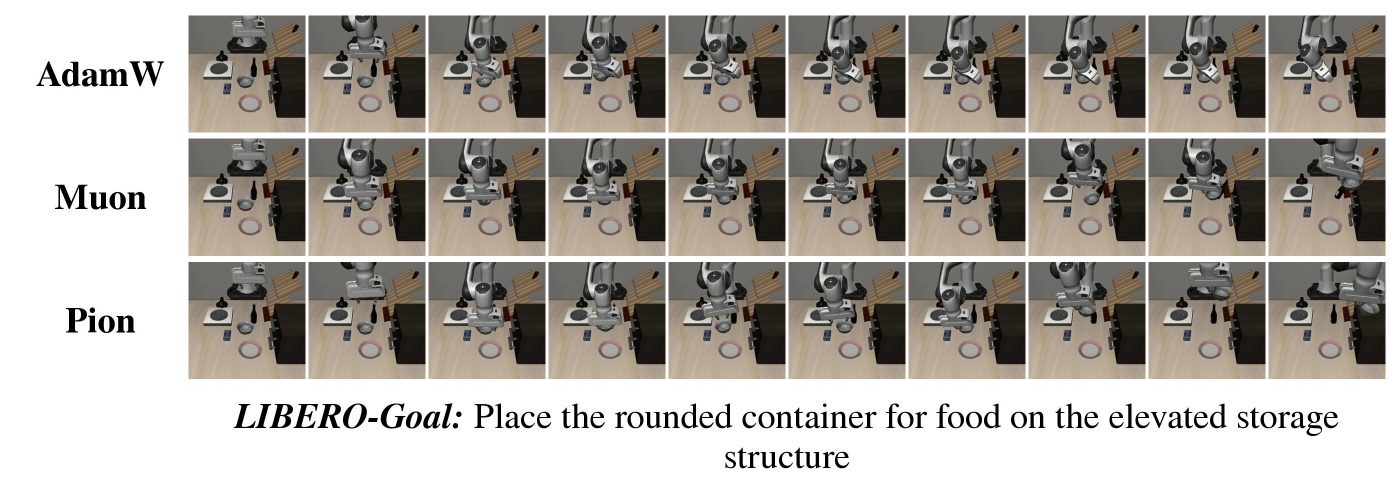}
\vspace{-1mm}

\vspace{-1mm}
\noindent\includegraphics[width=0.9\textwidth]{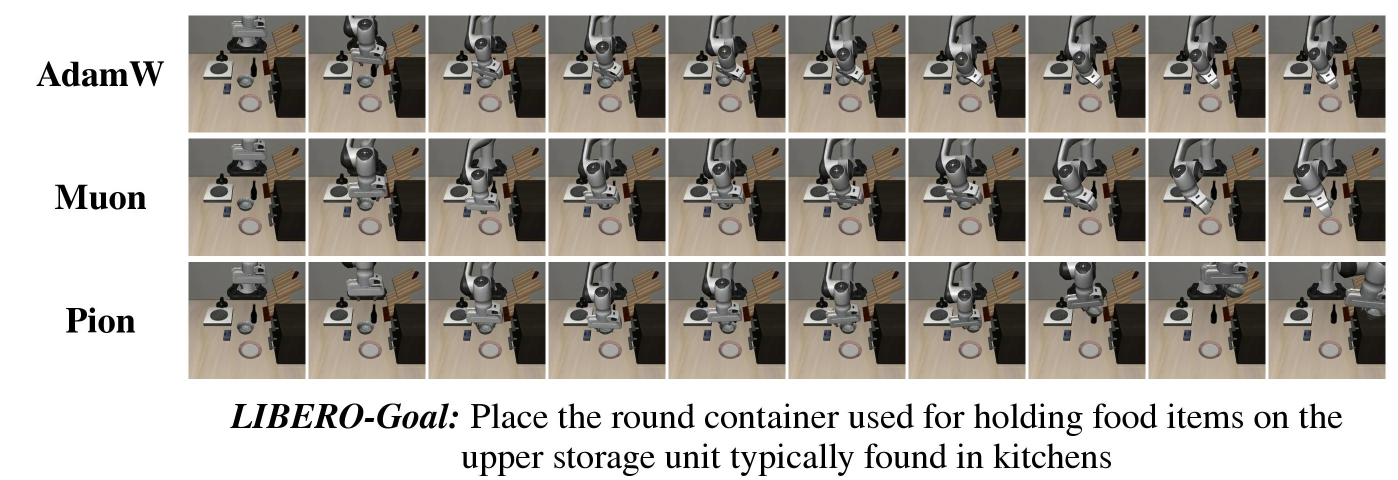}
\vspace{-1mm}

\clearpage

\subsection{LIBERO Long}\label{sec: examples_long}

\vspace{-1mm}
\noindent\includegraphics[width=0.9\textwidth]{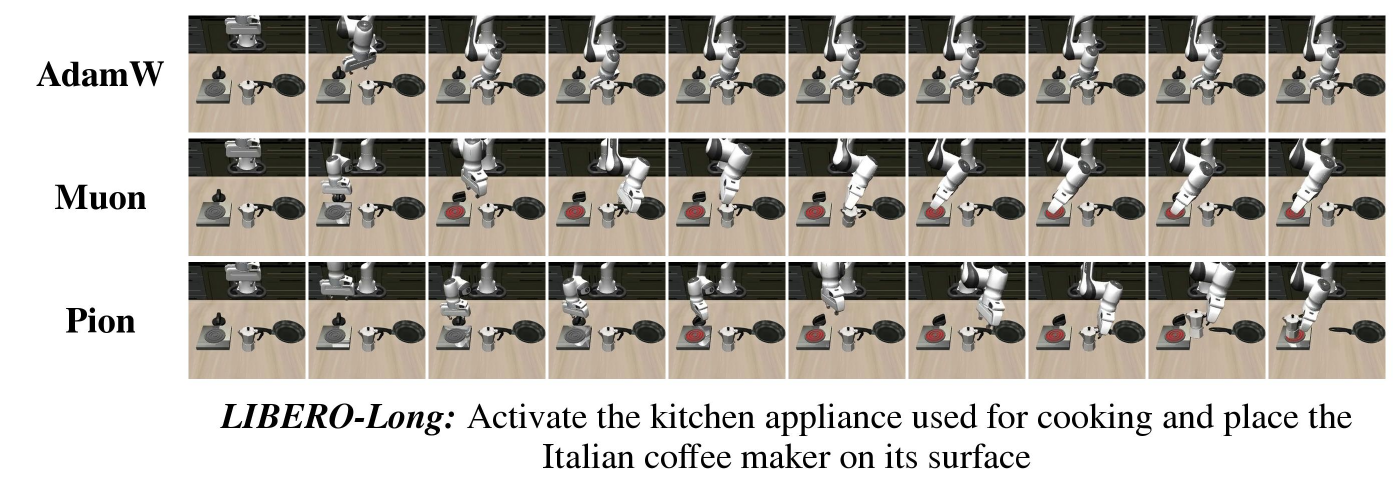}
\vspace{-1mm}

\vspace{-1mm}
\noindent\includegraphics[width=0.9\textwidth]{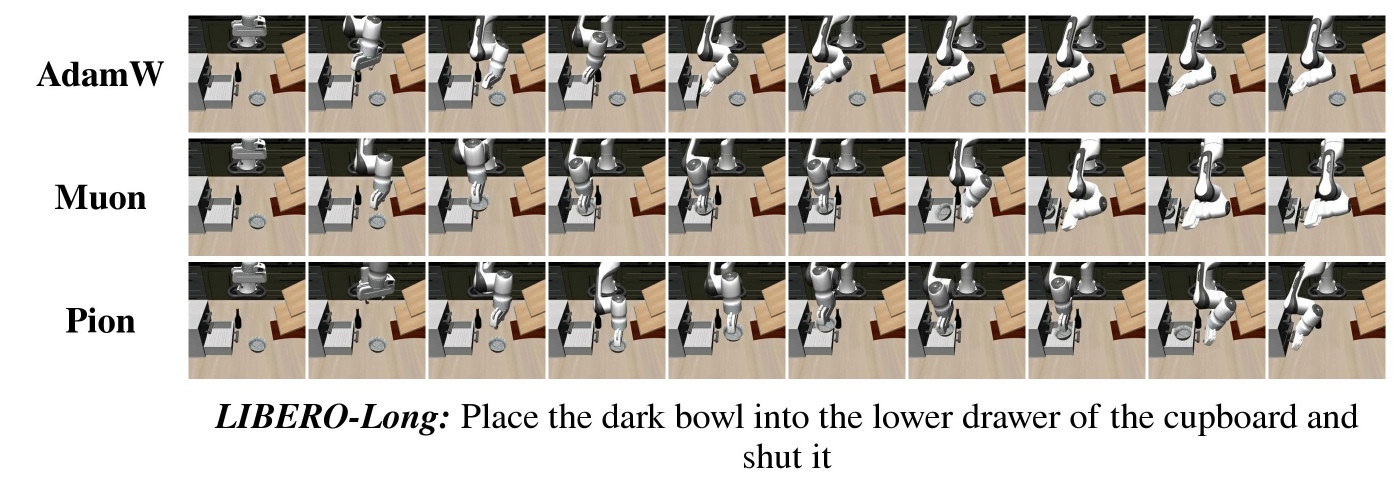}
\vspace{-1mm}

\vspace{-1mm}
\noindent\includegraphics[width=0.9\textwidth]{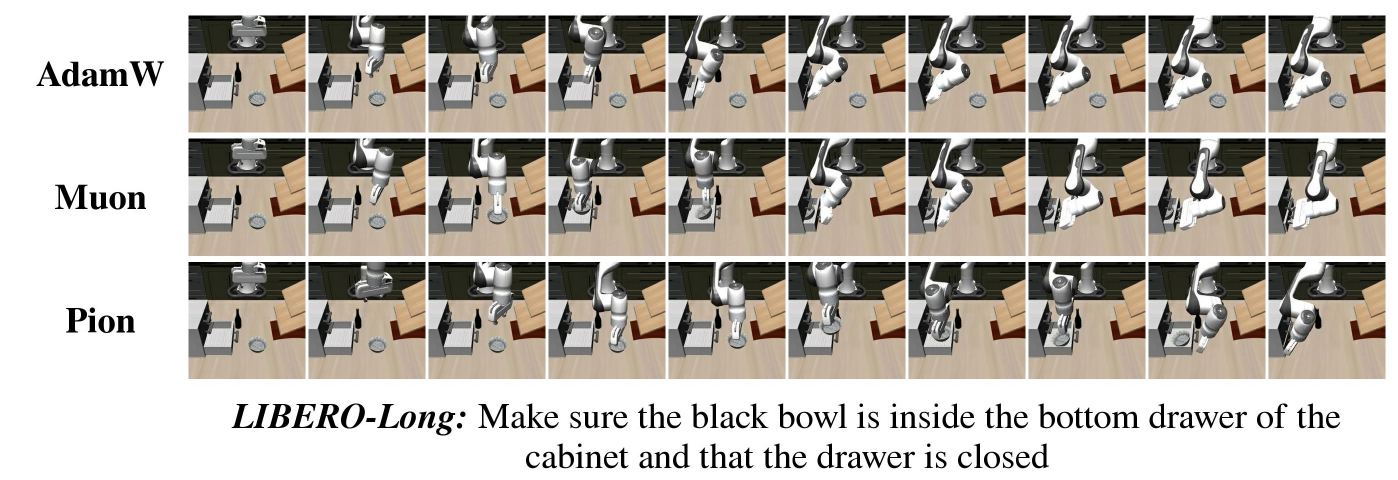}
\vspace{-1mm}

\vspace{-1mm}
\noindent\includegraphics[width=0.9\textwidth]{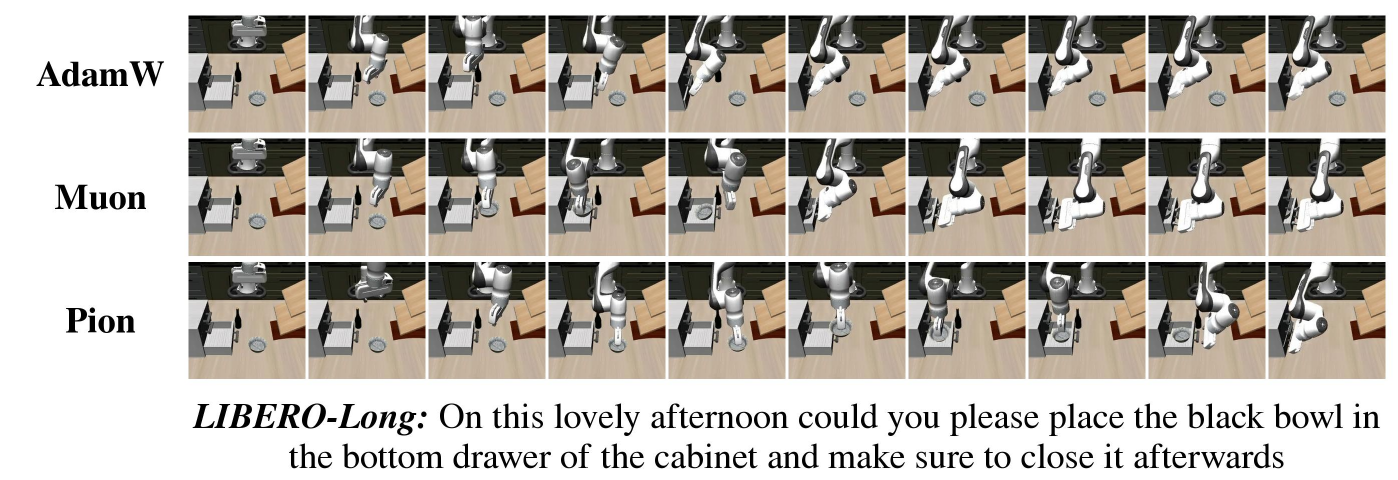}
\vspace{-1mm}

\vspace{-1mm}
\noindent\includegraphics[width=0.9\textwidth]{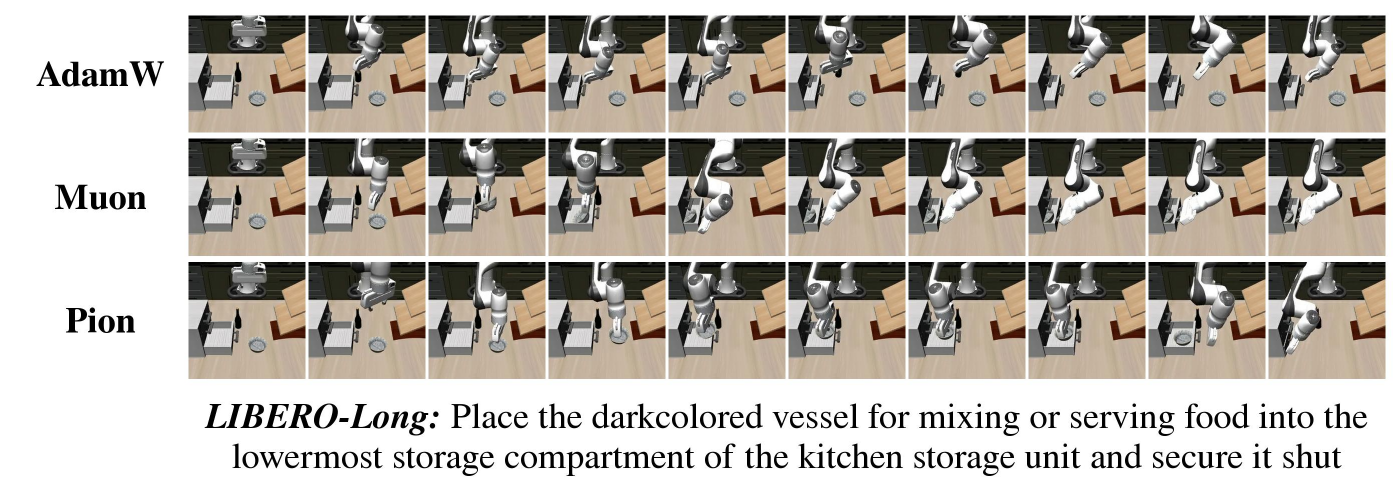}
\vspace{-1mm}

\vspace{-1mm}
\noindent\includegraphics[width=0.9\textwidth]{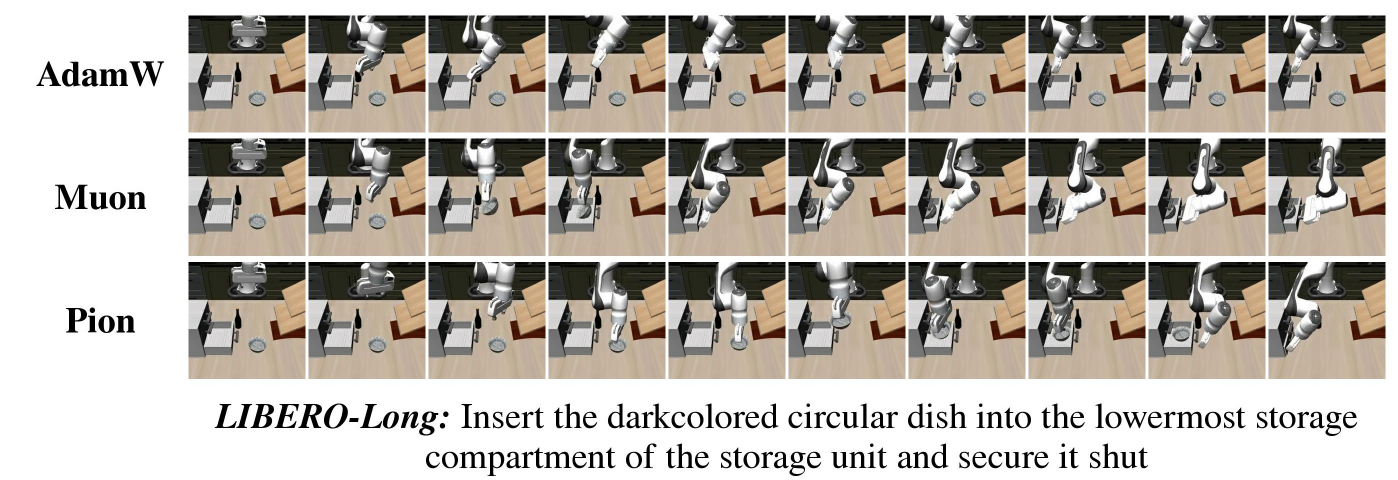}
\vspace{-1mm}

\vspace{-1mm}
\noindent\includegraphics[width=0.9\textwidth]{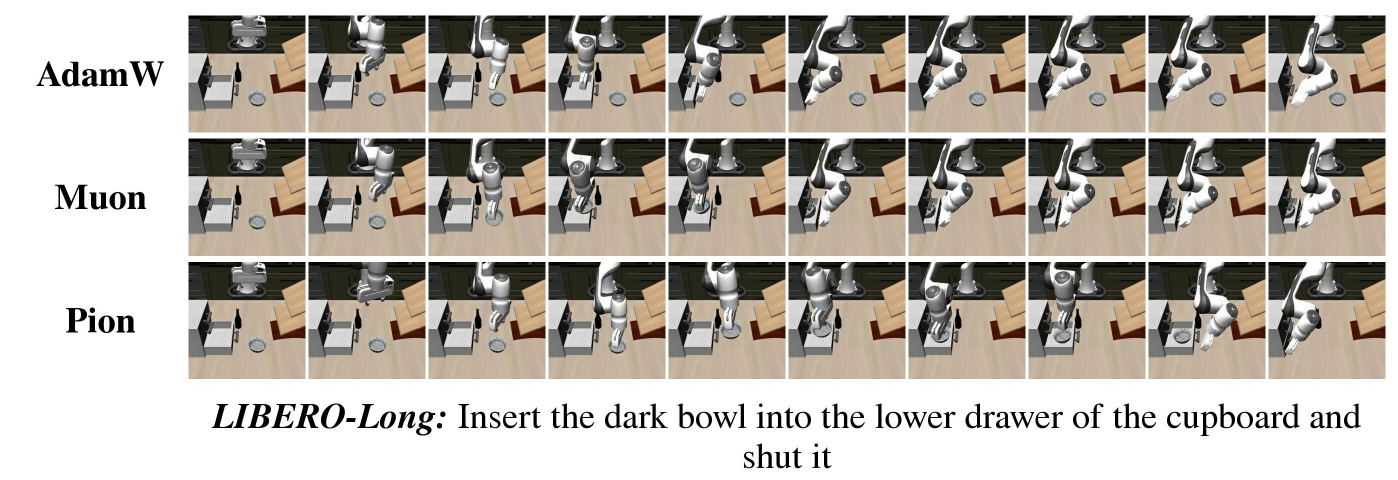}
\vspace{-1mm}

\vspace{-1mm}
\noindent\includegraphics[width=0.9\textwidth]{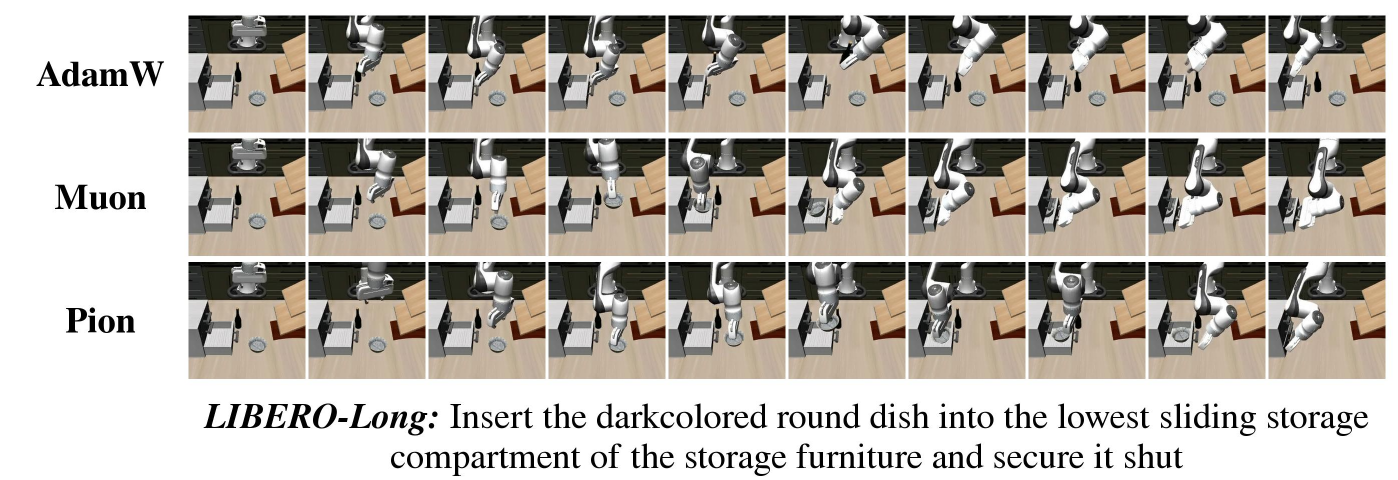}
\vspace{-1mm}

\clearpage

\section{Visualization of Real-Robot Rollouts}
\label{sec:real_robot}

\textbf{Table\,\ref{tab:real_robot_rollouts}} compares a single rollout of $\pi_{0.5}$ trained with AdamW, Muon, and Pion on each of the three tasks (\textit{Cucumber\,$\to$\,Plate}, \textit{Cube\,$\to$\,Plate}, \textit{Cube\,$\to$\,Bowl}, top to bottom). Each row shows $6$ frames uniformly sampled along that rollout, from approach to placement.

\emph{Cucumber\,$\to$\,Plate} (Table\,\ref{tab:real_robot_rollouts}, top): AdamW repeatedly attempts to grasp the cucumber but never lifts it off the table (frame $5$); Muon grasps it but opens the gripper prematurely, dropping the cucumber mid\mbox{-}transport (frame $3$); Pion grasps and places cleanly. \emph{Cube\,$\to$\,Plate} (Table\,\ref{tab:real_robot_rollouts}, middle): both AdamW and Muon open the gripper prematurely before reaching the plate (frame $3$ in either row), so the cube is released mid\mbox{-}air rather than on the plate, while Pion grasps and places the cube accurately. \emph{Cube\,$\to$\,Bowl} (Table\,\ref{tab:real_robot_rollouts}, bottom): on the hardest task, AdamW lifts the cube but not high enough to clear the rim of the bowl (frame $3$), and Muon misaligns the gripper with the cube and fails to establish a stable grasp (frame $3$); Pion deposits the cube inside the bowl, corroborating the quantitative gains in Table\,\ref{tab:real_robot}.

\begin{table*}[htbp]
    \centering
    \caption{Real-robot rollouts of $\pi_{0.5}$ trained with AdamW, Muon, and Pion on the three grasp-and-place tasks. Each row shows $6$ frames uniformly sampled along a single rollout, from approach to placement. The natural-language task prompt is shown above each task block (in gray).}
    \label{tab:real_robot_rollouts}
    \setlength{\tabcolsep}{3pt}
    \renewcommand{\arraystretch}{1.15}
    \small
    \begin{tabular}{@{}>{\centering\arraybackslash}m{0.10\textwidth}@{\hspace{4pt}}*{6}{>{\centering\arraybackslash}m{0.135\textwidth}}@{}}
    \toprule
    \multirow{2}{*}{\textbf{Optimizer}} & \multicolumn{6}{c}{\textbf{Frame index}} \\
    \cmidrule(lr){2-7}
    & \textbf{0} & \textbf{1} & \textbf{2} & \textbf{3} & \textbf{4} & \textbf{5} \\
    \midrule
    \multicolumn{7}{@{}>{\columncolor{black!8}}c@{}}{\textbf{Prompt:} \textit{``Pick up the cucumber and place it on the plate.''}} \\
    \midrule
    AdamW &
        \includegraphics[width=0.135\textwidth]{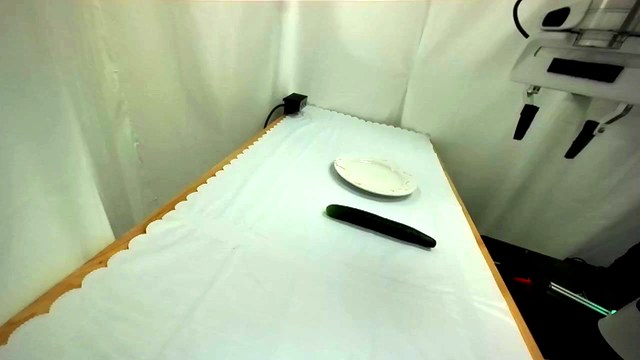} &
        \includegraphics[width=0.135\textwidth]{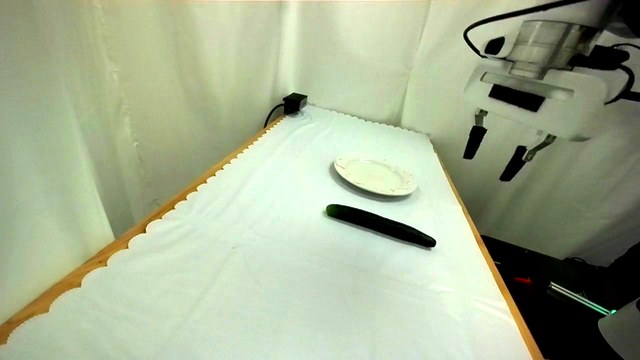} &
        \includegraphics[width=0.135\textwidth]{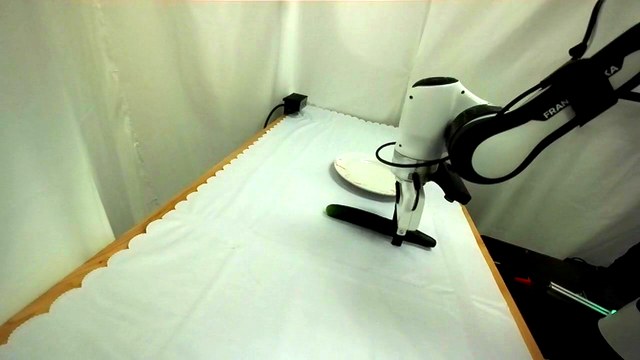} &
        \includegraphics[width=0.135\textwidth]{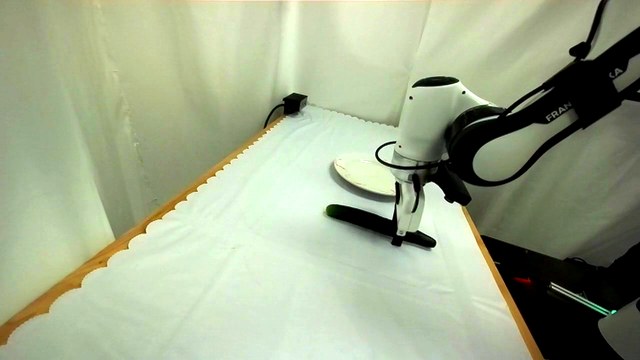} &
        \includegraphics[width=0.135\textwidth]{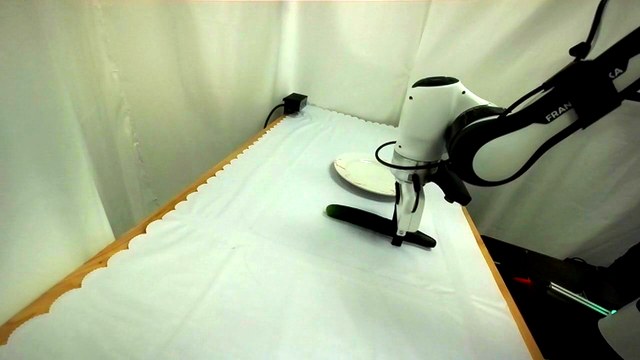} &
        \includegraphics[width=0.135\textwidth]{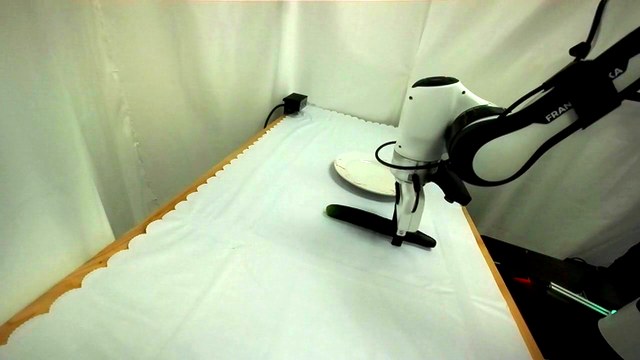} \\
    Muon &
        \includegraphics[width=0.135\textwidth]{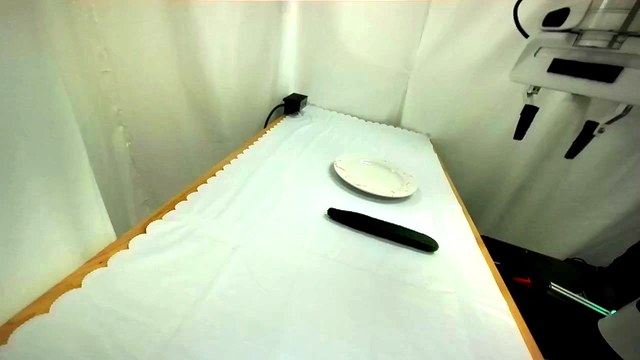} &
        \includegraphics[width=0.135\textwidth]{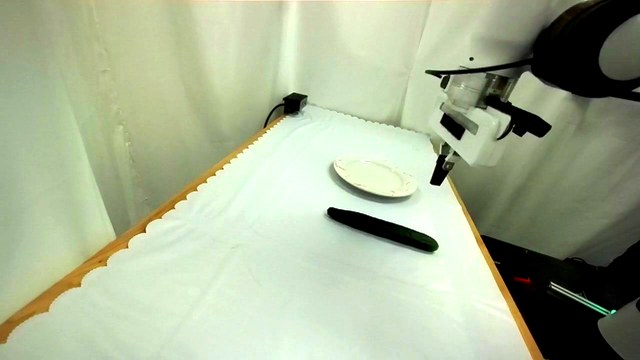} &
        \includegraphics[width=0.135\textwidth]{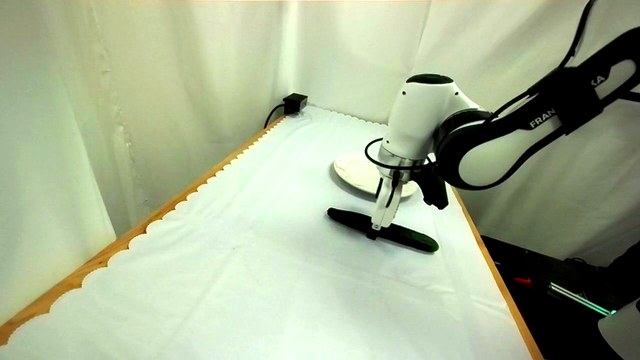} &
        \includegraphics[width=0.135\textwidth]{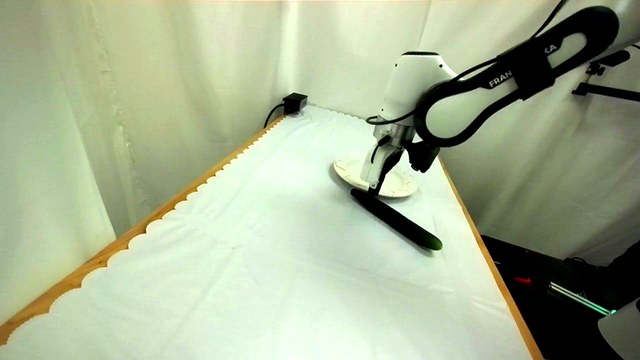} &
        \includegraphics[width=0.135\textwidth]{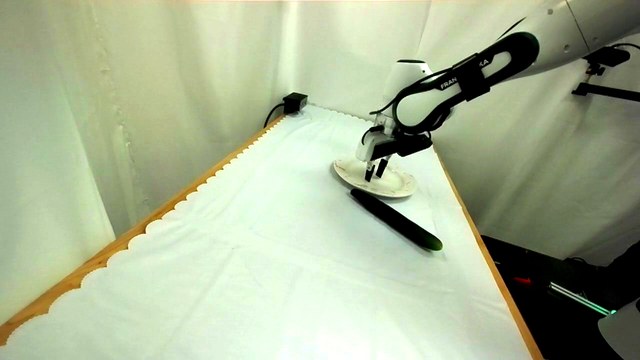} &
        \includegraphics[width=0.135\textwidth]{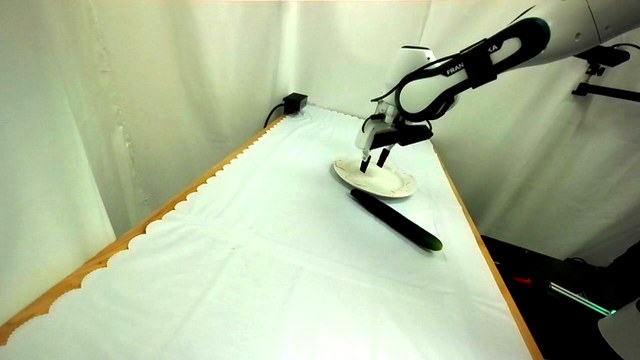} \\
    \textbf{Pion (Ours)} &
        \includegraphics[width=0.135\textwidth]{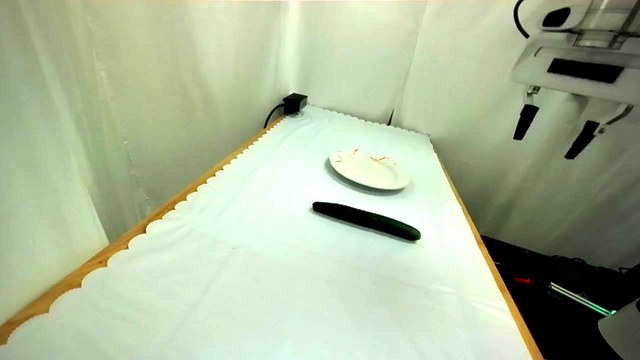} &
        \includegraphics[width=0.135\textwidth]{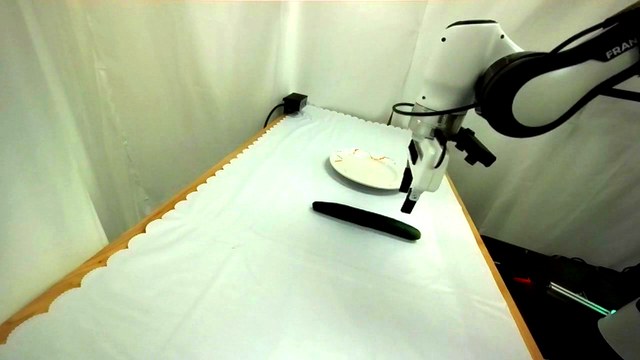} &
        \includegraphics[width=0.135\textwidth]{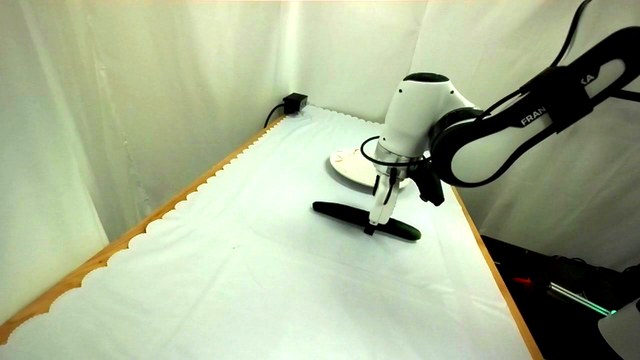} &
        \includegraphics[width=0.135\textwidth]{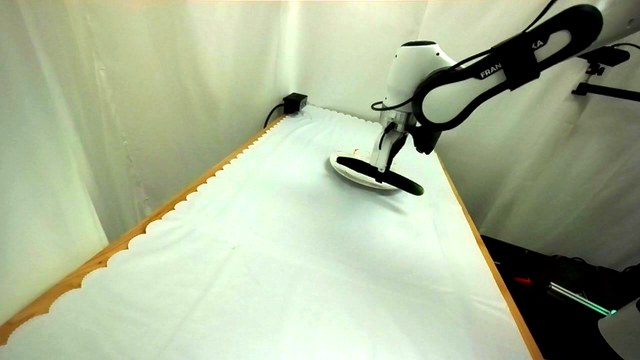} &
        \includegraphics[width=0.135\textwidth]{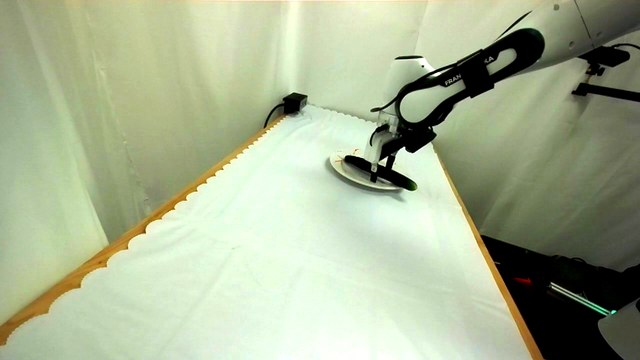} &
        \includegraphics[width=0.135\textwidth]{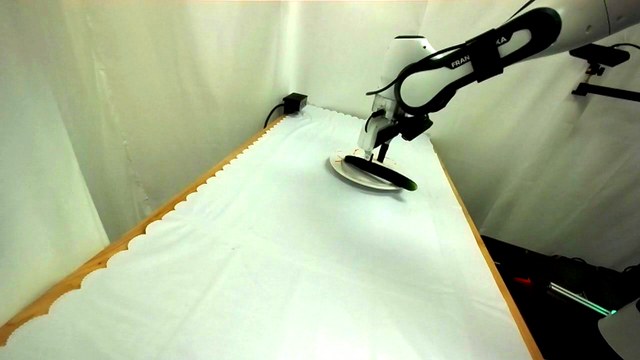} \\
    \midrule
    \multicolumn{7}{@{}>{\columncolor{black!8}}c@{}}{\textbf{Prompt:} \textit{``Pick up the cube and place it on the plate.''}} \\
    \midrule
    AdamW &
        \includegraphics[width=0.135\textwidth]{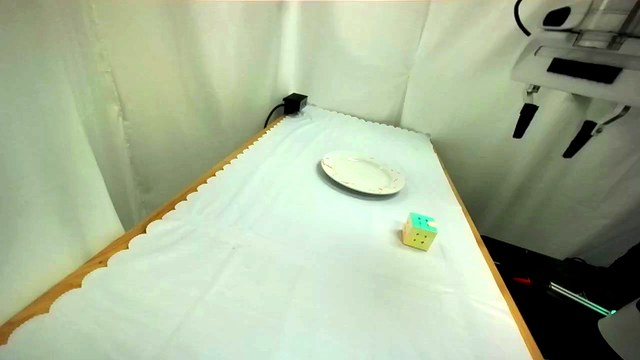} &
        \includegraphics[width=0.135\textwidth]{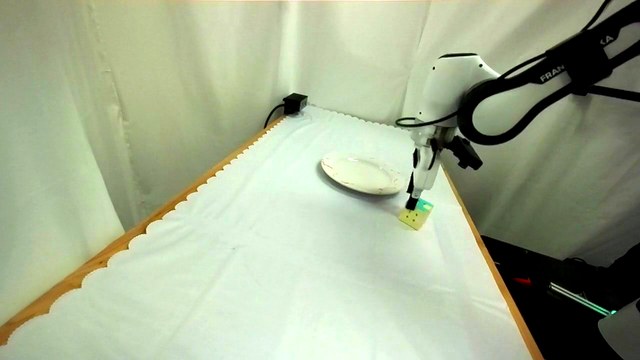} &
        \includegraphics[width=0.135\textwidth]{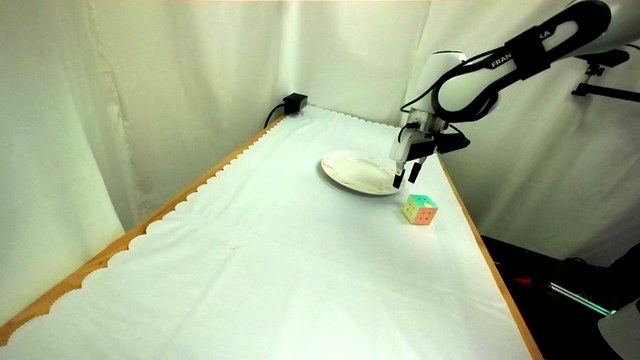} &
        \includegraphics[width=0.135\textwidth]{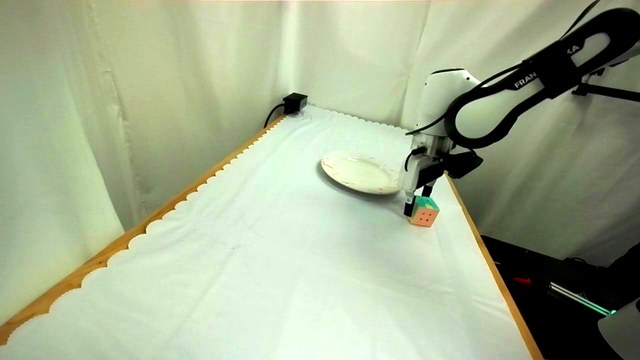} &
        \includegraphics[width=0.135\textwidth]{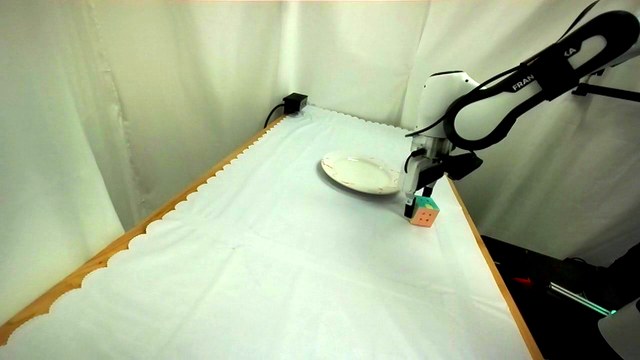} &
        \includegraphics[width=0.135\textwidth]{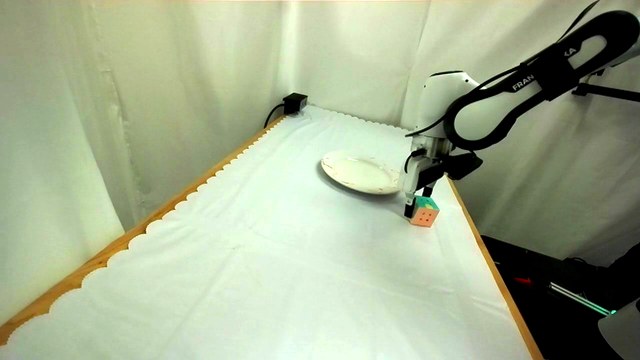} \\
    Muon &
        \includegraphics[width=0.135\textwidth]{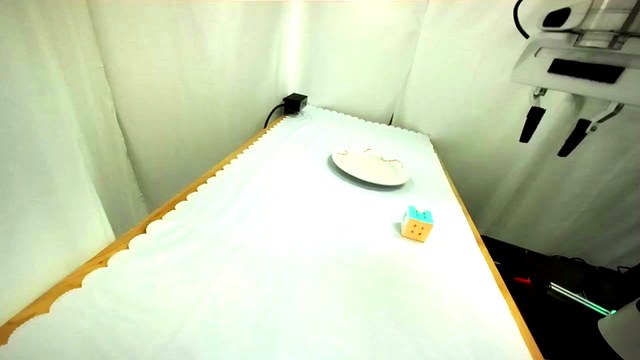} &
        \includegraphics[width=0.135\textwidth]{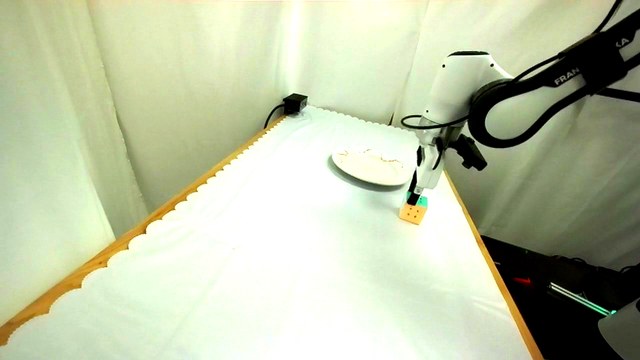} &
        \includegraphics[width=0.135\textwidth]{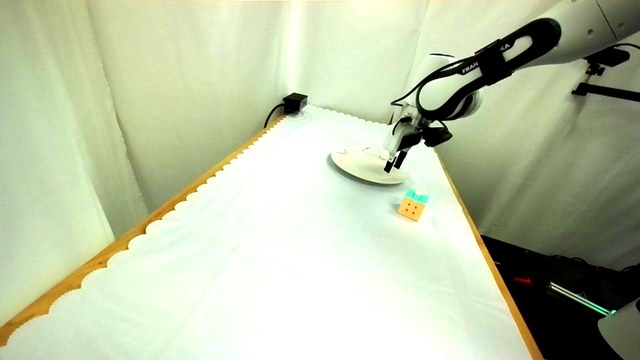} &
        \includegraphics[width=0.135\textwidth]{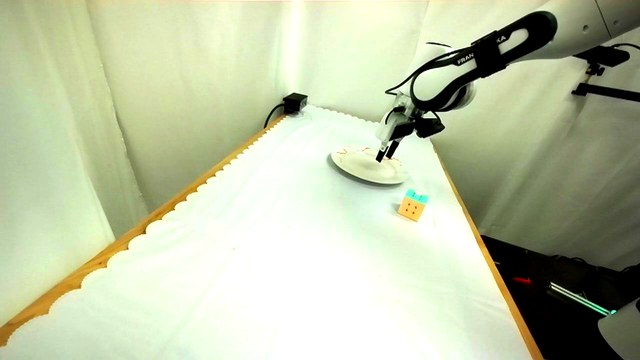} &
        \includegraphics[width=0.135\textwidth]{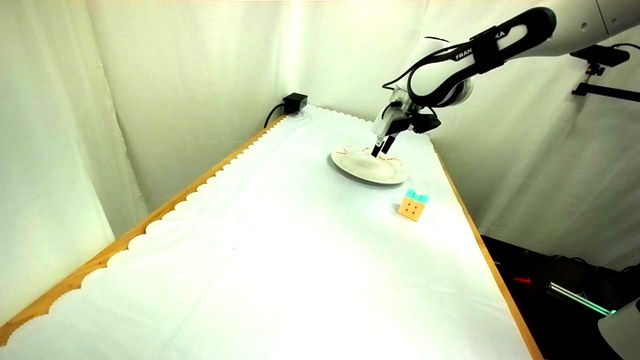} &
        \includegraphics[width=0.135\textwidth]{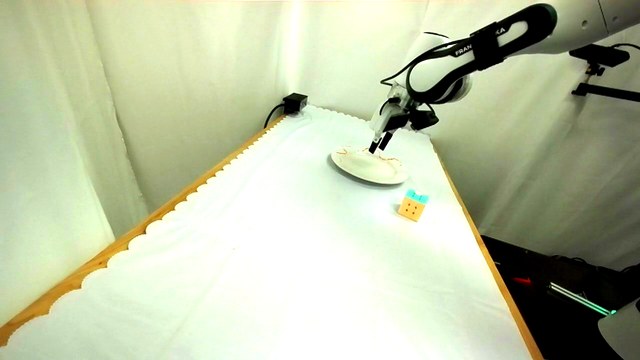} \\
    \textbf{Pion (Ours)} &
        \includegraphics[width=0.135\textwidth]{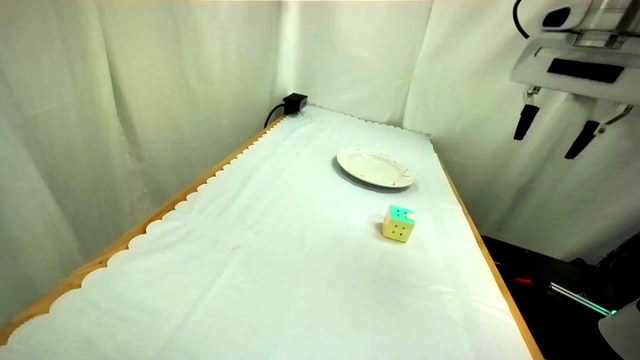} &
        \includegraphics[width=0.135\textwidth]{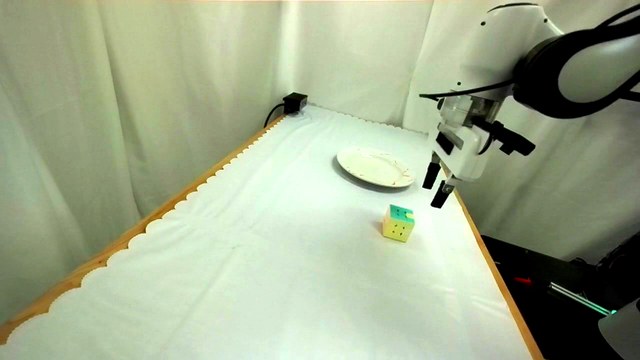} &
        \includegraphics[width=0.135\textwidth]{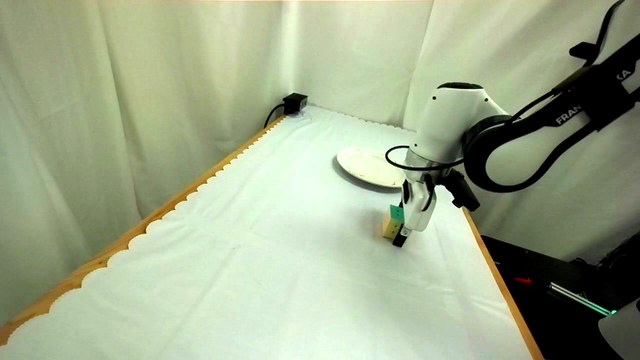} &
        \includegraphics[width=0.135\textwidth]{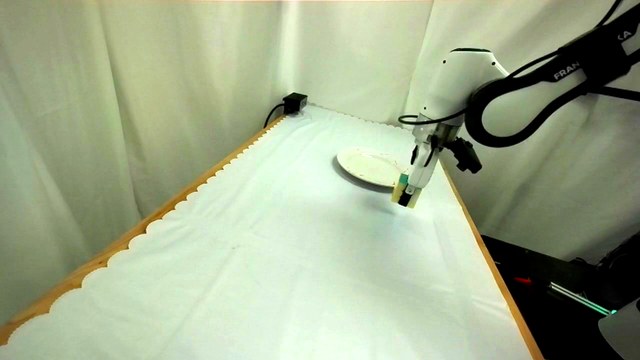} &
        \includegraphics[width=0.135\textwidth]{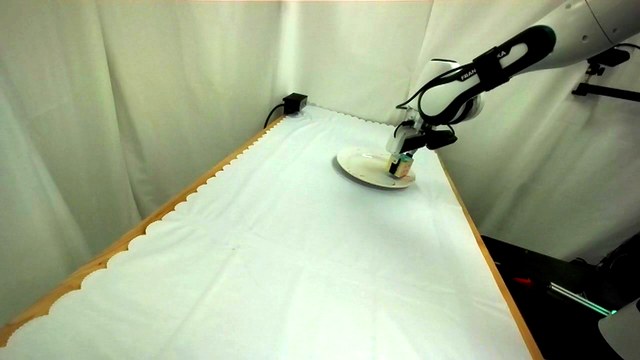} &
        \includegraphics[width=0.135\textwidth]{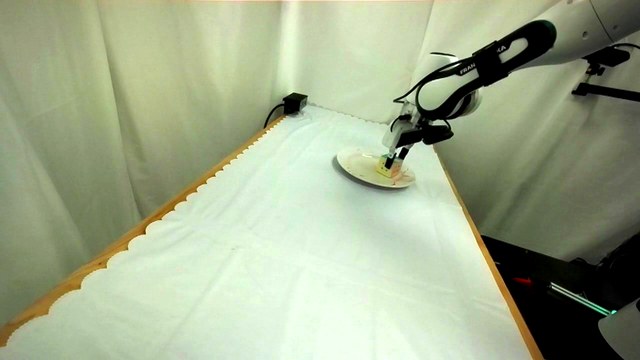} \\
    \midrule
    \multicolumn{7}{@{}>{\columncolor{black!8}}c@{}}{\textbf{Prompt:} \textit{``Pick up the cube and place it in the bowl.''}} \\
    \midrule
    AdamW &
        \includegraphics[width=0.135\textwidth]{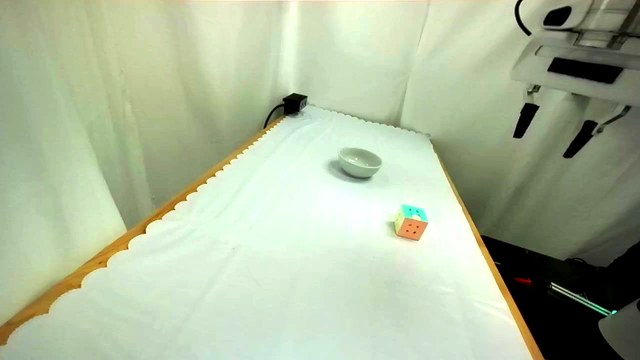} &
        \includegraphics[width=0.135\textwidth]{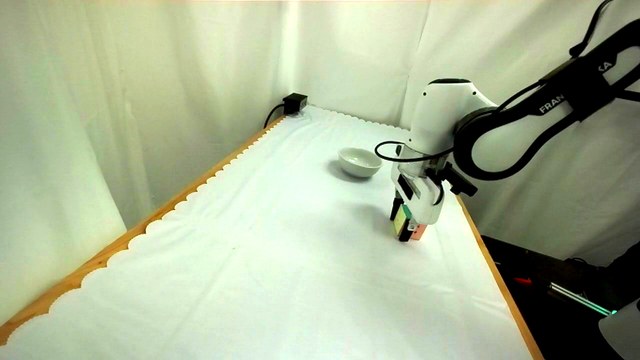} &
        \includegraphics[width=0.135\textwidth]{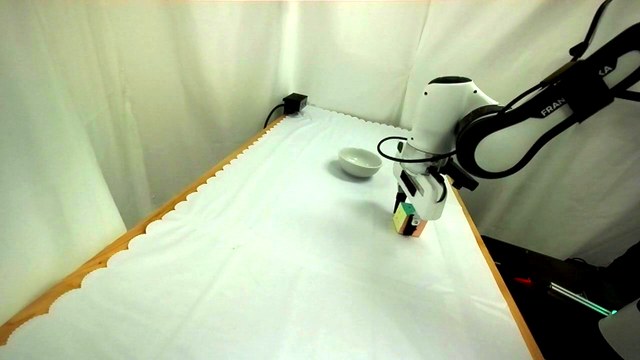} &
        \includegraphics[width=0.135\textwidth]{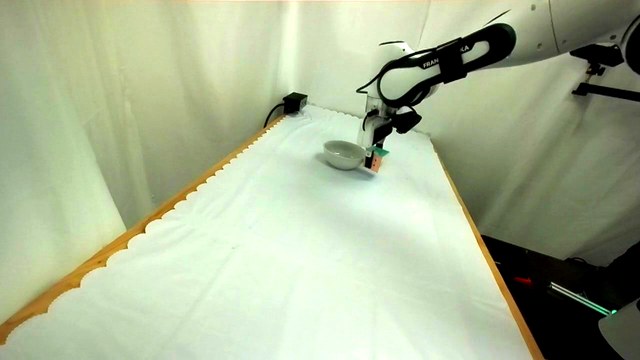} &
        \includegraphics[width=0.135\textwidth]{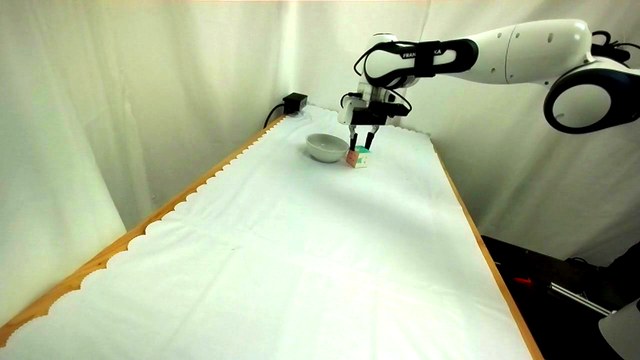} &
        \includegraphics[width=0.135\textwidth]{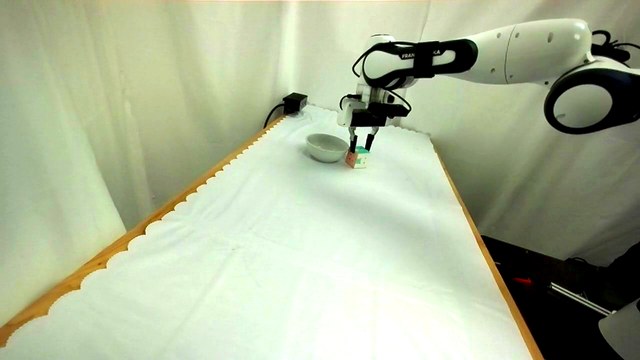} \\
    Muon &
        \includegraphics[width=0.135\textwidth]{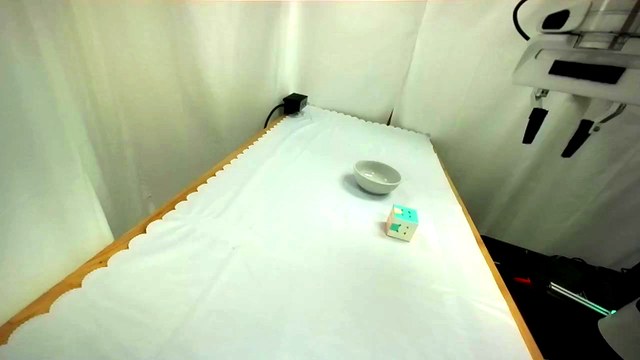} &
        \includegraphics[width=0.135\textwidth]{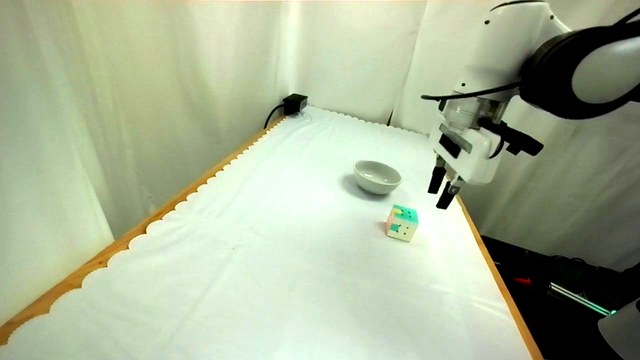} &
        \includegraphics[width=0.135\textwidth]{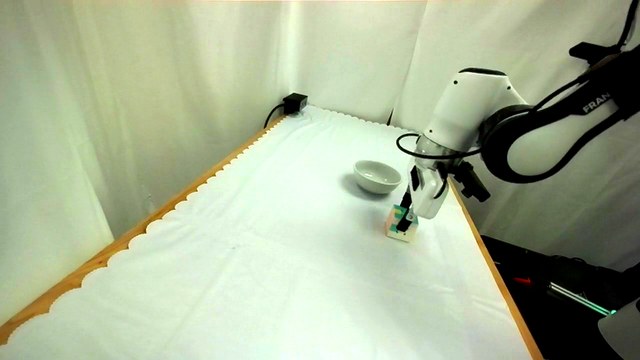} &
        \includegraphics[width=0.135\textwidth]{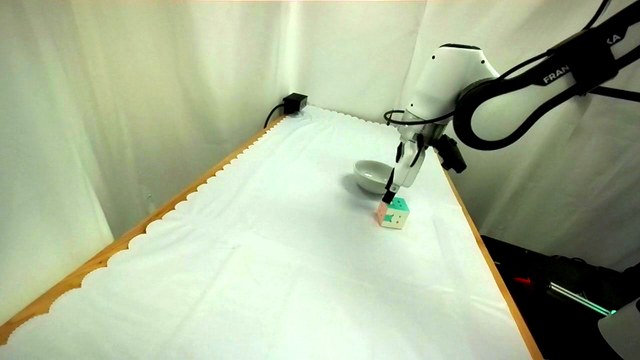} &
        \includegraphics[width=0.135\textwidth]{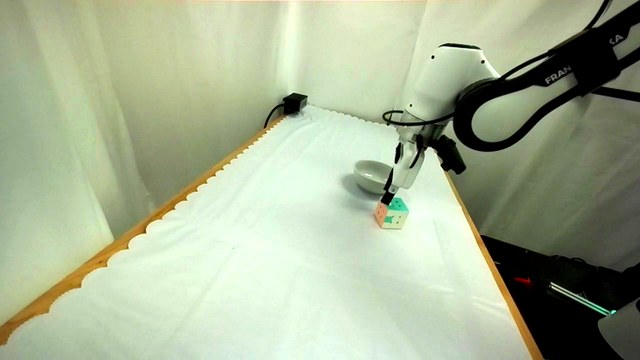} &
        \includegraphics[width=0.135\textwidth]{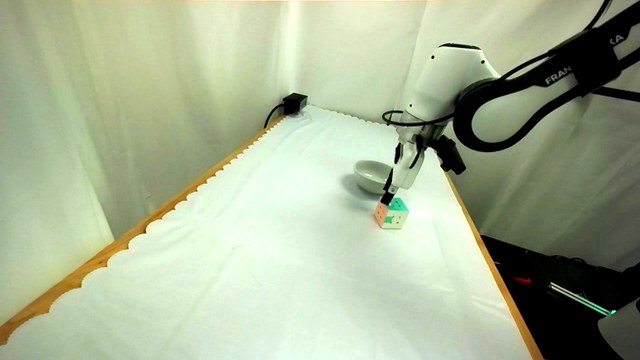} \\
    \textbf{Pion (Ours)} &
        \includegraphics[width=0.135\textwidth]{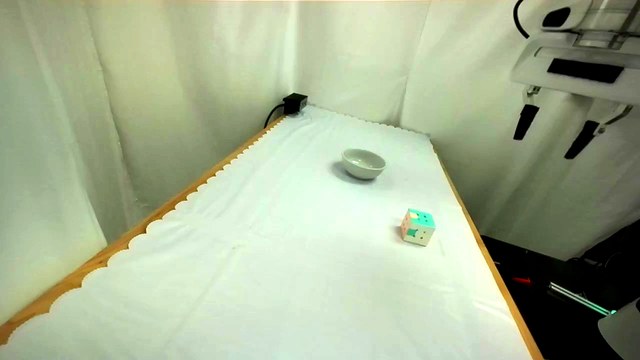} &
        \includegraphics[width=0.135\textwidth]{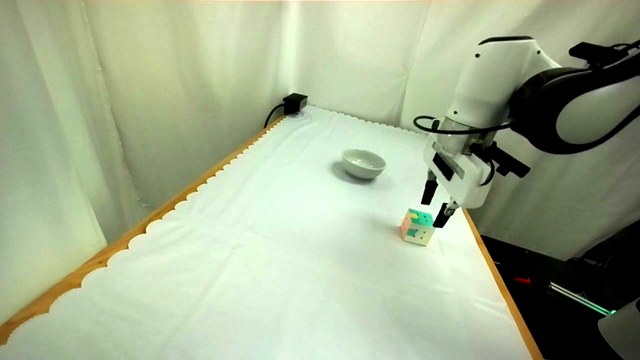} &
        \includegraphics[width=0.135\textwidth]{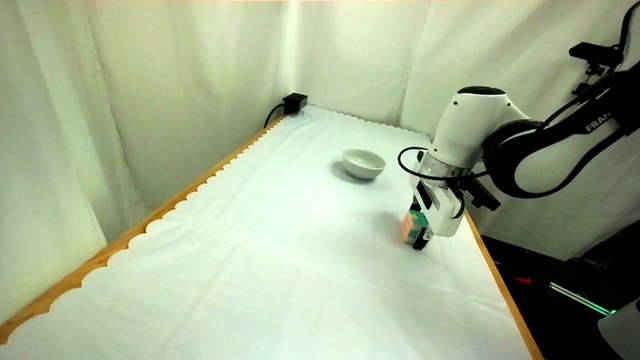} &
        \includegraphics[width=0.135\textwidth]{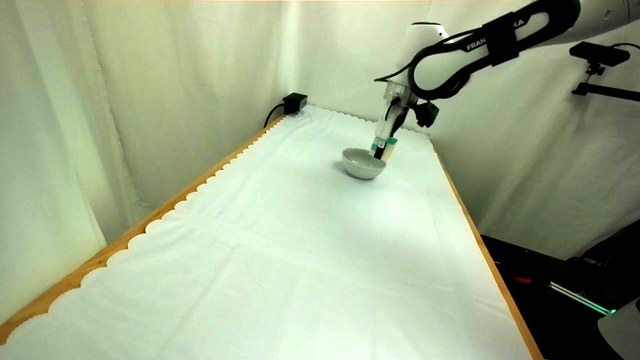} &
        \includegraphics[width=0.135\textwidth]{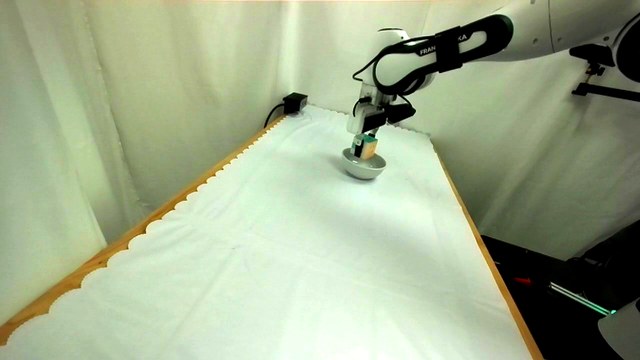} &
        \includegraphics[width=0.135\textwidth]{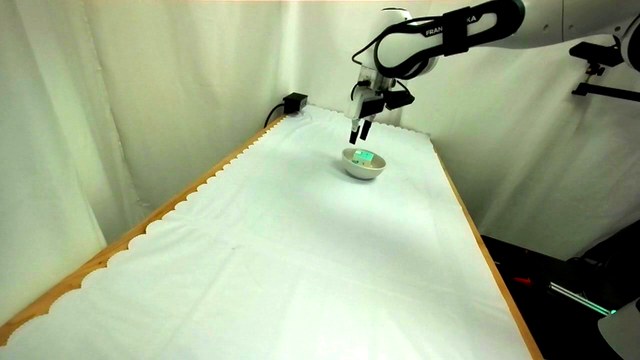} \\
    \bottomrule
    \end{tabular}
\end{table*}

\clearpage
\newpage

\section{Additional VLA Experiments}\label{sec: vla_ablations}

This appendix expands the ablation summary in Sec.\,\ref{sec: exp_vla} with the full setups, figures/tables, and per-row analysis of three studies on VLA-Adapter \citep{wang2026vla}: (i) Pion vs.\ LRMuon for action-module training, (ii) per-head vs.\ default Pion on the action head, and (iii) modality-wise optimizer assignment across the Vision, Language, and Action branches.

\subsection{Pion vs.\ LRMuon for VLA training}\label{sec: vla_ablation_lrmuon}

\noindent \textbf{Pion outperforms LRMuon for VLA training with near-Muon cost.}
\textbf{Fig.\,\ref{fig:vlaadapter_rank_time}} compares Pion with LRMuon (Low-rank Muon) for training VLA-Adapter on LIBERO Object. LRMuon computes an exact SVD of the momentum at each step, retains the top-$k$ singular subspace, and applies the corresponding top-$k$ polar factor $\mathbf{U}_k\mathbf{V}_k^\top$ \citep{he2025low}, as used in Fig.\,\ref{fig:vla_optimizer_compare}. We can observe from \textbf{Fig.\,\ref{fig:vlaadapter_rank_time}-(a)} that LRMuon improves over Muon across all top-$k$ ranks $k \in \{1, 16, 64, 256\}$, confirming the benefit of low-rank spectral filtering, but underperforms Pion at every $k$. This gap arises for two reasons. First, LRMuon uses a \textit{fixed} top-$k$ rank that cannot adapt to the per-step and per-layer rank of the momentum, whereas Pion applies a \textit{soft} spectral filter via high-pass NS. Second, \textbf{Fig.\,\ref{fig:vlaadapter_rank_time}-(b)} shows that the per-step exact SVD computation significantly increases total training time, whereas Pion matches Muon's cost almost exactly.

\begin{figure}[htbp]
    \centering
    \begin{tabular}{cc}
        \includegraphics[width=0.30\textwidth]{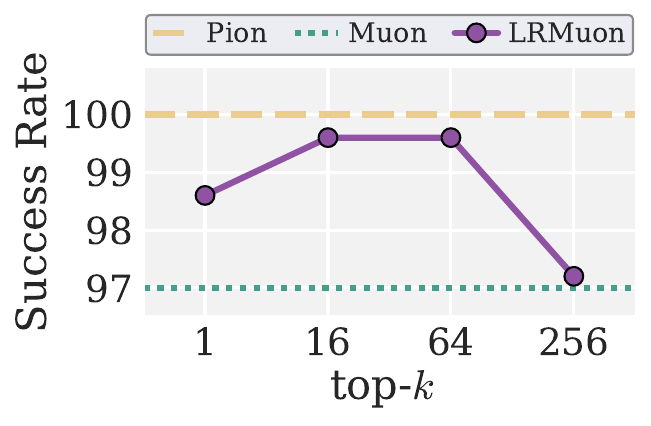} &
        \includegraphics[width=0.3\textwidth]{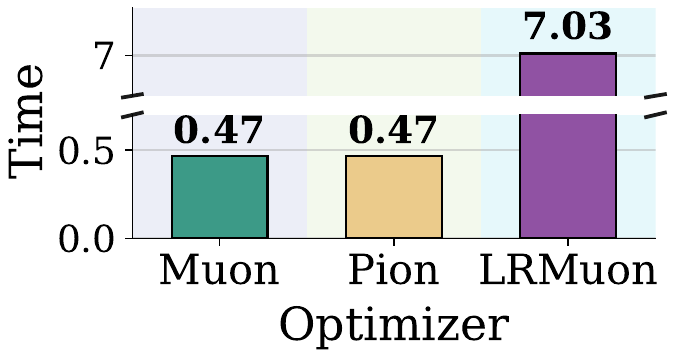} \\
        \small{(a) Success rate vs. top-$k$ rank} & \small{(b) Total training time (hrs)}
    \end{tabular}
    \caption{\small{Muon, Pion and LRMuon for VLA-Adapter on LIBERO Object for $1{,}500$ steps. (a) Test success rate as the top-$k$ rank of LRMuon sweeps $k \in \{1, 16, 64, 256\}$; Pion and Muon are shown as horizontal references. (b) Total training time (hours).}}
    \label{fig:vlaadapter_rank_time}
\end{figure}

\subsection{Per-head vs.\ default Pion on VLA}\label{sec: vla_ablation_perhead}

Sec.\,\ref{sec: method} introduces two application modes of high-pass NS, the default mode and the per-head mode. \textbf{Table\,\ref{tab:vlaadapter_libero}} reports both modes on VLA-Adapter across the four LIBERO task suites. The two modes perform on par, with the default mode marginally ahead on three of four suites (Object $100.0$ vs.\ $99.6$, Spatial $99.4$ vs.\ $98.8$, Long $92.4$ vs.\ $91.6$; only Goal slightly favors per-head, $97.4$ vs.\ $97.2$), yielding a $+0.4$ gap on the four-suite average. This is consistent with the intuition of Sec.\,\ref{sec: method}: unlike the LLM backbone in RLVR, the VLA action head is trained from scratch and carries no per-head heterogeneity for the per-head reshape to preserve, so the default whole-matrix mode already suffices. We therefore use default Pion on the action head throughout Sec.\,\ref{sec: exp_vla}.

\begin{table}[htbp]
    \vspace{-2mm}
    \centering
    \caption{\small AdamW, Muon, and Pion (default vs. per-head) for VLA-Adapter on LIBERO. Test success rates on LIBERO Object, Spatial, Goal, and Long at the same training budget ($1{,}500$ steps for Object and $15{,}000$ steps for others). The best results in each column are in \textbf{bold}.}
    \label{tab:vlaadapter_libero}
    \vspace{1mm}

    \small 
    
    \begin{tabular}{cccccc}
    \toprule
    \textbf{Optimizer} & \textbf{Object} & \textbf{Spatial} & \textbf{Goal} & \textbf{Long} & \textbf{Average} \\
    \midrule
    AdamW             & 32.2  & 97.0 & 89.2 & 69.6 & 72.00 \\
    Muon              & 97.0  & 99.0 & 95.8 & 88.0 & 94.95 \\
    \midrule
    \rowcolor{blue!10}
    \textbf{Pion (per-head)} & 99.6 & 98.8 & \textbf{97.4} & 91.6 & 96.85 \\
    \rowcolor{blue!10}
    \textbf{Pion (default)} & \textbf{100.0} & \textbf{99.4} & 97.2 & \textbf{92.4} & \textbf{97.25} \\
    \bottomrule
    \end{tabular}
    \vspace{-2mm}
\end{table}

\subsection{Modality-wise optimizer assignment on VLA}\label{sec: vla_ablation_modality}

The VLA configuration used throughout Sec.\,\ref{sec: exp_vla}, namely Muon on V/L and Pion on the action head, is one of several plausible assignments. To check whether it is the right one, we sweep the optimizer of each branch independently on VLA-Adapter/LIBERO Object at $1{,}500$ steps, indexing the resulting nine settings as \textbf{S1}--\textbf{S9} in \textbf{Table\,\ref{tab:vla_modality_assignment}}. \textbf{S1} is the all-AdamW reference; \textbf{S2--S3}, \textbf{S4--S5}, \textbf{S6--S7} perturb only the Action, Language, and Vision modules away from \textbf{S1} respectively; and \textbf{S8--S9} contrast the all-Muon configuration with our final ``Muon on V/L + Pion on action'' design. Three observations follow:

\begin{table}[htbp]
    \vspace{-2mm}
    \centering
    \caption{\small Modality-wise optimizer ablation for VLA-Adapter on LIBERO Object. Test success rates at $1{,}500$ training steps. AdamW, Muon, and Pion are ablated across Vision, Language, and Action modules. The best result is in \textbf{bold}.}
    \label{tab:vla_modality_assignment}
    \vspace{1mm}

    \small
    \begin{tabular}{ccccc}
    \toprule
    \multirow{2}{*}{\textbf{Setting}} & \multicolumn{3}{c}{\textbf{Optimizer}} & \multirow{2}{*}{\textbf{Success Rate (\%)}} \\
    \cmidrule(lr){2-4}
    & \textbf{Vision} & \textbf{Language} & \textbf{Action} & \\
    \midrule
    \textbf{S1} & AdamW & AdamW & AdamW & 43.6 \\
    \midrule
    \textbf{S2} & AdamW & AdamW & Muon & 40.0 \\
    \textbf{S3} & AdamW & AdamW & Pion & 73.6 \\
    \midrule
    \textbf{S4} & AdamW & Muon & AdamW & 94.6 \\
    \textbf{S5} & AdamW & Pion & AdamW & 73.8 \\
    \midrule
    \textbf{S6} & Muon & AdamW & AdamW & 96.8 \\
    \textbf{S7} & Pion & AdamW & AdamW & 17.8 \\
    \midrule
    \textbf{S8} & Muon & Muon & Muon & 97.0 \\
    \rowcolor{blue!10}
    \textbf{S9} & Muon & Muon & Pion & \textbf{100.0} \\
    \bottomrule
    \end{tabular}
    \vspace{-2mm}
\end{table}

(i) \textit{Action head wants Pion, not Muon.} With V/L fixed at AdamW, switching the action head from AdamW (\textbf{S1}, 43.6) to Muon (\textbf{S2}) \textit{drops} accuracy to 40.0, while switching it to Pion (\textbf{S3}) \textit{lifts} it to 73.6. This confirms the spectral diagnosis of Sec.\,\ref{sec: motivation}: the low-erank action gradient is mismatched with Muon's uniform whitening, but well-suited to Pion's high-pass.

(ii) \textit{Vision and Language want Muon, not Pion.} Symmetrically, with the other two branches fixed at AdamW, switching Language to Muon (\textbf{S4}) improves accuracy from \textbf{S1} (43.6) to 94.6, while switching it to Pion (\textbf{S5}) only reaches 73.8; switching Vision to Muon (\textbf{S6}) improves accuracy to 96.8, while switching it to Pion (\textbf{S7}) collapses to 17.8. The high-rank V/L modules thus genuinely benefit from Muon's uniform spectral updates, and applying a high-pass there discards informative tail components.

(iii) \textit{The chosen assignment is optimal.} Combining the two findings, ``Muon on V/L + Pion on action'' (\textbf{S9}) reaches $100.0\%$ success, strictly above all-Muon (\textbf{S8}, 97.0\%) and any single-module configuration in \textbf{S2--S7}. \textbf{S9} is therefore not an arbitrary engineering choice but the assignment that respects the spectral structure of each modality.

\clearpage
\newpage
\section{Low-pass Muon (LPMuon): Coefficient Design via Constrained Polynomial Fitting}
\label{sec:lowpass_muon}

This appendix details the coefficient design of \textbf{Low-pass Muon} (LPMuon), the reverse-ablation baseline of Sec.\,\ref{sec: exp_rl} (\textbf{Fig.\,\ref{fig:lowpassmuon}}). Unlike Pion, whose Promotion and Suppression polynomials admit closed-form solutions from analytic constraints at $\sigma\!=\!0$ and $\sigma\!=\!1$ (Sec.\,\ref{sec: method}), the LPMuon target profile is a sharp band indicator whose quality depends on the \textit{whole composition} across $\sigma$, and the $t\!=\!5$ steps exchange degrees of freedom (e.g., scaling by $p_1$ can be partially absorbed into $p_2$). We therefore treat all $15$ coefficients as free variables and fit them numerically via a multi-start L-BFGS-B procedure.

\paragraph{Target filter and matrix-level update.}
LPMuon composes $t=5$ odd quintic polynomials $p_k(\sigma) = a_{1,k}\sigma + a_{3,k}\sigma^3 + a_{5,k}\sigma^5$:
\begin{equation}
    \tilde f_{\boldsymbol{\theta}}(\sigma) \Def (p_5 \circ p_4 \circ p_3 \circ p_2 \circ p_1)(\sigma),
    \quad
    \boldsymbol{\theta} = \{(a_{1,k}, a_{3,k}, a_{5,k})\}_{k=1}^{5} \in \mathbb{R}^{15},
    \label{eq: lowpass_composition}
\end{equation}
to approximate an odd extension of the band indicator on $\sigma \in [-1, 1]$. The actual normalized singular values of the pre-normalized momentum are nonnegative and lie in $[0,1]$; the negative half-axis is included only to define and visualize the odd scalar extension:
\begin{equation}
    \tilde f_{\boldsymbol{\theta}}(\sigma) \approx \mathrm{sign}(\sigma) \cdot \mathds{1}\bigl[|\sigma| \leq \tau\bigr],
    \quad \tau \in (0, 1) \text{ is the cutoff.}
    \label{eq: lowpass_target}
\end{equation}
Since each $p_k$ is odd, $\tilde f_{\boldsymbol{\theta}}$ is automatically antisymmetric and we only need to fit on $\sigma \geq 0$. By the SVD factorization \eqref{eq: svd_factorization}, applying $p_k$ at the matrix level on $\mathbf{M}_t \in \mathbb{R}^{m \times n}$,
\begin{equation}
    \mathbf{M}_t \leftarrow a_{1,k}\,\mathbf{M}_t + a_{3,k}\,(\mathbf{M}_t\mathbf{M}_t^\top)\mathbf{M}_t + a_{5,k}\,(\mathbf{M}_t\mathbf{M}_t^\top)^2 \mathbf{M}_t,
    \label{eq: lowpass_matrix_step}
\end{equation}
is equivalent to applying $\tilde f_{\boldsymbol{\theta}}$ entry-wise to every singular value of $\mathbf{M}_t$, so LPMuon preserves Muon's per-step $5$-matmul cost and requires no explicit SVD.

\paragraph{Discretized fitting objective.}
Given a cutoff $\tau$, we discretize the positive half-axis into a pass band $\mathcal{S}_{\mathrm{p}}^{+} \subset [0.01,\, \tau\!-\!\Delta]$ and a stop band $\mathcal{S}_{\mathrm{s}}^{+} \subset [\tau\!+\!\Delta,\, 1]$ separated by a transition half-width $\Delta = 0.03$ (up to $250$ samples per band per side, reduced to $50$ when $\tau$ is close to $1$); these are mirrored to the negative half-axis for notational symmetry in the odd-extension loss, forming $\mathcal{S}_{\mathrm{p}}, \mathcal{S}_{\mathrm{s}}$. Intermediate iterates are clipped to $[-10^3, 10^3]$ to avoid early-iteration overflow. The fitting loss combines a pass-band, stop-band, overshoot, and non-negativity term:
\begin{align}
    \mathcal{L}(\boldsymbol{\theta};\, \tau)
    &= \lambda_{\mathrm{p}} \underbrace{\frac{1}{|\mathcal{S}_{\mathrm{p}}|}\!\sum_{\sigma \in \mathcal{S}_{\mathrm{p}}}\! \bigl(\tilde f_{\boldsymbol{\theta}}(\sigma) - \mathrm{sign}(\sigma)\bigr)^2}_{\mathcal{L}_{\mathrm{pass}}}
    + \lambda_{\mathrm{s}} \underbrace{\frac{1}{|\mathcal{S}_{\mathrm{s}}|}\!\sum_{\sigma \in \mathcal{S}_{\mathrm{s}}}\! \tilde f_{\boldsymbol{\theta}}(\sigma)^2}_{\mathcal{L}_{\mathrm{stop}}} \nonumber\\
    &\quad + \lambda_{\mathrm{o}} \underbrace{\frac{1}{|\mathcal{S}_{\mathrm{p}}\!\cup\!\mathcal{S}_{\mathrm{s}}|}\!\sum_{\sigma}\! \bigl(\max(|\tilde f_{\boldsymbol{\theta}}(\sigma)|\!-\!1.02,\, 0)\bigr)^2}_{\mathcal{L}_{\mathrm{over}}}
    + \lambda_{\mathrm{nn}} \underbrace{\sum_{\mathcal{S} \in \{\mathcal{S}_{\mathrm{p}}^{+}, \mathcal{S}_{\mathrm{s}}^{+}\}} \frac{1}{|\mathcal{S}|}\!\sum_{\sigma \in \mathcal{S}}\! \bigl(\max(-\tilde f_{\boldsymbol{\theta}}(\sigma),\, 0)\bigr)^2}_{\mathcal{L}_{\mathrm{nn}}},
    \label{eq: lowpass_total_loss}
\end{align}
with weights $(\lambda_{\mathrm{p}}, \lambda_{\mathrm{s}}, \lambda_{\mathrm{o}}, \lambda_{\mathrm{nn}}) = (3,\, 8,\, 30,\, 30)$. Here $\mathcal{L}_{\mathrm{pass}}$ anchors the pass band at $\pm 1$; $\mathcal{L}_{\mathrm{stop}}$ drives the stop band to $0$; $\mathcal{L}_{\mathrm{over}}$ keeps intermediate iterates bounded so the $5$-step composition does not blow up; $\mathcal{L}_{\mathrm{nn}}$ enforces non-negativity on $\sigma > 0$ (without it the fit admits sign-flipping solutions that would invert the gradient direction for a subset of singular components). The stop-band and overshoot terms are weighted more heavily because residual energy or overshoot compounds multiplicatively across the $5$ compositions.

\paragraph{Warm-start initialization, multi-start solver, and aggregation.}
To escape the many spurious local minima of quintic compositions, we use a structured warm start with random restarts. The first polynomial is initialized as identity ($p_1^{(0)} = (1, 0, 0)$) and the remaining four are initialized to Pion's Promotion coefficients \eqref{eq: pion_promotion}, $p_{2:5}^{(0)} = (1.875, -1.25, 0.375)$. Trial $m=1$ uses $\boldsymbol{\theta}^{(0)}$ directly; trials $m=2, \ldots, M$ ($M=8$) use $\boldsymbol{\theta}^{(0)} + \boldsymbol{\varepsilon}^{(m)}$ with $\boldsymbol{\varepsilon}^{(m)} \sim \mathcal{N}(\mathbf{0},\, 0.25^2 \mathbf{I}_{15})$. Each trial is solved by \texttt{scipy.optimize.minimize} with L-BFGS-B (maximum $2{,}000$ iterations, $f_{\mathrm{tol}} = 10^{-12}$, $g_{\mathrm{tol}} = 10^{-9}$, finite-difference gradients); divergent restarts are discarded. The final solution is $\hat{\boldsymbol{\theta}}(\tau) = \boldsymbol{\theta}_{\infty}^{(m^\star)}$ with $m^\star = \argmin_m \mathcal{L}(\boldsymbol{\theta}_{\infty}^{(m)};\, \tau)$. We sweep $\tau \in \{0.1, 0.2, \ldots, 0.9\}$, and use $\tau = 0.5$ inside the RLVR optimization loop in Fig.\,\ref{fig:lowpassmuon}. The full procedure is summarized in Alg.\,\ref{alg:lowpass_fitting}.

\begin{algorithm}[H]
    \caption{LPMuon Coefficient Fitting (L-BFGS-B)}
    \label{alg:lowpass_fitting}
    \begin{algorithmic}[1]
    \REQUIRE Cutoff $\tau$, transition half-width $\Delta$, steps $t=5$, restarts $M$, weights $(\lambda_{\mathrm{p}}, \lambda_{\mathrm{s}}, \lambda_{\mathrm{o}}, \lambda_{\mathrm{nn}})$
    \STATE Build discretizations $\mathcal{S}_{\mathrm{p}} \subset [-(\tau\!-\!\Delta),\, -0.01] \cup [0.01,\, \tau\!-\!\Delta]$, $\mathcal{S}_{\mathrm{s}} \subset [-1,\, -(\tau\!+\!\Delta)] \cup [\tau\!+\!\Delta,\, 1]$
    \STATE Warm start $\boldsymbol{\theta}^{(0)}$: $p_1 = (1, 0, 0)$ and $p_{2:t} = (1.875, -1.25, 0.375)$ \COMMENT{Pion Promotion}
    \FOR{$m = 1, \ldots, M$}
        \STATE $\boldsymbol{\theta}_{\mathrm{init}}^{(m)} \leftarrow \boldsymbol{\theta}^{(0)}$ if $m=1$, else $\boldsymbol{\theta}^{(0)} + \boldsymbol{\varepsilon}^{(m)}$ with $\boldsymbol{\varepsilon}^{(m)} \sim \mathcal{N}(\mathbf{0}, 0.25^2\, \mathbf{I}_{15})$
        \STATE $\boldsymbol{\theta}_{\infty}^{(m)} \leftarrow \mathrm{LBFGSB}\bigl(\mathcal{L}(\,\cdot\,;\, \tau),\, \boldsymbol{\theta}_{\mathrm{init}}^{(m)}\bigr)$;\quad $\mathcal{L}^{(m)} \leftarrow \mathcal{L}(\boldsymbol{\theta}_{\infty}^{(m)};\, \tau)$
    \ENDFOR
    \RETURN $\hat{\boldsymbol{\theta}}(\tau) = \boldsymbol{\theta}_{\infty}^{(m^\star)}$ with $m^\star = \argmin_m \mathcal{L}^{(m)}$, and the corresponding scalar filter $\tilde f_{\hat{\boldsymbol{\theta}}(\tau)}$ \eqref{eq: lowpass_composition}
    \end{algorithmic}
\end{algorithm}

\paragraph{Resulting filters and reverse-ablation evidence.}
\textbf{Fig.\,\ref{fig:lowpass_curves}} visualizes $\tilde f_{\hat{\boldsymbol{\theta}}(\tau)}$ across the full sweep $\tau \in \{0.1, 0.2, \ldots, 0.9\}$: as $\tau$ grows, the transition shifts rightward while the pass band ($|\sigma| \leq \tau$) stays anchored at $\pm 1$ and the stop band ($|\sigma| > \tau$) is driven to $0$, confirming that Alg.\,\ref{alg:lowpass_fitting} consistently recovers the desired low-pass profile. Numerical coefficients are listed in \textbf{Table\,\ref{tab:lowpass_coeffs}} and can be plugged directly into \eqref{eq: lowpass_matrix_step}. At the matrix level, this gives the LPMuon optimizer used in Fig.\,\ref{fig:lowpassmuon}-(b); its flat accuracy curve (LPMuon fails to train at all) provides the reverse-ablation evidence of Sec.\,\ref{sec: exp_rl}: retaining the small singular values while discarding the large ones destroys the learning signal, isolating that Pion's gains arise from \textit{high-pass} filtering rather than the iteration form, per-head reshape, or generic spectral transformation.

\begin{figure}[H]
    \centering
    \begin{tabular}{ccc}
        \includegraphics[width=0.225\textwidth]{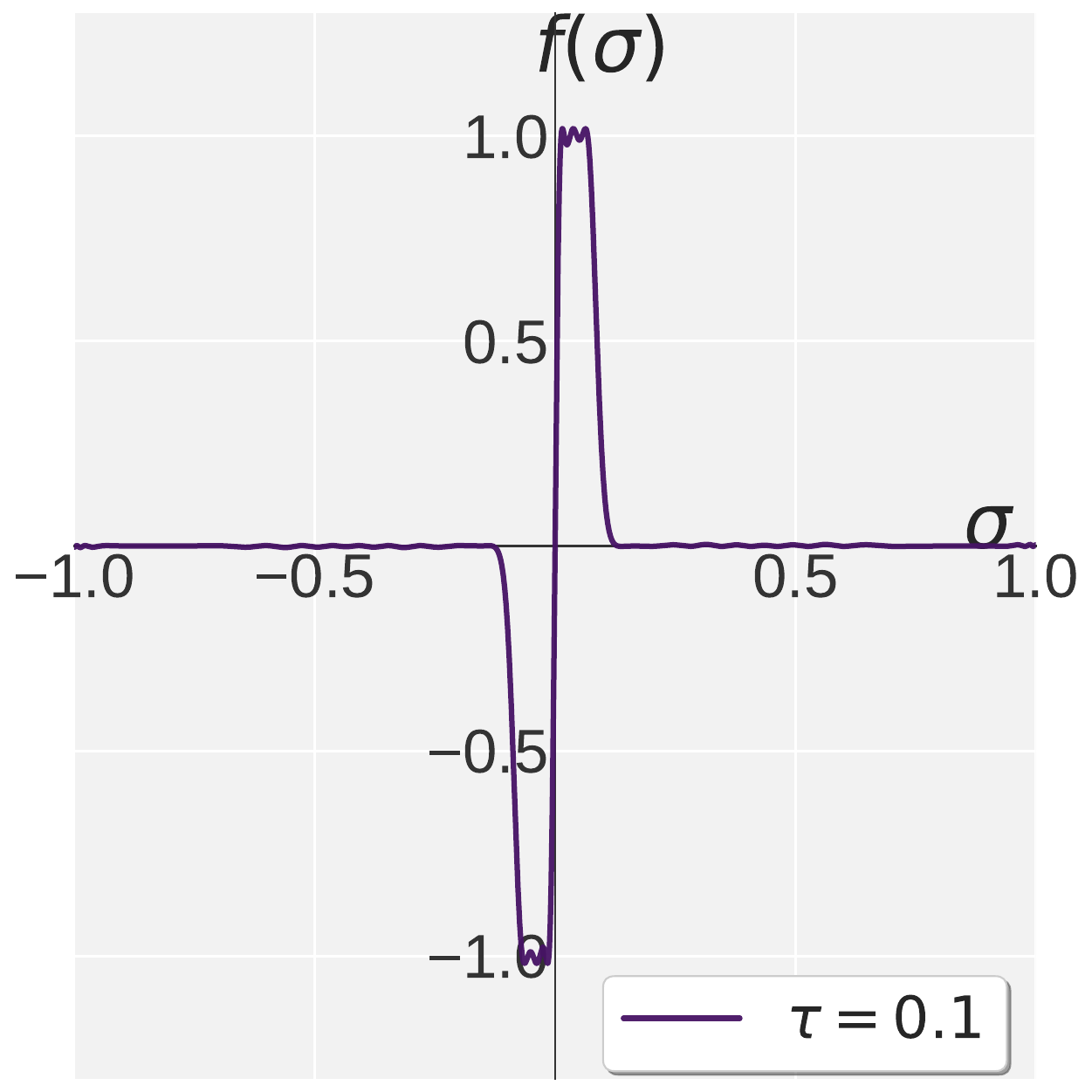} &
        \includegraphics[width=0.225\textwidth]{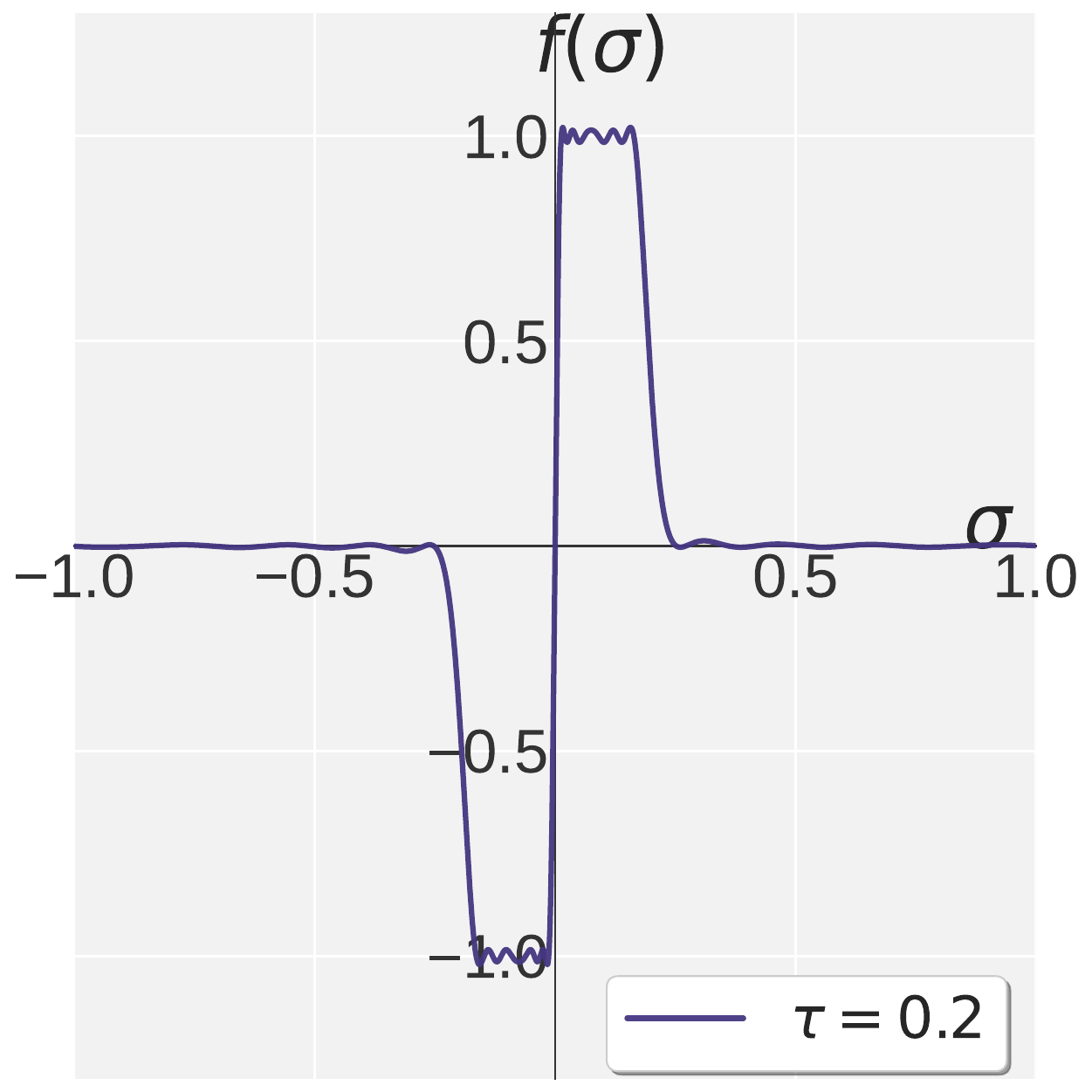} &
        \includegraphics[width=0.225\textwidth]{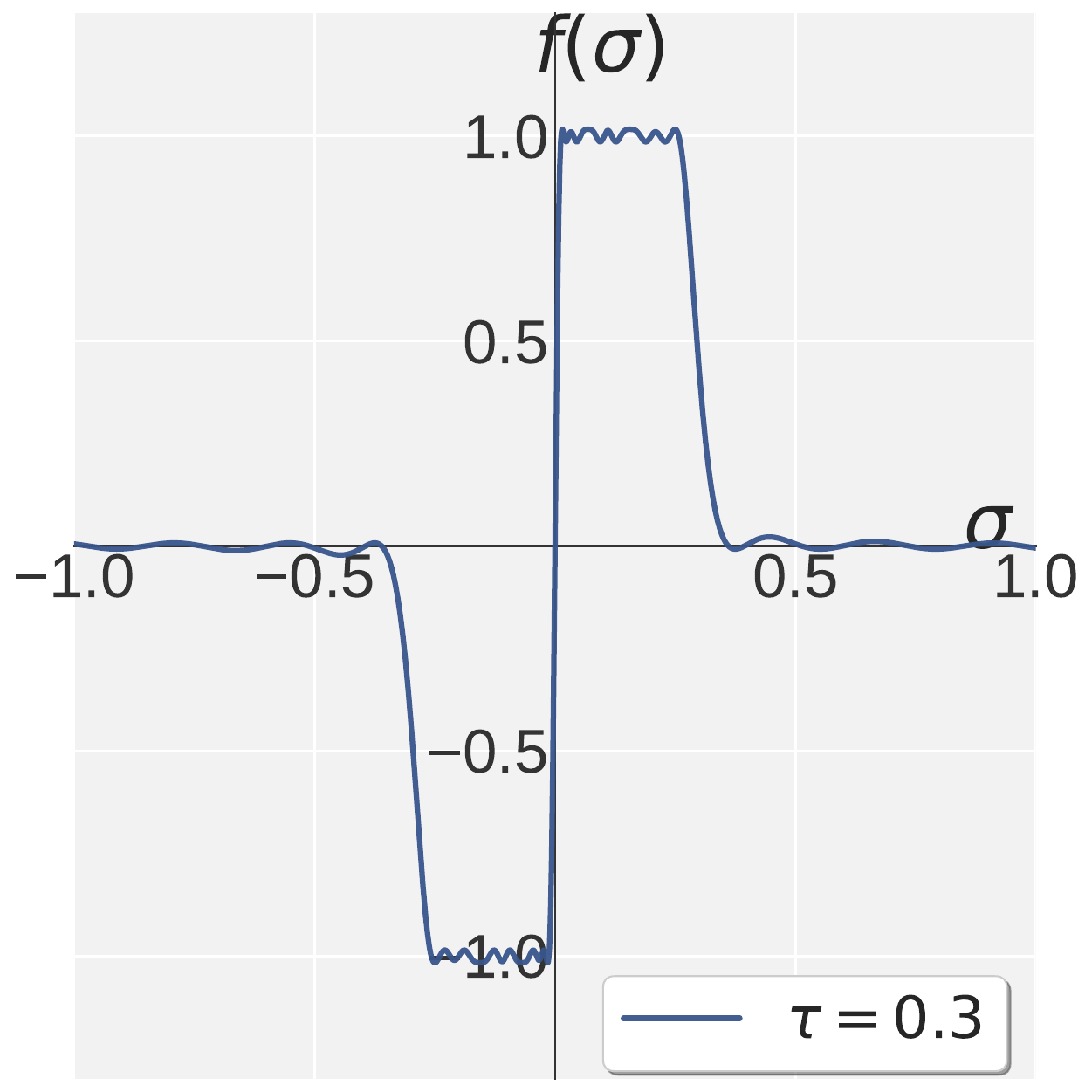} \\
        \small{(a) $\tau = 0.1$} & \small{(b) $\tau = 0.2$} & \small{(c) $\tau = 0.3$} \\[2pt]
        \includegraphics[width=0.225\textwidth]{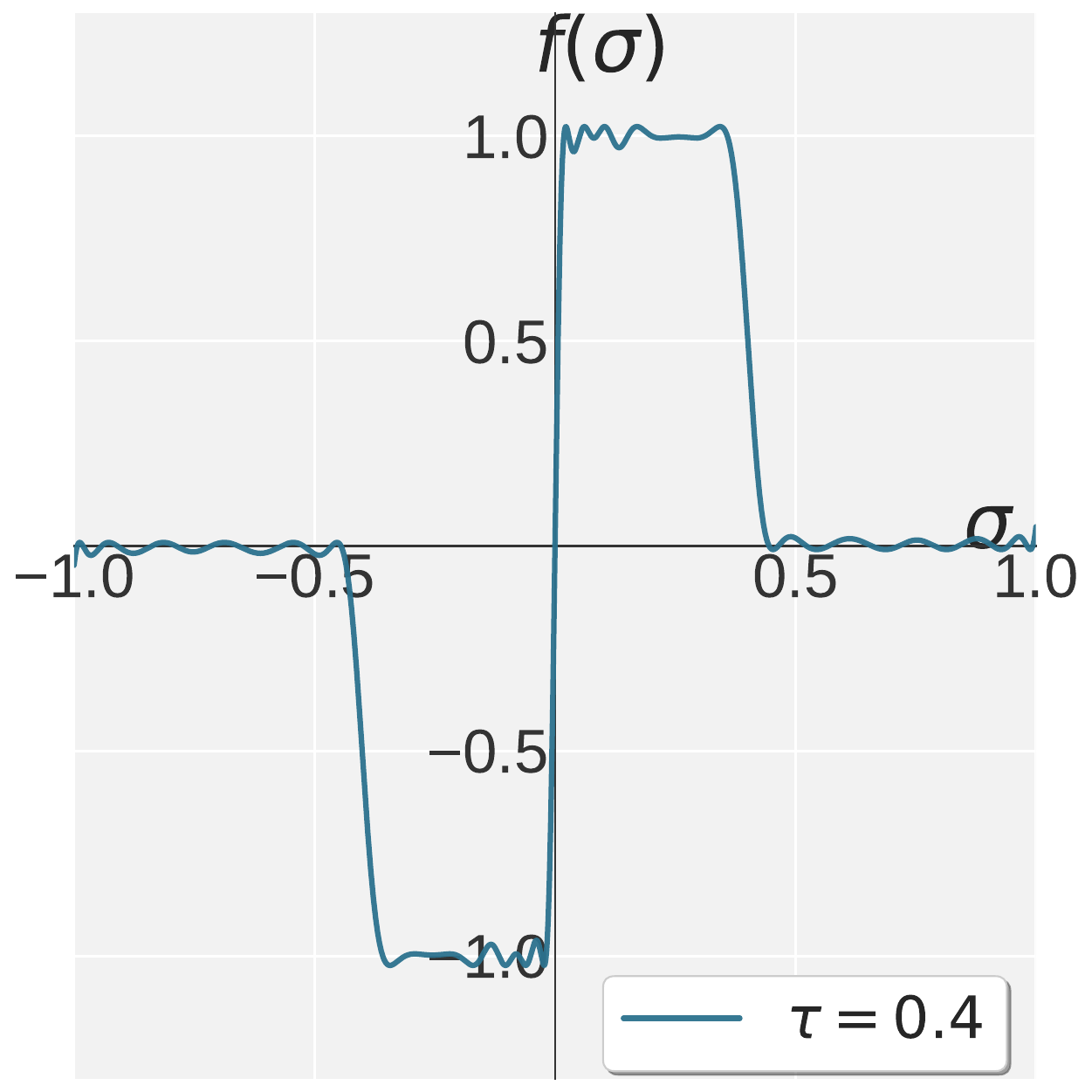} &
        \includegraphics[width=0.225\textwidth]{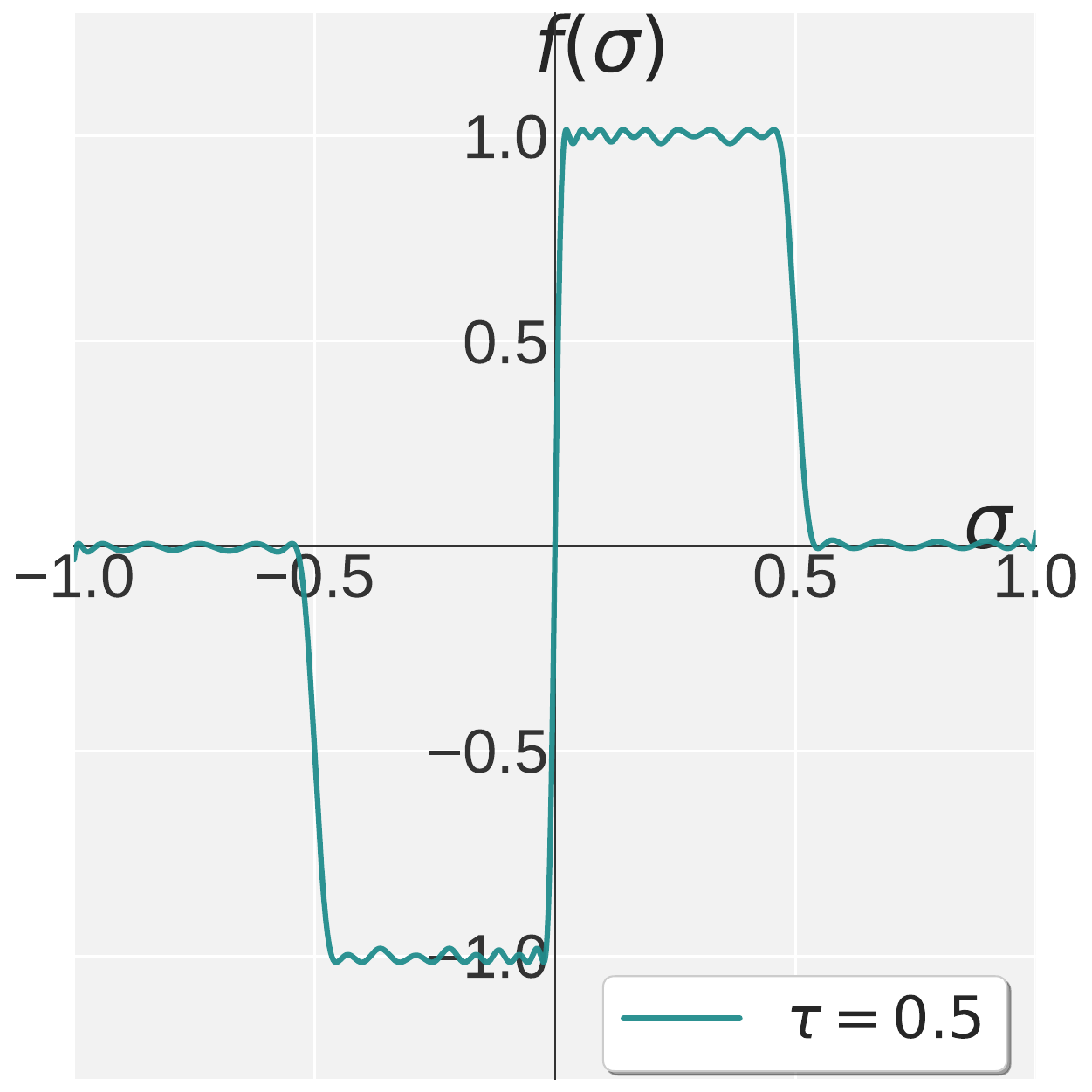} &
        \includegraphics[width=0.225\textwidth]{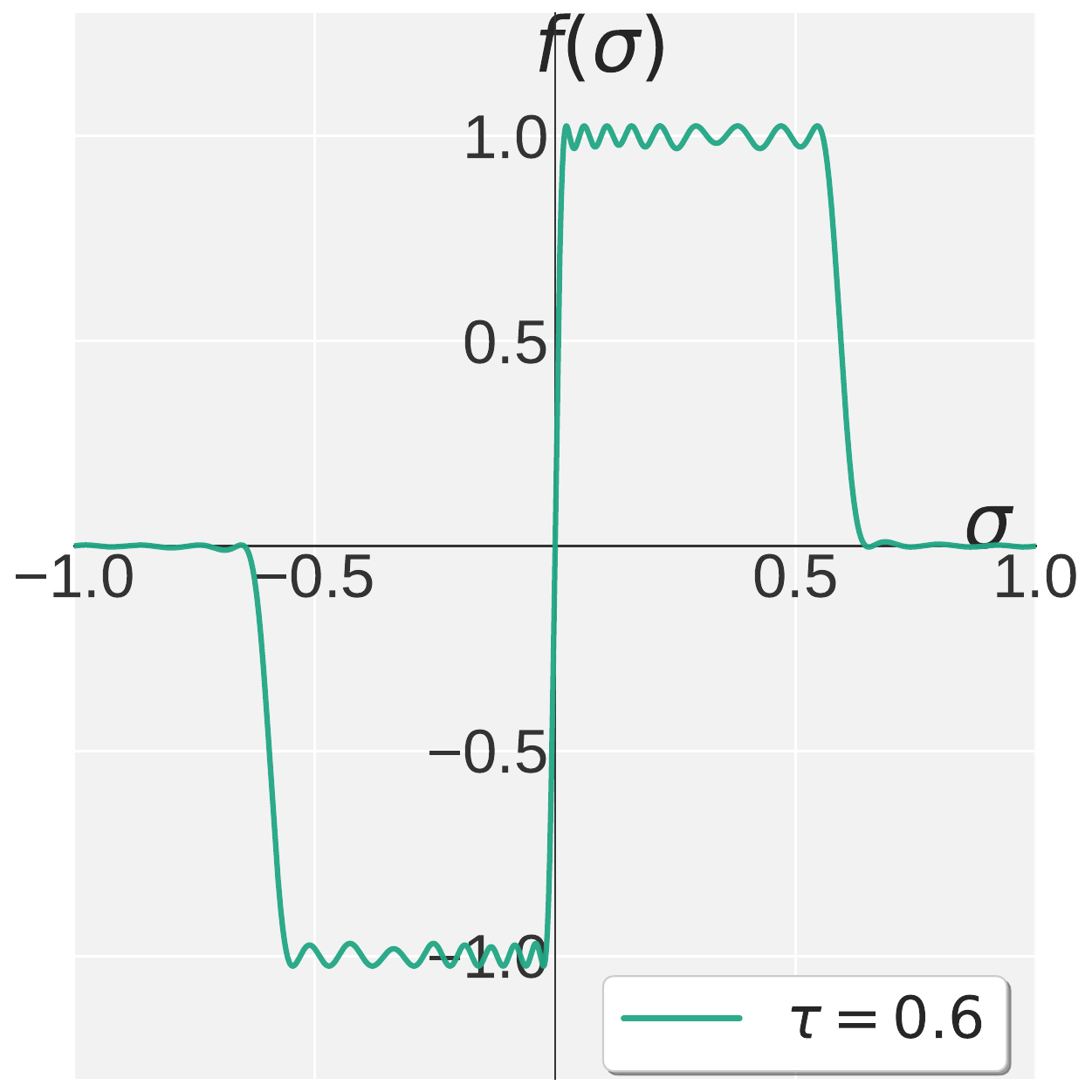} \\
        \small{(d) $\tau = 0.4$} & \small{(e) $\tau = 0.5$} & \small{(f) $\tau = 0.6$} \\[2pt]
        \includegraphics[width=0.225\textwidth]{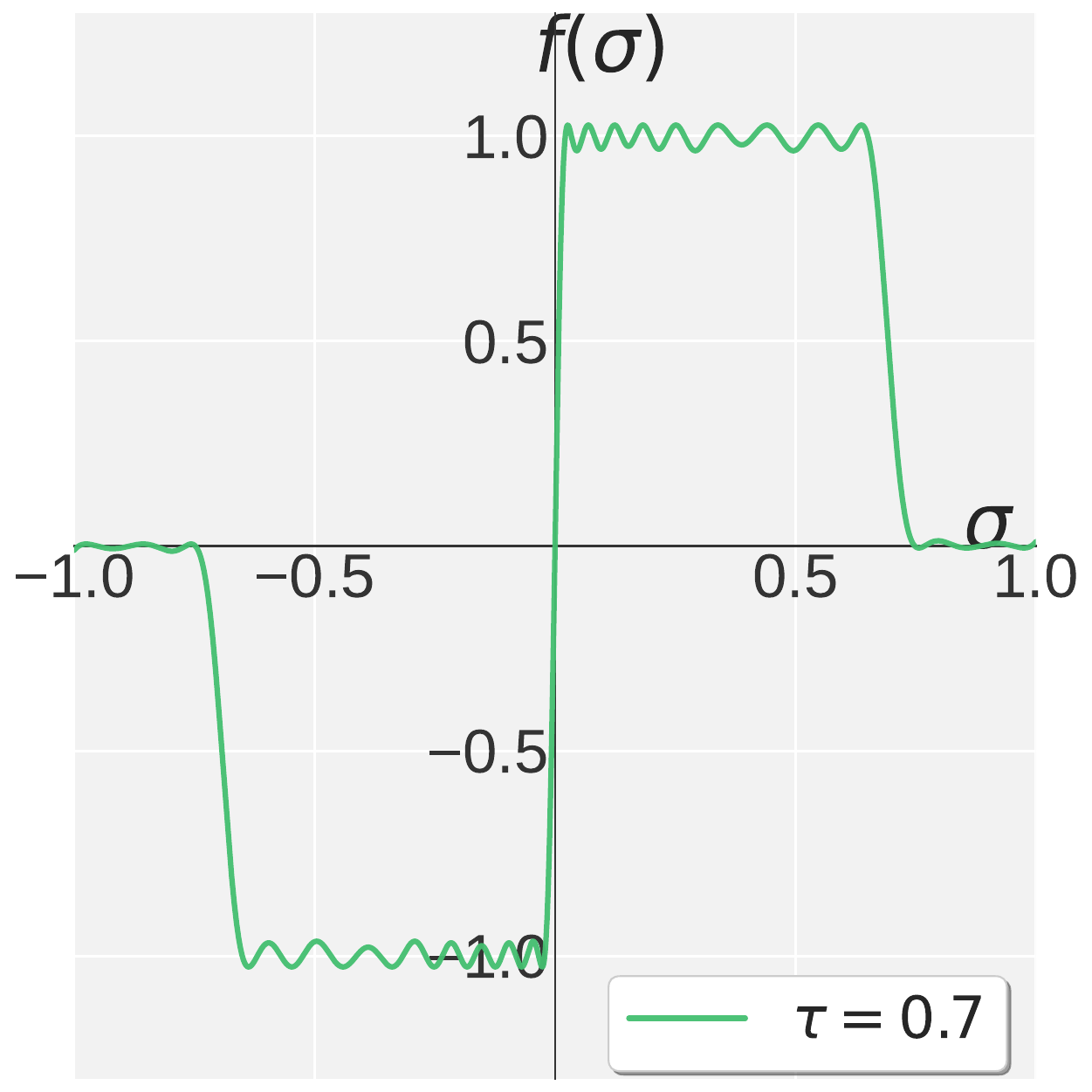} &
        \includegraphics[width=0.225\textwidth]{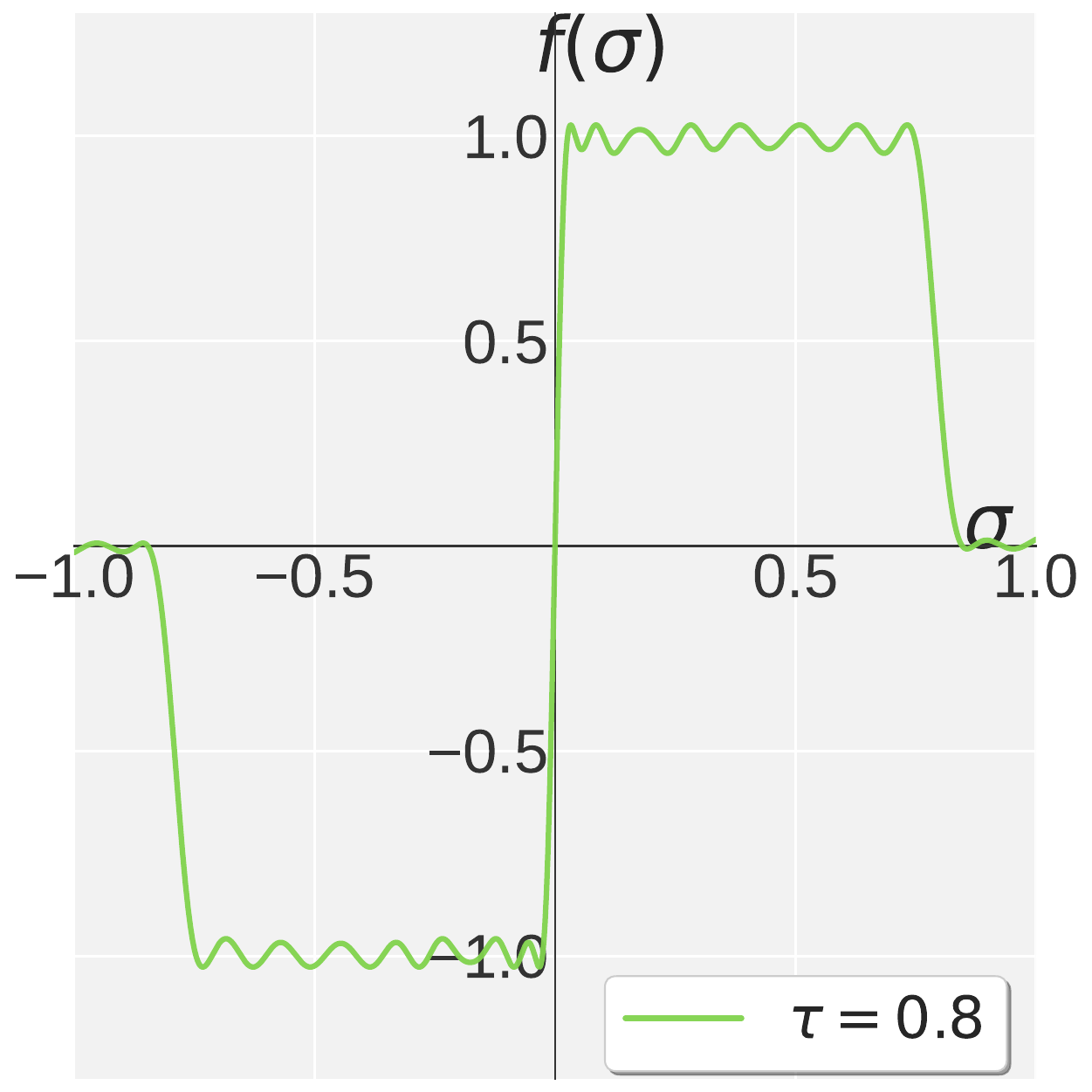} &
        \includegraphics[width=0.225\textwidth]{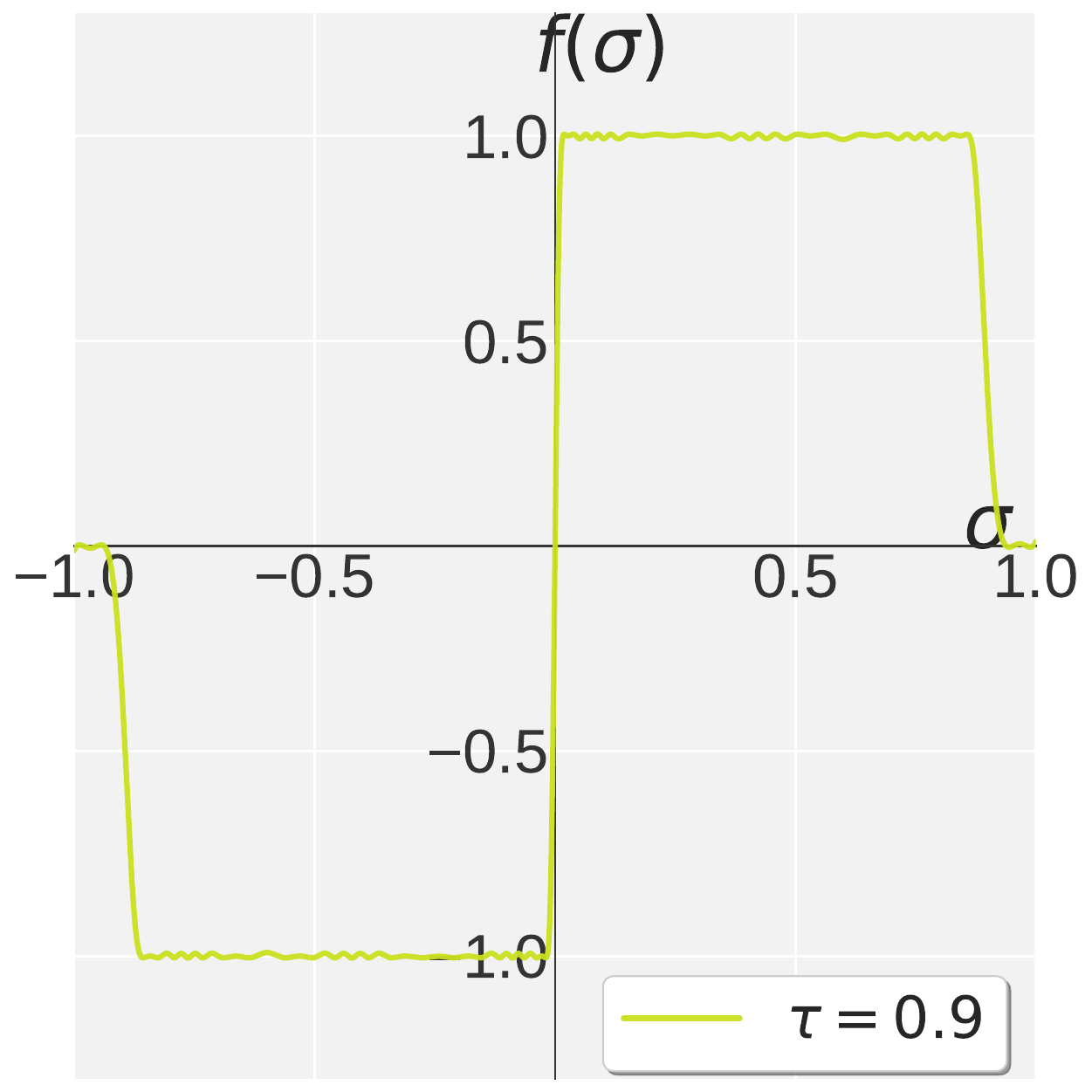} \\
        \small{(g) $\tau = 0.7$} & \small{(h) $\tau = 0.8$} & \small{(i) $\tau = 0.9$} \\
    \end{tabular}
    \vspace{-1mm}
    \caption{\small{Odd scalar extension $\sigma_{\mathrm{in}} \mapsto \sigma_{\mathrm{out}}$ fitted for LPMuon at nine thresholds $\tau \in \{0.1, 0.2, \ldots, 0.9\}$ obtained from Alg.\,\ref{alg:lowpass_fitting}. In the actual SVD update only the nonnegative half $\sigma_{\mathrm{in}}\in[0,1]$ is applied to singular values; the plotted negative half visualizes the antisymmetric extension. Each panel anchors the pass band ($|\sigma| \leq \tau$) at $\pm 1$ and contracts the stop band ($|\sigma| > \tau$) toward $0$.}}
    \label{fig:lowpass_curves}
\end{figure}

\begin{table}[H]
    \centering
    \caption{\small{Fitted coefficients $\hat{\boldsymbol{\theta}}(\tau) = \{(a_{1,k}, a_{3,k}, a_{5,k})\}_{k=1}^{5}$ of the $5$-step odd-quintic composition \eqref{eq: lowpass_composition} solved by Alg.\,\ref{alg:lowpass_fitting} for the cutoff sweep $\tau \in \{0.1, \ldots, 0.9\}$. Each row corresponds to one $\tau$; the $15$ entries are read step by step ($k=1, \ldots, 5$). At the matrix level, step $k$ applies \eqref{eq: lowpass_matrix_step}. The last column reports the converged loss $\mathcal{L}(\hat{\boldsymbol{\theta}}(\tau);\,\tau)$ of \eqref{eq: lowpass_total_loss}.}}
    \label{tab:lowpass_coeffs}
    \vspace{1mm}
    \resizebox{\textwidth}{!}{%
    \setlength{\tabcolsep}{4pt}
    \renewcommand{\arraystretch}{1.15}
    \begin{tabular}{c *{5}{|ccc}|c}
    \toprule
    \multirow{2}{*}{$\tau$} 
    & \multicolumn{3}{c|}{Step $k=1$}
    & \multicolumn{3}{c|}{Step $k=2$}
    & \multicolumn{3}{c|}{Step $k=3$}
    & \multicolumn{3}{c|}{Step $k=4$}
    & \multicolumn{3}{c|}{Step $k=5$}
    & \multirow{2}{*}{$\mathcal{L}(\hat{\boldsymbol{\theta}};\,\tau)$} \\
    \cmidrule(lr){2-4}\cmidrule(lr){5-7}\cmidrule(lr){8-10}\cmidrule(lr){11-13}\cmidrule(lr){14-16}
    & $a_{1,1}$ & $a_{3,1}$ & $a_{5,1}$
    & $a_{1,2}$ & $a_{3,2}$ & $a_{5,2}$
    & $a_{1,3}$ & $a_{3,3}$ & $a_{5,3}$
    & $a_{1,4}$ & $a_{3,4}$ & $a_{5,4}$
    & $a_{1,5}$ & $a_{3,5}$ & $a_{5,5}$ & \\
    \midrule
    $0.1$ & $+4.753$ & $-10.636$ & $+7.172$ & $+2.414$ & $-2.282$ & $+0.877$ & $+2.589$ & $-1.202$ & $+0.245$ & $+1.999$ & $-1.774$ & $+0.525$ & $+2.131$ & $-1.530$ & $+0.274$ & $0.00070$ \\
    $0.2$ & $+3.104$ & $-3.578$  & $+1.639$ & $+2.844$ & $-2.041$ & $+0.616$ & $+2.577$ & $-2.567$ & $+0.639$ & $+2.807$ & $-1.811$ & $+0.113$ & $+1.877$ & $-1.153$ & $+0.292$ & $0.00105$ \\
    $0.3$ & $+2.547$ & $-1.190$  & $+0.122$ & $+3.202$ & $-1.581$ & $+0.326$ & $+3.100$ & $-1.684$ & $+0.229$ & $+2.289$ & $-1.547$ & $+0.342$ & $+2.185$ & $-1.841$ & $+0.645$ & $0.00278$ \\
    $0.4$ & $+2.624$ & $-1.021$  & $-0.555$ & $+2.762$ & $-1.221$ & $+0.238$ & $+2.682$ & $-1.486$ & $+0.293$ & $+2.021$ & $-1.724$ & $+0.483$ & $+2.154$ & $-1.565$ & $+0.283$ & $0.00412$ \\
    $0.5$ & $+2.461$ & $-0.443$  & $-0.811$ & $+3.084$ & $-1.139$ & $+0.188$ & $+2.612$ & $-1.453$ & $+0.220$ & $+2.057$ & $-1.837$ & $+0.558$ & $+2.043$ & $-1.355$ & $+0.224$ & $0.00263$ \\
    $0.6$ & $+2.313$ & $-0.434$  & $-0.335$ & $+2.913$ & $-0.920$ & $+0.130$ & $+2.751$ & $-1.493$ & $+0.217$ & $+1.939$ & $-1.683$ & $+0.470$ & $+2.253$ & $-1.784$ & $+0.353$ & $0.00316$ \\
    $0.7$ & $+1.636$ & $-0.310$  & $-0.039$ & $+3.286$ & $-1.566$ & $+0.333$ & $+3.036$ & $-1.962$ & $+0.338$ & $+2.004$ & $-1.730$ & $+0.476$ & $+2.204$ & $-1.663$ & $+0.313$ & $0.00524$ \\
    $0.8$ & $+1.743$ & $-0.247$  & $+0.015$ & $+2.990$ & $-0.933$ & $+0.130$ & $+2.656$ & $-1.371$ & $+0.189$ & $+2.069$ & $-1.663$ & $+0.423$ & $+2.054$ & $-1.345$ & $+0.220$ & $0.00843$ \\
    $0.9$ & $+3.009$ & $-1.041$  & $-0.021$ & $+2.796$ & $-2.170$ & $+0.433$ & $+3.051$ & $-2.487$ & $+0.507$ & $+2.128$ & $-2.166$ & $+0.772$ & $+2.174$ & $-1.889$ & $+0.709$ & $0.00043$ \\
    \bottomrule
    \end{tabular}%
    }
\end{table}
\clearpage
\newpage
\section{Limitations}
\label{sec: limitations}


Pion is designed for regimes where the informative descent direction concentrates in a few leading singular values, which is not the case for LLM pretraining: pretraining benefits from Muon's uniform whitening, which lifts every singular value to $1$ and maximizes spectral exploration, whereas Pion's high-pass NS attenuates the tail and discards potentially useful directions. We therefore expect Pion to underperform Muon on LLM pretraining, and we leave to future work the question of how to adapt the high-pass cutoff to recover Muon's exploration behavior in pretraining while retaining Pion's noise robustness in VLA and RLVR.

\section{Broader Impact}
\label{sec: broader_impact}


On the positive side, Pion lowers the cost of training capable VLA policies and RLVR-tuned reasoning LLMs by stabilizing post-training under the same compute budget as Muon, which can broaden access to embodied agents and reasoning models. On the negative side, more capable VLA policies and reasoning LLMs carry the standard dual-use risks of robotic and language-based agents, including unsafe deployment and misuse for harmful content. We hope our work encourages further study of matrix-aware optimization beyond LLM pretraining alongside the safety practices already established for VLA and RLVR systems.


\end{document}